\documentclass[11pt]{article} 
\pdfoutput=1
\usepackage{times}
\usepackage{amsmath}
\usepackage{amsthm}
\usepackage{amssymb}
\usepackage{graphicx}
\usepackage{url}
\usepackage{setspace}
\usepackage{hyperref}
\usepackage{framed}
\usepackage{soul}
\usepackage[autostyle]{csquotes}
\usepackage{caption}
\usepackage{subcaption}
\usepackage[compact]{titlesec}
\usepackage{mathtools}
\usepackage{xspace}
\usepackage{cancel}
\usepackage{wrapfig}
\usepackage{nicefrac}
\usepackage{tikz}
\usetikzlibrary{tikzmark}
\usepackage{pstricks}
\usepackage{makecell}
\usepackage{fullpage}
\usepackage{algorithmic}
\usepackage[ruled]{algorithm2e}
\usepackage{tikz}
\usetikzlibrary{tikzmark}
\usepackage{subcaption}
\usepackage{makecell}
\usepackage{natbib}

\definecolor{myblue}{rgb}{0.204,0.204,0.612}
\definecolor{mygray}{rgb}{0.5,0.5,0.5}
\definecolor{mywhite}{rgb}{0.0,0.0,0.0}
\definecolor{mypurple}{rgb}{0.408,0.0,0.408}

\newtheorem{theorem}{Theorem}[section]
\newtheorem{lemma}[theorem]{Lemma}

\newtheorem{corollary}[theorem]{Corollary}

\newtheorem{definition}[theorem]{Definition}
\newtheorem{example}[theorem]{Example}
\newtheorem{itheorem}[theorem]{Informal Theorem}

\newcommand{\col}{j}
\newcommand{\bit}{\beta}

\renewcommand{\S}{\mathcal{S}}

\newcommand{\B}{\mathcal{B}}
\newcommand{\A}{\mathcal{A}}
\newcommand{\V}{\mathcal{V}}

\renewcommand{\P}{\mathcal{P}}

\newcommand{\X}{\mathcal{X}}
\newcommand{\Y}{\mathcal{Y}}
\newcommand{\I}{\mathcal{I}}
\newcommand{\G}{\mathcal{G}}

\newcommand{\E}{\mathbb{E}}
\newcommand{\R}{\mathbb{R}}

\newcommand{\poly}{\mathrm{poly}}

\newcommand{\rev}{\mathrm{Rev}}

\newcommand{\opt}{\textsc{\upshape{Opt}}\xspace}
\newcommand{\weakopt}{\textsc{\upshape{IntOpt}}\xspace}

\newcommand{\barCalR}{{\,\overline{\!\mathcal{R}\!}\,}}
\newcommand{\barRL}{{\overline{\textup{\textsc{RL}}\!}\,}}
\newcommand{\mGamma}{\mathbf{\Gamma}}
\newcommand{\valpha}{{\vec\alpha}}
\newcommand{\vq}{{\vec q}}
\newcommand{\va}{{\vec a}}
\newcommand{\inprod}{\mathbin{\cdot}}

\newcommand{\order}{O}
\newcommand{\sep}{Q}
\newcommand{\tsep}{P}
\newcommand{\states}{Q}
\newcommand{\pseudo}[1]{\mathrm{pseudo}(#1)}

\renewcommand{\vec}[1]{{\boldsymbol{\mathbf #1}}}

\newcommand{\regret}{{\textsc{\upshape{Regret}}}}

\DeclareMathOperator*{\argmax}{\arg\max}

\newcommand{\card}[1]{\lvert#1\rvert}
\newcommand{\set}[1]{\{#1\}}
\newcommand{\Set}[1]{\left\{#1\right\}}
\newcommand{\bigSet}[1]{\bigl\{#1\bigr\}}

\newcommand{\braces}[1]{\{#1\}}

\newcommand{\bigBraces}[1]{\bigl\{#1\bigr\}}
\newcommand{\BigBraces}[1]{\Bigl\{#1\Bigr\}}

\newcommand{\bigBracks}[1]{\bigl[#1\bigr]}
\newcommand{\BigBracks}[1]{\Bigl[#1\Bigr]}
\newcommand{\biggBracks}[1]{\biggl[#1\biggr]}

\newcommand{\Bracks}[1]{\left[#1\right]}
\newcommand{\parens}[1]{(#1)}
\newcommand{\Parens}[1]{\left(#1\right)}
\newcommand{\bigParens}[1]{\bigl(#1\bigr)}
\newcommand{\BigParens}[1]{\Bigl(#1\Bigr)}

\newcommand{\bigGiven}{\mathbin{\bigm\vert}}
\newcommand{\BigGiven}{\mathbin{\Bigm\vert}}
\newcommand{\biggGiven}{\mathbin{\biggm\vert}}

\newcommand{\norm}[1]{\lVert#1\rVert}

\newcommand{\floors}[1]{\lfloor#1\rfloor}
\newcommand{\Floors}[1]{\left\lfloor#1\right\rfloor}
\newcommand{\ceils}[1]{\lceil#1\rceil}
\mathchardef\dash="2D

\newcommand{\Eq}[1]{Equation~\eqref{eq:#1}}
\newcommand{\Sec}[1]{Section~\ref{sec:#1}}
\newcommand{\Thm}[1]{Theorem~\ref{thm:#1}}
\newcommand{\Cor}[1]{Corollary~\ref{cor:#1}}
\newcommand{\ellOther}{h}

\newcommand\numberthis{\addtocounter{equation}{1}\tag{\theequation}}

\hypersetup{
     colorlinks   = true,
      linkcolor = teal,
     urlcolor    = teal,
	 citecolor = teal
}

\title{Oracle-Efficient Online Learning and Auction Design}
\author{Miroslav Dud\'ik\thanks{Microsoft Research, {\tt mdudik@microsoft.com}}
\and
Nika Haghtalab\thanks{Computer Science Department, Cornell University, {\tt nika@cs.cornell.edu}}
\and
Haipeng Luo\thanks{University of Southern California, {\tt haipengl@usc.edu}}  \and
Robert E. Schapire\thanks{Microsoft Research, {\tt schapire@microsoft.com}}
\and
Vasilis Syrgkanis\thanks{Microsoft Research, {\tt vasy@microsoft.com}}
\and
Jennifer Wortman Vaughan\thanks{Microsoft Research, {\tt jenn@microsoft.com}}}

\date{}

\allowdisplaybreaks
\begin{document}

\maketitle

\begin{abstract}
We consider the design of computationally efficient online learning algorithms in an adversarial setting in which the learner has access to an offline optimization oracle. We present an algorithm called Generalized Follow-the-Perturbed-Leader and provide conditions under which it is oracle-efficient while achieving vanishing regret. Our results make significant progress on an open problem raised by \citet{HK16Talk}, who showed that oracle-efficient algorithms do not exist in general~\cite{HK16}  and asked whether one can identify properties under which oracle-efficient online learning may be possible.
 
Our auction-design framework considers an auctioneer learning an optimal auction for a sequence of adversarially selected valuations with the goal of achieving revenue that is almost as good as the optimal auction in hindsight, among a class of auctions. We give oracle-efficient learning results for: (1) VCG auctions with bidder-specific reserves in single-parameter settings, (2) envy-free item pricing in multi-item auctions, and (3) s-level auctions of \citet{morgenstern2015pseudo} for single-item settings. The last result leads to an approximation of the overall optimal Myerson auction when bidders' valuations are drawn according to a fast-mixing Markov process, extending prior work that only gave such guarantees for the i.i.d. setting. 

Finally, we derive various extensions, including: (1) oracle-efficient algorithms for the contextual learning setting in which the learner has access to side information (such as bidder demographics), (2) learning with approximate oracles such as those based on Maximal-in-Range algorithms, and (3) no-regret bidding in simultaneous auctions, resolving an open problem of \citet{DS16}.

\end{abstract}


\setcounter{page}{0}
\thispagestyle{empty}
\newpage

\section{Introduction}
\label{sec:intro}

Online learning plays a major role in the adaptive optimization of
computer systems, from the design of online marketplaces
\citep{BH05,balcan2006approximation,CGM13,roughgarden2016minimizing}
to the optimization of routing schemes in communication networks
\citep{Awerbuch2008}. The environments in these applications are
constantly evolving,
requiring continued adaptation of these systems. Online learning
algorithms have been designed to robustly address this challenge, with
performance guarantees that hold even when the environment is
adversarial. However, the information-theoretically optimal learning algorithms that work with arbitrary payoff functions are computationally inefficient when the learner's action space is exponential in the natural problem representation~\cite{FS97}. For certain action spaces and environments, efficient online
learning algorithms can be designed by reducing the online learning
problem to an optimization problem
\citep{KV05,Kakade2005,Awerbuch2008,HK12}.
However, these approaches do not easily extend to the complex and
highly non-linear problems faced by real learning systems, such as the
learning systems used in online market design.
In this paper, we address the problem of efficient online learning
with an exponentially large action space under arbitrary learner
objectives.

This goal is not achievable without some assumptions on
the problem structure.
Since an online optimization problem is at least as hard as the corresponding offline optimization problem~\citep{Cesa-Bianchi2006,DS16}, a minimal assumption is the existence of an algorithm that returns a near-optimal solution to the offline problem. We assume that our learner has access to such an offline algorithm, which we call \emph{an offline optimization oracle}. This oracle, for any (weighted) history of choices by the environment, returns an action of the learner that (approximately) maximizes the learner's reward.
We seek to design \emph{oracle-efficient learners}, that is, learners that run in polynomial time, with each oracle call counting $O(1)$.

An oracle-efficient learning algorithm can be viewed as a reduction from the online to the
offline problem, providing conditions under which the online problem
is not only as hard, but also as easy as the offline problem, and thereby
offering computational equivalence between online and offline
optimization.
Apart from theoretical significance, reductions from online to
offline optimization are also practically important.
For example, if one has
already developed and implemented a Bayesian optimization procedure
which optimizes against a static stochastic environment, then our
algorithm offers a black-box transformation of that procedure into an
adaptive optimization algorithm with provable learning guarantees in
non-stationary, non-stochastic environments. Even if the existing
optimization system does not run in worst-case polynomial
time, but is rather a well-performing fast heuristic, a reduction
to offline optimization is capable of leveraging any expert domain knowledge
that went into designing the heuristic, as well as any further
improvements of the heuristic or even discovery of polynomial-time
solutions.

Recent work of \citet{HK16} shows that oracle-efficient learning in
adversarial environments is not achievable in general, while leaving
open the problem of identifying the properties under which
oracle-efficient online learning may be possible~\cite{HK16Talk}.
We introduce a generic algorithm called
\emph{Generalized Follow-the-Perturbed-Leader} (Generalized FTPL) and
derive sufficient conditions under which this algorithm yields
oracle-efficient online learning. Our results are enabled by providing
a new way of adding \emph{regularization} so as to \emph{stabilize}
optimization algorithms in general optimization settings. The latter
could be of independent interest beyond online learning.  Our approach
unifies and extends previous approaches to oracle-efficient learning,
including the Follow-the-Perturbed Leader (FTPL) approach introduced
by \citet{KV05} for linear objective functions, and its
generalizations to submodular objective functions~\cite{HK12},
adversarial contextual learning~\cite{Syrgkanis2016}, and learning in
simultaneous second-price auctions~\cite{DS16}. Furthermore, our
sufficient conditions are related to the notion of a
universal identification set of~\citet{goldman1993exact} and oracle-efficient online optimization techniques of~\citet{DS16}.

The second main contribution of our work is to introduce a new
framework for the problem of adaptive auction design for revenue
maximization and to demonstrate the power of Generalized FTPL through
several applications in this framework. Traditional auction theory
assumes that the valuations of the bidders are drawn from a population
distribution which is known, thereby leading to a Bayesian
optimization problem. The knowledge of the distribution by the seller
is a strong assumption. Recent work in algorithmic mechanism
design \cite{Cole2014,morgenstern2015pseudo,DHP16,Roughgarden2016}
relaxes this assumption by solely assuming access to a set of samples
from the distribution. {In this work, we drop any distributional
assumptions and introduce the adversarial learning framework
of \emph{online auction design}.}  On each round, a learner
adaptively designs an auction rule for the allocation of a set of
resources to a fresh set of bidders from a
population.\footnote{Equivalently, the set of bidders on each round
  can be the same as long as they are myopic and optimize
  their utility separately in each round.  Using our extension to
  contextual learning (Section~\ref{sec:contexts}), this approach can
  also be applied when the learner's choice of auction is allowed to
  depend on features of the arriving set of bidders, such as
  demographic information.}  The goal of the learner is to achieve
average revenue at least as large as the revenue of the best
auction from some target class.  Unlike the standard approach to
auction design, initiated by the seminal work of
\citet{Myerson1981}, our approach is devoid of any assumptions about a
prior distribution on the valuations of the bidders for the resources
at sale. Instead, similar to an agnostic approach in
learning theory, we incorporate prior knowledge in the form of a
target class of auction schemes that we want to compete with.
This is especially appropriate when the auctioneer is restricted to using a particular design of auctions with power to make only a few design choices, such as deciding the reserve prices in a second-price auction. A special case of our framework is considered in the recent work of
\citet{roughgarden2016minimizing}. They study online learning
of the class of single-item second-price auctions with
bidder-specific reserves, and give an algorithm with performance that
approaches a constant factor of the optimal revenue in hindsight. We
go well beyond this specific setting and show that our Generalized
FTPL can be used to optimize over several standard classes of auctions
including VCG auctions with bidder-specific reserves and the level
auctions of \citet{morgenstern2015pseudo}, achieving low additive
regret to the best auction in the class.

In the remainder of this section, we describe our main results in more detail and then discuss several extensions and applications of these results, including (1) learning with side information (i.e., contextual learning); (2) learning with constant-factor approximate oracles (e.g., using Maximal-in-Range algorithms~\cite{Nisan2007});
(3) regret bounds with respect to stronger benchmarks for the case in which the environment is not completely adversarial, but follows a fast-mixing Markov process.

Our work contributes to two major research agendas: the design of efficient and oracle-efficient online learning algorithms~\citep{AgarwalHsKaLaLiSc14, DHKKLRZ11, DS16,HK12, HK16,Kakade2005, KV05, RS16, SyrgkanisLuKrSc16}, and the design of auctions using machine learning tools~\citep{Cole2014, DHP16, morgenstern2015pseudo,kleinberg2003value,BH05,CGM13}.
We describe related work from both areas in more detail in Appendix~\ref{app:related}.

\subsection{Oracle-Efficient Learning with Generalized FTPL}
\label{sec:intro:ftpl}

We consider the following online learning problem. On each round
$t = 1, \ldots, T$, a learner chooses an action $x_t$ from a finite
set $\X$, and an adversary chooses an action $y_t$ from a set $\Y$. The learner then observes $y_t$ and
receives a payoff $f(x_t, y_t) \in [0,1]$, where the function $f$ is
fixed and known to the learner. The goal of the learner is to obtain
low expected regret with respect to the best action in hindsight,
i.e., to minimize
\[
\regret\coloneqq\E\left[ \max_{x \in \X} \sum_{t=1}^T f(x, y_t) -
  \sum_{t=1}^T f(x_t,
  y_t) \right]
,
\]
where the expectation is over the randomness of the learner.\footnote{To
simplify exposition, we assume that the adversary is oblivious, i.e.,
that the sequence $y_1, \ldots, y_T$ is chosen in advance. Our results generalize
to adaptive adversaries using standard
techniques~\citep{HP05,DS16}.}
We desire algorithms, called \emph{no-regret algorithms}, for which this regret is sublinear in the time horizon $T$.

Our algorithm takes its name from the seminal Follow-the-Perturbed-Leader (FTPL) algorithm of
 \citet{KV05}. FTPL achieves low regret, $O(\sqrt{T \log|\X|})$, by
independently perturbing the historical payoff of each of the learner's
actions and choosing on each round the action with the highest
perturbed payoff. However, this approach is inefficient when the
action space is exponential in the natural representation of the
learning problem, because it requires creating $|\X|$ independent random variables.\footnote{%
  If payoffs are linear in some low-dimensional representation of $\X$ then the number of variables needed is equal to this dimension. But for non-linear payoffs, $|\X|$ variables are required.}
Moreover, because of the form of the perturbation, the optimization of
the perturbed payoffs cannot be performed by the offline optimization
oracle for the same problem.
We overcome both of these challenges by, first, generalizing FTPL to work with perturbations that can be compactly represented and are thus not necessarily independent across different actions (\emph{sharing randomness}), and, second, by implementing such perturbations via synthetic histories of adversary actions and thus creating offline problems of the same form as the online problem (\emph{implementing  randomness}).

\paragraph{Sharing Randomness.}
Our Generalized
FTPL begins by drawing a random vector $\valpha\in\R^N$ of
dimension~$N$,
with components $\alpha_j$ drawn independently from a
dispersed distribution $D$. The payoff of each of the learner's actions
is perturbed by a linear combination of these independent
variables, as prescribed by a {\it perturbation translation matrix} $\mGamma$ of size $|\X|\times N$, with entries
in $[0,1]$.
Let $\mGamma_x$ denote the row of
$\mGamma$ corresponding to $x$. On each round $t$, the algorithm
outputs an action $x_t$ that (approximately) maximizes the perturbed
historical performance.  In other words, $x_t$ is chosen such that for all $x \in \X$,
\[
\sum_{\tau=1}^{t-1} f(x_t, y_\tau) + \valpha\inprod\mGamma_{x_t} \geq  \sum_{\tau=1}^{t-1} f(x, y_\tau) + \valpha\inprod\mGamma_{x} - \epsilon
\]
for some fixed optimization accuracy $\epsilon \geq 0$. This procedure is fully described in
Algorithm~\ref{alg:ftpl-approximate} of Section~\ref{sec:ftpl}.

We show that Generalized FTPL is no-regret
as long as $\epsilon$ is sufficiently small
and the translation matrix $\mGamma$
satisfies an \emph{admissibility} condition.
This condition requires the rows of $\mGamma$ to be (sufficiently) distinct
so that each
action's perturbation uses a different weighted combination of the $N$-dimensional noise.
To the best of our knowledge, the approach of using an arbitrary matrix to induce shared randomness among actions of the learner is novel.
The formal no-regret result is in Theorem~\ref{thm:FTPL-U}. The informal statement is the following:
\begin{itheorem}
\label{ithm:admiss}
A translation matrix is
$\delta$-admissible if any two rows of the matrix are
distinct and the minimum non-zero difference between any two values
within a column is at least $\delta$. Generalized FTPL with a
$\delta$-admissible matrix $\mGamma$ and an appropriate
distribution $D$ achieves regret
$O(N\sqrt{T}/\delta + \epsilon T)$.
\end{itheorem}

A technical challenge here is to show that the randomness induced by $\mGamma$ on the set of actions $\X$ \emph{stabilizes} the algorithm, i.e., the probability that $x_t\neq x_{t+1}$ is small.
We use the admissibility of $\mGamma$ to guide us through the analysis of stability.
In particular, we consider how each column of $\mGamma$ partitions
actions of $\X$ to a few subsets (at most $1+\delta^{-1}$) based on
their corresponding entries in that column.
Since the matrix rows are distinct, the algorithm is stable as a whole if, for
each column, the partition to which the algorithm's chosen
  action belongs to remains the same between consecutive rounds with probability close to $1$.
This allows us to decompose the stability analysis of the algorithm as a whole to the analysis of stability across partitions of each column.
At the column level, stability of the partition between two rounds follows by showing that a switch between partitions happens only if the perturbation $\alpha_j$ corresponding to that column falls into a small sub-interval of the support of the distribution $D$, from which it is sampled. The latter probability is small if $D$ is sufficiently dispersed. This final argument is similar in nature to the reason why perturbations lead to stability in the original FTPL algorithm of \citet{KV05}.

\paragraph{Implementing Randomness.}
To ensure oracle-efficient learning, we additionally need the property
that the induced action-level perturbations can be simulated by a
(short) synthetic history of adversary actions.  This allows us to avoid working with $\mGamma$ directly, or even explicitly writing it down.
This requirement is captured by our \emph{implementability} condition, which states that each column of the translation matrix corresponds to a scaled version of the expected reward of the learner on some distribution of adversary actions.
The formal statement is in Theorem~\ref{thm:adm:imp}. The informal statement is the following:

\begin{itheorem}
\label{ithm:impl}
A translation matrix is implementable if each column corresponds to a
scaled version of the expected reward of the learner against some
finitely supported
distribution of actions of the adversary. Generalized
FTPL with an implementable translation matrix can be implemented with
one oracle call per round and runs in time polynomial in $N$, $T$, and
the size of the
support of the distributions implementing the translation matrix.
Oracle calls count $O(1)$ in the running time.
\end{itheorem}

The use of synthetic histories in online optimization was first explored by Daskalakis and Syrgkanis~\cite{DS16}, who sample histories of length $\poly(|\Y|)$ from a fixed distribution.  Our implementability property uses matrix $\mGamma$ to obtain problem-specific distributions that stabilize online optimization with shorter histories.

For some learning problems, it is easier to first construct an
implementable translation matrix and argue about its
admissibility; for others, it is easier to construct an admissible
matrix and argue about its implementability. We pursue both strategies
in various applications,
demonstrating the versatility of our conditions.

Our theorems yield the following simple sufficient condition for oracle-efficient no-regret learning (see Theorems~\ref{thm:FTPL-U} and \ref{thm:adm:imp} for more general statements):
\begin{quote}
If there exist $N$ adversary actions such that any two actions of
the learner yield different rewards on at least one of these $N$
actions, then Generalized FTPL with an appropriate translation matrix has regret $O (N\sqrt{T}/\delta)$ and its oracle-based runtime is $\poly(N,T)$ where $\delta$ is the smallest difference between distinct rewards obtainable on any one of the $N$ adversary actions.
\end{quote}

The aforementioned results establish a reduction from online optimization to offline optimization. Recall that in the oracle-based runtime, each oracle call counts $O(1)$.
When the offline optimization problem can be solved in polynomial time, these results imply that the online optimization problem can also be solved in (fully) polynomial time. The formal statement is in Corollary~\ref{cor:oracle-ftpl}.

\subsection{Main Application: Online Auction Design}

In many applications of auction theory, including electronic marketplaces, a seller repeatedly sells an item or a set of items to a population of buyers, with a few arriving for each auction. In such cases, the seller can optimize his auction design in an online manner, using historical data consisting of observed bids. We consider a setting in which the seller would like to use this historical data to select an auction from a fixed target class.
{For example, a seller in a sponsored-search auction might be limited by practical constraints to consider only second-price auctions with bidder-specific reserves. The seller can optimize the revenue by using the historical data for each bidder to set these reserves. Similarly, a seller on eBay may be restricted to set a single reserve price for each item. Here, the seller can optimize the revenue by using historical data from auctions for similar goods to set the reserves for new items. In both cases, the goal is to leverage the historical data to pick an auction on each round in such a way that the seller's overall revenue compares favorably with the optimal auction from the target class.}

More formally, on round $t =1, \dotsc, T$, a tuple of $n$ bidders arrives with a vector of $n$ bids or, equivalently, a vector of valuations (since we only consider truthful auctions), denoted $\vec{v}_t\in \V^n$.
We allow these valuations to be arbitrary, e.g., chosen by an adversary.
Prior to observing the bids, the auctioneer commits to an auction $a_t$ from a class of truthful auctions $\A$.
The goal of the {auctioneer}
is to achieve a revenue that, in hindsight, is
very close to the revenue that would have been achieved by the best
fixed auction in class $\A$ if that auction were used on all rounds. In other words, the {auctioneer} aims to minimize the {expected regret}
\begin{equation*}
\E\left[ \max_{a \in \A} \sum_{t=1}^T \rev (a, \vec v_t)
  - \sum_{t=1}^T \rev (a_t, \vec v_t) \right],
\end{equation*}
where $\rev(a, \vec v)$ is the revenue of auction $a$ on bid profile $\vec v$ and the expectation is over the actions of the auctioneer.

This problem can easily be cast in our oracle-efficient online learning framework. The learner's action space is the set of target auctions $\A$, while the adversary's action space is the set of bid or valuation vectors $\V^n$. The offline oracle is a revenue maximization oracle which computes an (approximately) optimal auction within the class $\A$ for any given set of valuation vectors.
Using the Generalized FTPL with appropriate matrices $\mGamma$, we provide the first oracle-efficient no-regret algorithms for three commonly studied auction classes:
\begin{itemize}
\item[--] Vickrey-Clarkes-Groves (VCG) auctions with bidder-specific reserve
  prices in single-dimensional matroid settings, which are
  known to achieve half the revenue of
  the optimal auction in i.i.d.\ settings under some conditions \cite{hartline2009simple};
\item[--] envy-free
  item-pricing mechanisms in combinatorial markets with unlimited
  supply, often studied in the static Bayesian setting
  \citep{balcan2006approximation, guruswami2005profit};
 \item[--]single-item level auctions, introduced by
  \citet{morgenstern2015pseudo}, who show that these
 auctions approximate, to an arbitrary accuracy, the Myerson
    auction \citep{Myerson1981}, which is known to be optimal for the
    Bayesian independent-private-value setting.
\end{itemize}

The crux of our approach  is in designing admissible and implementable matrices. For the first two mentioned classes, VCG auctions with bidder-specific reserves and envy-free item-pricing auctions, we show how to implement an (obviously admissible) matrix $\mGamma$, where each row corresponds, respectively,
to the concatenated binary representation of bidder reserves or item prices.
We show that, surprisingly, any perturbation
that is a linear function of this binary representation can be simulated by a distribution of bidder valuations (see Figure \ref{fig:table} for an example).
For the third class, level auctions, our challenge is to show that a clearly implementable matrix $\mGamma$, with each column implemented by a single bid profile, is also admissible.

Table~\ref{tab:examples} summarizes regret bounds and runtimes of our oracle-efficient
algorithms, assuming oracle calls run in time
$O(1)$. Our algorithms perform a single oracle call per round, so $T$ oracle calls in total.
Our runtimes demonstrate an efficient reduction from online to offline learning.

While, in theory, the auction classes discussed in the table do not have worst-case polynomial-time algorithms for the offline problem, in practice, the offline problems are solved by a wide range of approaches, such as integer program solvers, which run fast on typical problem instances.
Our approach can turn existing offline routines into online algorithms
at almost no additional cost.
When one wishes to leverage provably polynomial-time algorithms, even if their outputs are sub-optimal, our framework still applies as long as the algorithms
return actions whose payoffs are within a constant factor of the best payoff.
We discuss this approach briefly in Section~\ref{sec:intro:ext} below and then demonstrate its applicability in Section~\ref{sec:additional}, where we derive a 
polynomial-time algorithm for online welfare maximization in multi-unit auctions.

\begin{table*}[t]
\centering
\renewcommand{\arraystretch}{1}%
\caption{\label{tab:examples} Regret bounds and oracle-based runtime for the auction classes considered in
  this work, for $n$ bidders and time horizon $T$. All our algorithms perform a single oracle call per round.}
\resizebox{\columnwidth}{!}{
\begin{tabular}{|c|c|c|c|}
\hline
Auction Class & Regret & Oracle-Based Runtime & Section \\
\hline
VCG with bidder-specific reserves, $s$-unit &
$O(n s\sqrt{T}\log T)$ & $O(T^2+nT^{3/2}\log T)$ &
\ref{sec:VCG} \\
\hline
\parbox{2.75in}{\centering envy free $k$-item pricing with infinite supply and unit-demand or single-minded bidders} &
$O\bigParens{n k \sqrt{ T}\log(kT)}$ & $O\bigParens{T^2 +k^2 T^{3/2}\log(kT)}$ &
\ref{sec:envy-free}\\
\hline
level auction${}^a$ with discretization level $m$ &
$O(n^3 m^4 \sqrt{ T})$ &  $O(T^2+n^3 m^3 T)$ &
\ref{sec:level} \\
\hline
\end{tabular}
}
\footnotesize
${}^a$In the extended abstract \cite{DHL+17} we presented the regret bound $O(n m^2 \sqrt{ T})$ and running time $O(T^2+n m T)$, corresponding to discretized level auctions with distinct thresholds (see \Thm{S_m,s}). Here, we present the result that allows repetitions of threshold values (see \Thm{R_m,s}).
\end{table*}

\subsection{Extensions and Additional Applications}
\label{sec:intro:ext}

In Sections~\ref{sec:overallopt}--\ref{sec:additional}, we present several extensions and additional applications (see Table~\ref{tab:examples-additional} for a summary):

\paragraph{Markovian Adversaries and Competing with the Optimal Auction (Section~\ref{sec:overallopt}).}
\citet{morgenstern2015pseudo} show that level auctions can provide
an arbitrarily accurate approximation to the overall optimal Myerson
auction in the Bayesian single-item auction setting if the values of the bidders are drawn from independent distributions and i.i.d.\ over time.
Therefore, {if the environment in an online setting picks bidder valuations from independent distributions}, standard online-to-batch reductions imply that the
revenue of Generalized FTPL with the class of level auctions is close to the overall optimal (i.e., not just best-in-class) single-shot auction.
We generalize this reasoning and show the same strong optimality guarantee when the valuations of bidders on each round are drawn from a fast-mixing Markov process that is independent across bidders but Markovian over rounds.
For this setting, our results give an oracle-efficient algorithm with regret $O(n^{3/5}T^{9/10})$ to the overall optimal auction, rather than just best-in-class. This is the first result on competing with the Myerson optimal auction for non-i.i.d. distributions, as all prior work \cite{Cole2014,morgenstern2015pseudo,DHP16,Roughgarden2016} assumes i.i.d.\ samples.

\paragraph{Contextual Learning (Section~\ref{sec:contexts}).}
In this setting, on each round $t$ the learner observes a context
 $\sigma_t$ before choosing an action.  For example, in online auction design, the context might represent demographic information about the set of bidders.  The goal of the learner is to compete with the best \emph{policy} in some fixed class, where each policy is a mapping from a context $\sigma_t$ to an action.  We propose a contextual extension of the translation matrix $\mGamma$.  Generalized FTPL can be applied using this extended translation matrix and provides sublinear regret bounds for both the case in which there is a small ``separator'' of the policy class and the transductive setting in which the set of all possible contexts is known ahead of time.
Our results extend and generalize the results of  \citet{Syrgkanis2016} from  contextual combinatorial optimization to any learning setting that admits an implementable and admissible translation matrix.

The contextual extension is particularly useful in online auction design, because it allows the learner to use any side information available about the bidders before they place their bids to guide the choice of auction. While the number of bidders might be too large to learn about them individually,
the learner can utilize the side information to design a common treatment for bidders that are similar, that is, to  \emph{generalize across  a population}.

{Our performance guarantees for adaptive auction design, similar to much prior work, rely on the assumption that the bidders are either myopic or are different on each round. One criticism of this assumption is that such adaptive mechanisms might be manipulated by strategic bidders who distort their bids so as to gain in the future. The contextual learning algorithms mitigate this risk by pooling similar bidders, which reduces the probability that the exact same bidder will be overly influential in the choices of the algorithm.}

\paragraph{Approximate Oracles and Approximate Regret (Section~\ref{sec:approximate}).}
For some problems there might not exist a sufficiently fast (e.g., polynomial-time or FPTAS) offline oracle with small additive error as required for Generalized FTPL. To make our results more applicable in practice, we extend them to handle oracles that are only required to return an action with performance that is within a constant multiplicative factor, {$C\le 1$}, of that of the optimal action in the class.
We consider two examples of such oracles:
Relaxation-based Approximations~\citep{balcan2006approximation} and Maximal-in-Range (MIR) algorithms~\cite{Nisan2007}.
Our results hold in both cases with a modified version of regret, called \emph{$C$-regret}, in which the online algorithm competes with $C$ times the payoff of the optimal action in hindsight.

\paragraph{Additional Applications (Section~\ref{sec:additional}).}
Finally, we provide further applications of our work in the area of online combinatorial optimization with MIR approximate oracles, and in the area of no-regret learning for bid optimization in simultaneous second-price auctions. \begin{itemize}
\item[--] In the first application, we give a polynomial-time learning algorithm for online welfare maximization in multi-unit auctions that achieves {$1/2$-regret}, by invoking the polynomial-time MIR approximation algorithm of \citet{dobzinski2007mechanisms} as an offline oracle.
\item[--] In the second application, we solve an open problem
raised in the recent work of \citet{DS16}, who offered efficient learning algorithms only for the weaker benchmark of no-envy learning, rather than no-regret learning, in simultaneous second-price auctions, and left open the question of oracle-efficient no-regret learning. We show that no-regret learning in simultaneous item auctions is efficiently achievable, assuming access to an optimal bidding oracle against a known distribution of opponents bids (equivalently, against a distribution of item prices).
\end{itemize}

\renewcommand{\arraystretch}{1}
\begin{table*}[t]
\centering
\caption{\label{tab:examples-additional} Additional results considered in Sections~\ref{sec:overallopt}--\ref{sec:additional} and their significance. Above, $m$ is the discretization level of the problems, $n$ is the number of bidders, and $T$ is the time horizon.}
\renewcommand{\tabcolsep}{3.5pt}
\resizebox{\columnwidth}{!}{
\begin{tabular}{|c|c|c|c|}
\hline
Problem Class & Regret & Section & Notes \\
\hline
 Markovian, single item
 & $O(n^{3/5}T^{9/10})$
 & \ref{subsec:overallOPT} & competes with Myerson's optimal auction \\
\hline
contextual online auction${}^a$
 & $O(\sqrt{T})$ or $O(T^{3/4})$
 & \ref{sec:contexts} & allows side information about bidders\\
\hline
welfare maximization, $s$-unit${}^b$
 & $1/2$-regret: $O(n^4\sqrt{T})$
 & \ref{sec:knapsack} & fully (oracle-free) polynomial-time algorithm\\
\hline
bidding in SiSPAs, $k$ items & $O(km\sqrt{ T})$ & \ref{sec:sispas} &
solves an open problem~\cite{DS16}\\
\hline
\end{tabular}
}
\footnotesize
${}^a$The two bounds are for the small separator setting and the transductive setting, respectively. Dependence on parameters other than $T$ is omitted.\\
${}^b$The regime of interest in this problem is $s\gg n$. The extended abstract \cite{DHL+17} contained a worse bound $O(n^6\sqrt{T})$.
\end{table*}

\section{Generalized FTPL and Oracle-Efficient Online Learning}
\label{sec:ftpl}

In this section, we introduce the Generalized Follow-the-Perturbed-Leader
(Generalized FTPL) algorithm and describe the conditions under which it efficiently
reduces online learning to offline optimization.

As described in Section~\ref{sec:intro:ftpl},
we consider the following online learning problem. On each round
$t = 1, \dotsc, T$, a learner chooses an action $x_t$ from a finite
set $\X$, and an adversary chooses an action $y_t$ from a set $\Y$,
which is not necessarily finite. The learner then observes $y_t$ and
receives a payoff $f(x_t, y_t) \in [0,1]$, where the function $f$ is
fixed and known to the learner. The goal of the learner is to obtain
low expected regret with respect to the best action in hindsight,
i.e., to minimize
\[
\regret\coloneqq\E\left[ \max_{x \in \X} \sum_{t=1}^T f(x, y_t) -
  \sum_{t=1}^T f(x_t,
  y_t) \right]
,
\]
where the expectation is over the randomness of the learner.
An online algorithm is called a \emph{no-regret algorithm} if its regret is sublinear in $T$,
  which means that its per-round regret goes to $0$ as $T\rightarrow \infty$.
To
simplify exposition, we assume that the adversary is oblivious, i.e.,
that the sequence $y_1, \ldots, y_T$ is chosen up front without
knowledge of the learner's realized actions.  Our results generalize
to adaptive adversaries using standard
techniques~\citep{HP05,DS16}.

A natural first attempt at an online learning algorithm with oracle
access would be one that simply invokes the oracle on the historical data
at each round and plays the best action in hindsight.  In a stochastic
environment in which the adversary's actions are drawn i.i.d.\ from a
fixed distribution,
this \emph{Follow-the-Leader}
approach achieves a regret of $O(\sqrt{T\log|\X|})$.
However, because the algorithm is deterministic, it performs poorly in
adversarial environments~\cite{blum2007learning}.

To achieve sublinear regret,
we use a common scheme,
introduced by \citet{KV05}, and
optimize over a perturbed objective at each round. Indeed, our algorithm
takes its name from Kalai and Vempala's \emph{Follow-the-Perturbed-Leader} (FTPL) algorithm.
Unlike FTPL, we do not generate a separate independent perturbation for each action,
because this creates the two problems mentioned in Section~\ref{sec:intro:ftpl}.
First, FTPL for unstructured payoffs
requires creating $|\X|$ independent random variables,
which is intractably large in many applications, including the auction
design setting considered here. Second, FTPL yields optimization problems that
require a stronger offline optimizer than assumed here.
 We overcome the first problem
by working with perturbations that are not necessarily independent across different actions (prior instances of such an approach were only known for linear \cite{KV05} and submodular \cite{HK12} minimization).  We address the
second problem by implementing such perturbations with synthetic historical samples of
adversary actions; this idea was introduced by~\citet{DS16}, but they did not provide a method of randomly generating such samples in general learning settings.
Thus, our work unifies and extends these previous lines of research.

We create shared randomness among actions in $\X$
by drawing a random vector $\valpha\in\R^N$ of
size $N$, with components $\alpha_j$ drawn independently from a
dispersed distribution $D$. The payoff of each of the learner's actions
is perturbed by a linear combination of these independent
variables, as prescribed by a {\it perturbation translation matrix} $\mGamma$ of size $|\X|\times N$, with entries
in $[0,1]$.
The rows of $\mGamma$, denoted $\mGamma_x$, describe the linear combination for each action $x$.
That is, on each round $t$, the payoff of each learner action $x\in \X$ is perturbed by $\valpha \cdot \mGamma_x$,
and our Generalized FTPL algorithm
outputs an action $x$ that approximately maximizes
$\sum_{\tau=1}^{t-1} f(x, y_\tau) + \valpha\inprod\mGamma_x$.
This procedure is fully described in Algorithm~\ref{alg:ftpl-approximate}.
(For non-oblivious adversaries, a fresh random vector $\valpha$
is drawn in each round.)

\begin{algorithm}[h!]
  \begin{algorithmic}
    \STATE{Input: matrix $\mGamma\in[0,1]^{|\X | \times N}$, distribution $D$ over $\R$, and optimization accuracy $\epsilon\ge 0$.}
    \STATE{Draw $\alpha_j \sim D$ independently for $j = 1, \ldots, N$.}
    \FOR{$t=1, \ldots, T$}
    \STATE{Choose any $x_t$ such that for all $x\in \X$,
      \[
      \sum_{\tau=1}^{t-1} f(x_t, y_\tau) + \valpha\inprod\mGamma_{x_t} \geq  \sum_{\tau=1}^{t-1} f(x, y_\tau) + \valpha\inprod\mGamma_{x} - \epsilon.
      \]
    }
    \STATE{Observe $y_t$ and receive payoff $f(x_t, y_t)$.}
    \ENDFOR
\end{algorithmic}
\caption{Generalized FTPL}
  \label{alg:ftpl-approximate}
\end{algorithm}

In the remainder of this section, we analyze the properties of matrix $\mGamma$
that guarantee that Generalized FTPL is no-regret
 and that
its perturbations can be efficiently transformed into synthetic history.
Together these properties give rise to efficient reductions of online
learning to offline optimization.

\subsection{Regret Analysis}
\label{sec:regret}

To analyze Generalized FTPL, we first bound its regret
by the sum of a \emph{stability} term, a \emph{perturbation} term, and an \emph{error} term in
the following lemma.
While this approach is standard~\citep{KV05}, we
include a proof in Appendix~\ref{app:approx} for completeness.

\begin{lemma}[$\epsilon$-FTPL  Lemma]
\label{lem:stab+noise_approx}
For Generalized FTPL, we have
\begin{equation} \label{eq:stab+noise_approx}
\regret \leq
\E\left[ \sum_{t=1}^T f(x_{t+1}, y_t) - f(x_t, y_t)  \right] +
\E\left[ \valpha\inprod( \mGamma_{x_1}  -  \mGamma_{x^*} ) \right] + \epsilon(T+1)
\end{equation}
where $x^* = \argmax_{x\in \X} \sum_{t=1}^T f(x, y_t)$.
\label{lem:stab+noise_approx}
\end{lemma}

In this lemma, the first term measures the stability of the
algorithm, i.e., how often the action
changes from round to round. The second term measures the strength of the perturbation, that is,
how much the perturbation amount differs between the best action and the initial action.
The third term measures the aggregated approximation error in choosing $x_t$ that only approximately optimizes   $\sum_{\tau=1}^{t-1} f(x, y_\tau)
 + \valpha\inprod\mGamma_x$.

To bound the stability term, we require that the matrix $\mGamma$ be \emph{admissible} and the distribution
$D$ be \emph{dispersed} in the following sense.
\begin{definition}[$\delta$-Admissible Translation Matrix]\label{defn:admissible}
A translation matrix $\mGamma$ is \emph{admissible} if its rows are distinct. It is \emph{$\delta$-admissible} if it is admissible
and distinct elements within each column differ by at least $\delta$.
\end{definition}

\begin{definition}[$(\rho,L)$-Dispersed Distribution] A distribution
  $D$ on the real line is $(\rho, L)$-dispersed if for any interval of length $L$, the probability measure placed by $D$ on this interval is at most $\rho$.
\end{definition}
In the next lemma, we bound the stability term in \Eq{stab+noise_approx} by showing that with high probability,
for all rounds $t$, we have $x_{t+1} = x_{t}$.
At a high level, since all rows of an admissible matrix $\mGamma$
are distinct, it suffices to show that the probability that $\mGamma_{x_{t+1}}\ne\mGamma_{x_{t}}$
is small. We prove this for each coordinate $\Gamma_{x_{t+1}j}$ separately, by showing
that it is only possible to have
$\Gamma_{x_{t+1} j} \neq \Gamma_{x_{t}j}$ when the random variable
$\alpha_j$ falls in a small interval, which happens with only small
probability for a sufficiently dispersed distribution $D$.\footnote{%
  The proof of Lemma~\ref{lem:stability-approx} implies a slightly tighter bound of $2TN\kappa\rho$,
  where $\kappa$ is the maximum number of distinct elements in any column of $\mGamma$.
  Note that $\delta$-admissibility implies that $\kappa\le 1+\delta^{-1}$.}

\begin{lemma}\label{lem:stability-approx}
Consider Generalized FTPL with a $\delta$-admissible matrix $\mGamma$ with $N$ columns and a $\Parens{\rho,\frac{1 + 2\epsilon}{\delta}}$-dispersed distribution $D$. Then,
$
  \E\left[ \sum_{t=1}^T f(x_{t+1}, y_t)- f(x_t, y_t) \right] \leq 2
  T N \rho(1+\delta^{-1}).
$
\end{lemma}
\begin{proof}
Fix any $t \leq T$.  The bulk of the proof establishes that, with
high probability, $\mGamma_{x_{t+1}} = \mGamma_{x_t}$, which by
admissibility implies that $x_{t+1} = x_t$ and therefore
$f(x_{t+1},y_t) - f(x_t, y_t) = 0$.

Fix any $j \leq N$.  We first show that
$\Gamma_{x_{t+1}j} = \Gamma_{x_t j}$ with high probability.
Let $V$ denote the set of values that appear in the $j^{\text{th}}$ column of $\mGamma$. By $\delta$-admissibility, $\card{V}\le 1+\delta^{-1}$.
For any value $v \in V$, let
$x^v$ be any action that maximizes the perturbed cumulative payoff among those whose $\mGamma$ entry in the $j^{\text{th}}$
column equals~$v$:
\[
x^v \in
\argmax_{x\in \X:\:\Gamma_{xj} = v}
\Bracks{\sum_{\tau=1}^{t-1} f(x,y_\tau)
+\valpha\inprod\mGamma_x}
=
\argmax_{x\in \X:\:\Gamma_{xj} = v}
\Bracks{\sum_{\tau=1}^{t-1} f(x,y_\tau)
+\valpha\inprod\mGamma_x-\alpha_j v}.
\]

For any $v, v' \in V$, define
\[ \Delta_{vv'} = \left( \sum_{\tau=1}^{t-1} f(x^v, y_\tau)
   +\valpha\inprod\mGamma_{x^v} - \alpha_j v
  \right)
   -\left( \sum_{\tau=1}^{t-1} f(x^{v'}, y_\tau)
   +\valpha\inprod\mGamma_{x^{v'}} - \alpha_j v'
  \right).
\]
Note that $x^v$ and $\Delta_{vv'}$ are independent of $\alpha_j$, as we removed the payoff perturbation
corresponding to $\alpha_j$.

If $\Gamma_{x_t j} = v$, then by the $\epsilon$-optimality of $x_t$ on the perturbed cumulative payoff, we have
$\alpha_j(v'- v) - \epsilon \leq \Delta_{v v'}$ for all $v'\in V$.
Suppose $\Gamma_{x_{t+1} j}=v' \neq v$. Then by the optimality of $x^{v'}$ and the $\epsilon$-optimality of $x_{t+1}$,
\begin{align*}
\sum_{\tau=1}^{t-1} f(x^{v'}, y_\tau)
+ f(x_{t+1}, y_t)
+ \valpha\inprod\mGamma_{x^{v'}}
&
\geq
\sum_{\tau=1}^{t-1} f(x_{t+1}, y_\tau) +  f(x_{t+1}, y_t)
+\valpha\inprod\mGamma_{x_{t+1}}
\\
&
\geq
\sum_{\tau=1}^{t-1} f(x^{v}, y_\tau) +  f(x^{v}, y_t)
+\valpha\inprod\mGamma_{x^{v}}
-\epsilon.
\end{align*}
Rearranging, we obtain for this same $v'$ that
\[
\Delta_{v v'} \leq \alpha_j(v' - v)
+ f(x_{t+1}, y_t)
- f(x^{v}, y_t) + \epsilon
\leq \alpha_j(v' - v) + 1 + \epsilon.
\]
If $v' > v$, then
\[
\alpha_j \geq \frac{\Delta_{v v'} - 1-\epsilon}{v' - v}
\geq \min_{\hat{v} \in V,\,\hat{v} >v} \frac{\Delta_{v \hat{v}} - 1-\epsilon}{\hat{v} - v}
\enspace,
\]
and so $\alpha_j(\overline{v} - v) + 1 + \epsilon\geq \Delta_{v \overline{v}}$
where $\overline{v}$ is the value of $\hat{v}$ minimizing the
expression on the right.  Thus, in this case we have
$-\epsilon\le \Delta_{v \overline{v}}-\alpha_j(\overline{v} - v)\le 1+\epsilon$.
Similarly, if $v' < v$, then
\[
\alpha_j \leq \frac{\Delta_{v v'} - 1-\epsilon}{v' - v}
\leq \max_{\hat{v} \in V,\,\hat{v} < v} \frac{\Delta_{v \hat{v}} - 1-\epsilon}{\hat{v} - v}
\enspace,
\]
and so $\alpha_j(\underline{v} - v) + 1 +\epsilon \geq \Delta_{v \underline{v}}$
where $\underline{v}$ is the value maximizing the expression on the right. In this case we have
$-\epsilon\le \Delta_{v \underline{v}}-\alpha_j(\underline{v} - v)\le 1+\epsilon$.
Putting this all together, we have
\begin{align*}
 \Pr&\bigBracks{ \Gamma_{x_{t+1} j} \neq \Gamma_{x_t j} \bigGiven \alpha_k, k\neq j } \\
&\quad \leq  \Pr\BigBracks{
  \exists v \in V:\:  -\epsilon \leq \Delta_{v \overline{v}} - \alpha_j (\overline{v}-v) \leq  1+\epsilon \text{\ \ or\ \ }  {-\epsilon} \leq \Delta_{v \underline{v}}- \alpha_j (\underline{v}-v) \leq 1+\epsilon
  \BigGiven
  \alpha_k, k\neq j
  }\\
&\quad \leq \sum_{v \in V}
  \Parens{
  \textstyle
  \Pr \biggBracks{ \alpha_j \in \left[\frac{\Delta_{v
  \overline{v}}- 1-\epsilon}{\overline{v} - v}, \frac{\Delta_{v
  \overline{v}}+\epsilon}{\overline{v} - v} \right]  \biggGiven \alpha_k, k\neq j
  }
+ \Pr \biggBracks{ \alpha_j \in \left[\frac{- \Delta_{v
  \underline{v}}-\epsilon}{v - \underline{v}}, \frac{-\Delta_{v
  \underline{v}}+1+\epsilon}{v - \underline{v}} \right]  \biggGiven \alpha_k, k\neq j
  }}
\\
&\quad\leq 2\card{V}\rho
      \leq 2\rho(1+\delta^{-1})
.
\end{align*}
The first inequality on the last line follows from the fact that
$\overline{v} - v \geq \delta$ and $v - \underline{v} \geq \delta$,
the fact that $D$ is $\Parens{\rho, \frac{1+2\epsilon}{\delta}}$-dispersed, and a union bound. The final inequality
follows because $\card{V}\le 1+\delta^{-1}$ by $\delta$-admissibility.

Since this bound does not depend on the values of  $\alpha_k$ for $k\neq j$, we can remove the
conditioning and bound
$\Pr[ \Gamma_{x_{t+1} j} \neq \Gamma_{x_{t} j}] \leq 2\rho(1+\delta^{-1})$.
Taking a union bound over all $j\le N$, we then have
that, by admissibility,
$\Pr\left[ x_{t+1} \neq x_{t} \right]= \Pr\left[ \mGamma_{x_{t+1}} \neq \mGamma_{x_{t}} \right] \leq 2 N\rho(1+\delta^{-1})$,
which implies the result.
\end{proof}

To bound the regret, it remains to bound the perturbation term in
Equation~\eqref{eq:stab+noise_approx}. This bound is specific to the
distribution $D$. Many distribution families, including
  (discrete and continuous) uniform, Gaussian, Laplacian, and
  exponential can lead to a sublinear regret when the variance is set appropriately.
Here we present a concrete regret analysis for a uniform distribution:

\begin{theorem}\label{thm:FTPL-U-approx} \label{thm:FTPL-U}
  Let $\mGamma$ be a $\delta$-admissible matrix with $N$ columns  and let $D$ be the uniform distribution on $[0, 1/\eta]$ for $\eta  =  \delta /\sqrt{2T(1+2\epsilon)(1+\delta)}$.
Then, the regret of Generalized FTPL can be bounded as
\[
  \regret \leq  \frac{N\sqrt{T}}{\delta}\cdot 2\sqrt{2(1+2\epsilon)(1+\delta)} \;+\; \epsilon(T+1)
\enspace.
\]
\end{theorem}
The proof follows from Lemmas~\ref{lem:stab+noise_approx} and~\ref{lem:stability-approx},
by bounding the perturbation term by $\norm{\vec\alpha}_1\le N/\eta$,
then setting $\rho = \eta(1+2\epsilon)\delta^{-1}$, and finally using
the value of $\eta$ from the theorem, which
minimizes the sum of the stability and perturbation terms,
$2TN\eta(1+2\epsilon)\delta^{-1}(1+\delta^{-1})+N/\eta$.

Throughout the paper we consider $\epsilon=1/\sqrt{T}$, which is the weakest setting with respect to $T$ that does not negatively impact the regret rate regardless of other problem-dependent constants such as $N$ and $\delta$:

\begin{corollary}\label{cor:FTPL-U}
Let $\mGamma$ be a $\delta$-admissible matrix with $N$ columns, let $D$ be the uniform distribution on $[0, 1/\eta]$ for $\eta  =  \delta /\sqrt{2T(1+2T^{-1/2})(1+\delta)}$, and let $\epsilon=1/\sqrt{T}$.
Then the regret of Generalized FTPL is $\order(N\sqrt{T}/\delta)$.
\end{corollary}

The multiplicative weights (or exponentiated gradient) algorithms achieve the regret $\order(\sqrt{T\log\card{\X}})$ in this setting \citep{FS97,KivinenWa97}. Our bound is generally worse, because $\card{\X}\le(1+\delta^{-1})^N$ by $\delta$-admissibility, which
implies that $\log\card{\X}\le N/\delta$, and therefore
$\order(\sqrt{T\log\card{\X}})\le\order(\sqrt{NT/\delta})\le\order(N\sqrt{T}/\delta)$. However, the multiplicative weights algorithms typically require running times that are polynomial in $\card{\X}$, whereas our algorithm can be exponentially faster assuming access to efficient optimization oracles. We therefore next turn our attention to the analysis of the running time of our algorithm.

\subsection{Oracle-Efficient Online Learning}
\label{sec:imp}

We now define the offline oracle and oracle-efficient online learning
more formally. Our oracles are defined for real-weighted datasets, but can
be easily implemented by integer-weighted oracles (see the reduction in Appendix~\ref{app:integral}).
Since many natural offline oracles are iterative optimization algorithms, which are only guaranteed to return an approximate solution in finite time, our definition assumes that the oracle takes the desired precision $\epsilon$ as an input.
For ease of exposition, we assume that all numerical computations, even those involving real numbers, take $\order(1)$ time. We discuss this point in more detail in Appendix~\ref{app:weak_oracle}.

\begin{definition}[Offline Oracle]
    An \emph{offline oracle} $\opt$ is any algorithm that receives
  as input a weighted set of adversary actions
 $S = \{ (w_\ell,y_\ell) \}_{\ell\in\mathcal{L}}$ with $w_\ell\in\R^+$, $y_\ell\in \Y$ and a desired precision $\epsilon$,
 and returns an action $\hat{x}=\opt(S,\epsilon)$
 such that
 \[
    \sum_{(w,y) \in S} w f(\hat{x}, y)
    \ge
    \max_{x\in \X} \sum_{(w,y) \in S} w f(x, y) - \epsilon.
 \]
\end{definition}

\begin{definition}[Oracle Efficiency] We say that
  an online algorithm is \emph{oracle-efficient}
  with per-round complexity $g(...)$ if its
  per-round running time is $\order(g(...))$ with oracle calls
  counting $O(1)$. The notation $g(...)$ refers to the fact that
  $g$ may be a function of problem-specific parameters, including $T$.
\end{definition}

We next define a property of a translation matrix $\mGamma$ which allows us to transform the perturbed objective into a dataset, thus achieving oracle efficiency of Generalized FTPL:
\begin{definition}
\label{def:imp}
A matrix $\mGamma$ is \emph{implementable} with complexity $M$ if for each $j\in[N]$
there exists a weighted dataset $S_j$, with $|S_j|\leq M$, such that
\begin{align*}
\hspace{1in}
    &\text{for all $x,x'\in\X$:}
    &\Gamma_{xj}-\Gamma_{x'j} &= \sum_{(w,y)\in S_j} w\bigParens{f(x,y)-f(x',y)}.
&\hspace{1in}
\end{align*}
In this case, we say that weighted datasets $S_j$, $j\in[N]$,
implement $\mGamma$ with complexity $M$.
\end{definition}
One simple but useful example of implementability is when
each column $j$ of $\mGamma$ specifies the payoffs of
learner's actions under a particular adversary action $y_j\in\Y$, i.e., $\Gamma_{xj}=f(x,y_j)$.
In this case,
$S_j=\set{(1,y_j)}$.
Using an implementable $\mGamma$ gives rise to an oracle-efficient variant of the Generalized FTPL, provided in Algorithm~\ref{alg:oracleftpl},
in which we explicitly set $\epsilon=1/\sqrt{T}$.
Theorem~\ref{thm:imp} shows that the output of this algorithm is equivalent to the output of Generalized
FTPL
and therefore the same regret guarantees hold.
Note the assumption that the perturbations $\alpha_j$ are
non-negative. The algorithm can be extended to negative perturbations
when both $\mGamma$ and $-\mGamma$ are implementable.
(See \Sec{contexts} for details.)

\begin{algorithm}[t]
  \begin{algorithmic}
  \STATE Input:
                datasets  $S_j$, $j\in[N]$, that implement a matrix $\mGamma\in[0,1]^{\card{\X}\times N}$,
  \STATE \hspace{\algorithmicindent}%
                distribution $D$ over $\R^+$,
  \STATE \hspace{\algorithmicindent}%
                an offline oracle \opt.
    \STATE{Draw $\alpha_j \sim D$ independently for $j = 1, \ldots, N$.}
    \FOR{$t=1, \ldots, T$}
    \STATE{For all $j$, let $\alpha_j S_j$ denote the scaled version of $S_j$, i.e., $\alpha_j S_j \coloneqq\{(\alpha_j w, y): (w,y) \in S_j\}$}.
    \STATE{Set $S=\bigSet{(1,y_1),\dotsc,(1,y_{t-1})}\cup\bigcup_{j\le N} \alpha_j S_j$.}
    \STATE{Play
      $x_t =  \opt\bigParens{S,\frac{1}{\sqrt{T}}}$.
    }
    \STATE{Observe $y_t$ and receive payoff $f(x_t, y_t)$.}
    \ENDFOR
\end{algorithmic}
\caption{Oracle-Based Generalized FTPL}
  \label{alg:oracleftpl}
\end{algorithm}

\begin{theorem}
\label{thm:imp}
\label{thm:adm:imp}
If $\,\mGamma$ is  implementable with complexity $M$,
  then Algorithm~\ref{alg:oracleftpl} is an oracle-efficient implementation
  of Algorithm~\ref{alg:ftpl-approximate} with $\epsilon = 1/\sqrt{T}$ and has per-round complexity $\order\bigParens{T+NM}$.
\end{theorem}
{
\begin{proof}
To show that the Oracle-Based FTPL procedure (Algorithm~\ref{alg:oracleftpl}) implements Generalized FTPL (Algorithm~\ref{alg:ftpl-approximate}) with
$\epsilon = 1/\sqrt{T}$, it suffices to show that at each round $t$, for any $x$,
\begin{align}
\notag
\sum_{\tau=1}^{t-1} f(x, y_\tau) + \valpha\inprod\mGamma_x
&\geq
\max_{x\in \X}\Bracks{\sum_{\tau=1}^{t-1} f(x, y_\tau) + \valpha\inprod\mGamma_x} - \epsilon
\\[2pt]
\label{eq:equalargmax}
&
\iff
\sum_{(w,y) \in S_j} w f(x, y)
\geq
\max_{x\in \X} \sum_{(w,y) \in S_j} w f(x, y) -\epsilon.
\end{align}
Note that if the above equivalence holds, then in each round, the set of actions that $\opt(S, \epsilon)$ can return legally, i.e., actions whose payoffs are an approximation of the optimal payoff, is exactly the same as the set of actions that the offline optimization step of Algorithm~\ref{alg:ftpl-approximate} can legally play. Therefore, the theorem is proved if Equation~\eqref{eq:equalargmax} holds.

Let us show that Equation~\eqref{eq:equalargmax} is indeed true.
For $S=\bigSet{(1,y_1),\dotsc,(1,y_{t-1})}\cup\bigcup_{j\le N} \alpha_j S_j$,
consider any $x,x'\in\X$. Then, from the definition of $S$ and by implementability,
\begin{align*}
\!\!\sum_{(w,y) \in S}\!\!\!w f(x, y)
-
\!\!\!\sum_{(w,y) \in S}\!\!\!w f(x', y)
&=
\sum_{\tau=1}^{t-1} \Bracks{f(x, y_\tau)-f(x',y_\tau)}
+
\;\sum_{j\in [N]}\alpha_j \sum_{(w,y) \in S_j}\!\!\! w \left( f(x, y)-f(x',y) \right)
\\
&=
\sum_{\tau=1}^{t-1} \Bracks{f(x, y_\tau)-f(x',y_\tau)}
+
\sum_{j\in[N]}\alpha_j(\Gamma_{xj}-\Gamma_{x'j})
\\
&=
\Parens{
\sum_{\tau=1}^{t-1} f(x, y_\tau)+ \valpha\inprod\mGamma_x
}
-
\Parens{
\sum_{\tau=1}^{t-1} f(x', y_\tau)+ \valpha\inprod\mGamma_{x'}
},
\end{align*}
which immediately yields Equation~\eqref{eq:equalargmax}.

Also, by implementability, the running time to construct the set $S$ is at most $T+NM$. Since there is only one oracle call per round, we get
the per-round complexity of $T+NM$.
\end{proof}
}

As an immediate corollary, we obtain that the existence of a polynomial-time offline oracle implies the existence of polynomial-time online learner with regret $\order(\sqrt{T})$, whenever we have
access to an implementable and admissible matrix.

\begin{corollary} \label{cor:oracle-ftpl}
Assume that $\mGamma\in[0,1]^{\card{\X}\times N}$ is implementable with complexity $M$ and $\delta$-admissible,
and there exists an approximate offline oracle
$\opt\bigParens{S, \frac{1}{\sqrt{T}}}$ which runs in time $\poly(\card{S},T)$.
Then Algorithm~\ref{alg:oracleftpl} with distribution $D$ as defined in Theorem~\ref{thm:FTPL-U}
runs in time $\poly(N,M,T)$ and achieves
regret $\order(N\sqrt{T}/\delta)$.
\end{corollary}

\paragraph{Alternative Notions of Oracles.}

Throughout this paper we work with offline optimization oracles
that take data with arbitrary non-negative real weights as their input.
In Corollary~\ref{cor:oracle-ftpl}, we then consider oracles of this form
that run in polynomial time. Other notions of oracles may be natural in various
applications. For instance, instead of real-weighted, one can consider integer-weighted
oracles, and instead of polynomial-time, one can consider pseudo-polynomial oracles.
We discuss these alternative
notions in Appendix~\ref{app:weak_oracle}. We show that
integer-weighted oracles can be used to implement approximately optimal real-weighted oracles,
so all of our results immediately extend to integer-weighted oracles. For
pseudo-polynomial oracles, the running time of the algorithm depends
on the pseudo-complexity of the datasets that implement $\mGamma$ (pseudo-complexity is defined in Appendix~\ref{app:pseudo}). In
that case, for instance, the magnitude of the weights implementing
the matrix $\mGamma$ affects the final running time of
the learning algorithm.
In Appendix~\ref{app:pseudo-examples} we show that the pseudo-complexities of matrices $\mGamma$ constructed in the next section are polynomial in the parameters of interest,
so in those cases even pseudo-polynomial offline oracles give rise to polynomial-time no-regret algorithms.

\section{Online Auction Design}
\label{sec:auctions}

In this section, we apply the general techniques developed in
Section~\ref{sec:ftpl} to obtain oracle-efficient
no-regret algorithms for several common auction classes.

Consider a mechanism-design setting in which a seller wants to allocate $k \geq 1$ heterogeneous resources to a set of $n$ bidders. The allocation to a bidder $i$ is a subset of $\{1, \dots, k\}$, which we represent as a vector in $\{0,1\}^k$, and the seller has some feasibility constraints on the allocations across bidders. Each bidder $i\in [n]$ has a combinatorial valuation function $v_i\in \V$, where $\V\subseteq \bigParens{\set{0,1}^k \rightarrow [0,1]}$.
We use $\vec{v}\in\V^n$ to denote the vector of valuation functions across all bidders.
A special case of the setting is that of multi-item auctions for $k$
heterogeneous items, where each resource is an item and the
feasibility constraint simply states that no item is allocated to more
than one bidder. Another special case is that of
\emph{single-parameter (service-based) environments},
which we describe in more detail in \Sec{VCG}.

An auction $a$ takes as input a \emph{bid profile} consisting
of reported valuations for each bidder, and returns both the allocation for each bidder $i$ and the price that he is charged.
In this work, we only consider \emph{truthful auctions}, where each bidder maximizes his utility by reporting his true valuation, irrespective of what other bidders report. We therefore make the assumption that each bidder reports $v_i$ as their bid and refer to $\vec v$ not only as the valuation profile, but also as the bid profile throughout the rest of this section. The allocation that the bidder $i$ receives is denoted
$\vq_i(\vec v)\in\set{0,1}^k$ and
the price that he is charged is $p_i(\vec v)$; we allow sets $\vq_i(\vec v)$
to overlap across bidders, and drop the argument $\vec v$ when it is clear from the context.
We consider bidders with quasilinear utilities: the
utility of bidder $i$ is $v_i(\vq_i(\vec v)) - p_i(\vec v)$.
For an auction $a$ with price
function $\vec p(\cdot)$, we denote by $\rev(a, \vec v)$ the \emph{revenue of the
  auction} for bid profile $\vec v$, i.e.,
$\rev (a, \vec v) = \sum_{i\in[n]} p_i(\vec v)$.

Fixing a class of truthful auctions $\A$ and a set of possible
valuations $\V$, we consider the problem in which on each round
$t = 1, \ldots, T$, a learner chooses an auction $a_t \in \A$ while an
adversary chooses a bid profile $\vec v_t \in \V^n$. The learner then
observes $\vec v_t$ and receives revenue $\rev(a_t, \vec v_t)$. The
goal of the learner is to obtain low expected regret with respect to
the best auction from $\A$ in hindsight.  That is, we would like to
guarantee that
\[ \regret \coloneqq \E\left[ \max_{a \in \A} \sum_{t=1}^T \rev (a, \vec v_t)
  - \sum_{t=1}^T \rev (a_t, \vec v_t) \right] \leq o(T) \poly(n, k).\]
We require our online algorithm to be oracle-efficient, assuming
access to an $\epsilon$-optimal offline optimization oracle that takes as input
a weighted set of bid profiles, $S = \set{(w_\ell,{\vec
    v}_\ell)}_{\ell\in\mathcal{L}}$, and returns
an auction that achieves an approximately optimal revenue on $S$, i.e.,
  a revenue at least $\max_{a\in \A} \sum_{(w,\vec v)\in S} w \rev(a, \vec v)-\epsilon$.
  Throughout the section, we assume that there exists such an oracle for $\epsilon=1/\sqrt{T}$,
  as needed in Algorithm~\ref{alg:oracleftpl}.

Using the language of oracle-based online learning developed in
Section~\ref{sec:ftpl}, the learner's action corresponds to the choice
of auction, the adversary's action corresponds to the choice of bid
profile, the payoff of the learner corresponds to the revenue
generated by the auction, and we assume access to an offline
optimization oracle $\opt$. These correspondences are summarized in
the following table.

\begin{table}[h]
\centering
  \begin{tabular}{| l | l | }
    \hline
    \textbf{Auction Setting} & \textbf{Oracle-Based Learning Equivalent} \\ \hline
    Auctions $a_t \in \A$ & Learner actions $x_t \in \X$ \\ \hline
    Bid/valuation profiles $\vec{v}_t \in \V^n$ & Adversary actions $y_t \in \Y$ \\ \hline
    Revenue function $\rev$ & Payoff function $f$ \\
    \hline
  \end{tabular}
\end{table}

For several of the auction classes we consider, such as multi-item or multi-unit auctions, the revenue of an auction on a bid profile is in range $[0, R]$ for $R > 1$. In order to use  the results of Section~\ref{sec:ftpl}, we implicitly re-scale all the revenue functions  by dividing them by $R$ before applying Theorem~\ref{thm:FTPL-U-approx}.
Note that, since $\mGamma$ does not change, the admissibility condition keeps the regret of the normalized problem at $O(N \sqrt{T}/ \delta)$, according to Theorem~\ref{thm:FTPL-U-approx}.
We then scale up to get a regret bound that is $R$ times the regret for the normalized problem, i.e., $O(R N \sqrt{T} / \delta)$.
This re-scaling does not increase the runtime,
because the complexity of implementing $\mGamma$ is unchanged, only the weights appearing in sets $S_j$ are scaled up by a factor of $R$, and we assume that all numerical computations take $O(1)$ time. Refer to Appendix~\ref{app:weak_oracle} for a note on numerical computations and the mild change in runtime when numerical computations do not take $O(1)$ time.

We now derive results for three auction classes: VCG auctions with bidder-specific reserves,
envy-free item-pricing auctions, and level auctions. Each auction class is formally defined
in its respective subsection.

	\subsection{VCG with Bidder-Specific Reserves}
\label{sec:VCG}

In this section, we consider a standard class of auctions, VCG auctions with
bidder-specific reserve prices, which we define more
formally below and denote by $\I$.
These auctions are  known to approximately maximize the revenue
when bidder valuations are drawn from independent
(but not necessarily identical) distributions~\citep{hartline2009simple}.
Recently,
\citet{roughgarden2016minimizing} considered online learning for this class
and provided a computationally
efficient algorithm whose total revenue is at least $1/2$ of the best
revenue among auctions in~$\I$, minus a term that is
$o(T)$.
We apply the techniques from Section~\ref{sec:ftpl} to generate an
oracle-efficient online algorithm with low \emph{additive} regret with respect to
the optimal auction in the class $\I$, without any loss in multiplicative
factors.

We go beyond single-item auctions and consider general \emph{single-parameter} environments.
In these environments, each bidder has one piece of private
valuation for receiving a \emph{service}, i.e., being included in the set of winning bidders.
We allow for some combinations of bidders to be
\emph{served} simultaneously,  and let $\S \subseteq 2^{[n]}$ be the family
of feasible sets, i.e., sets of bidders that can be served
simultaneously; with some abuse of notation we write $\vec q\in\S$, to mean that the
  set represented by the binary allocation vector $\vec q$ is in~$\S$.
We assume that any bidder is allowed to be the
sole bidder served, i.e., that $\{i\} \in \S$ for all $i$, and that it
is also allowed that no bidder be served, i.e.,
$\emptyset \in \S$.\footnote{%
  A more common and stronger assumption used in previous
  work~\cite{hartline2009simple,roughgarden2016minimizing} is that
  $\S$ is a downward-closed matroid.}  Examples of such environments
include single-item single-unit auctions (for which $\S$ contains only
singletons and the empty set), single-item $s$-unit auctions (for
which $\S$ contains any subset of size at most $s$), and combinatorial
auctions with single-minded bidders. In the last case, we begin with some
set of original items, define the service as receiving the desired
bundle of items, and let $\S$ contain any subset of bidders seeking disjoint
sets of items.

In a basic VCG auction, an allocation $\vec q^*\in\S$ is chosen to
maximize social welfare, that is, maximize $\sum_{i=1}^n v_i q^*_i$,
where we slightly simplify notation and use $v_i \in [0,1]$ to denote
the valuation of bidder $i$ for being served. Each bidder who is
served is then charged the externality he imposes on others,
$p_i(\vec v) = \max_{\vec q\in\S} \sum_{i'\neq i} v_{i'} q_{i'} -
\sum_{i'\neq i} v_{i'} q^*_{i'}$,
which can be shown to equal the minimum bid at which he would be
served. Such auctions are known to be truthful.  The most common
example is the second-price auction for the single-item single-unit
case in which the bidder with the highest bid receives the item and
pays the second highest bid.  VCG auctions with reserves, which
maintain the property of truthfulness, are defined as follows.

\begin{definition}[VCG auctions with bidder-specific reserves]
  A VCG auction with bidder-specific reserves is specified by a vector
  $\vec{r}$ of reserve prices for each bidder. As a first step, all
  bidders whose bids are below their reserves (that is, bidders $i$
  for which $v_i < r_i$) are removed from the auction. If no bidders
  remain, no item is allocated. Otherwise, the basic VCG auction
  is run on the remaining bidders to determine the allocation. Each
  bidder who is served is charged the larger of his reserve and his
  VCG payment.
\end{definition}

Fixing the set $\S$ of feasible allocations, we denote by $\I$ the
class of all VCG auctions with bidder-specific reserves. With a slight
abuse of notation we write $\vec r \in\I$ to denote the auction with
reserve prices $\vec r$.  To apply the results from
Section~\ref{sec:ftpl}, which require a finite action set for the
learner, we limit attention to the finite set of auctions
$\I_m \subseteq \I$ consisting of those auctions in which the reserve
price for each bidder is a strictly positive integer multiple of $1/m$, i.e.,
those where $r_i\in\{1/m,\ldots,m/m\}$ for all $i$.
We will show for some common choices of $\S$ that
the best auction in this class yields almost as high a revenue as the best auction in~$\I$.

We next show how to design a matrix $\mGamma$ for $\I_m$
that is admissible and implementable.  As a warmup, suppose
we use the $|\I_m| \times n$ matrix $\mGamma$ with entries
$\Gamma_{\vec r,i} = \rev(\vec r, \vec e_i)$.
That is, the $i^{\text{th}}$ column of $\mGamma$ corresponds to
the revenue of each auction on a bid profile in which bidder $i$ has
valuation~$1$ and all others have valuation~$0$.
By definition, $\mGamma$ is
implementable with complexity $1$ using $S_j = \{(1, \vec e_j)\}$ for each $j\in[n]$. Moreover, $\rev(\vec r, \vec e_i) = r_i$ so any two
rows of $\mGamma$ are different and $\mGamma$ is thus
$1/m$-admissible. By Theorems~\ref{thm:FTPL-U-approx} and~\ref{thm:imp},
we obtain an oracle-efficient implementation of the Generalized FTPL with
regret $\order(nm/\sqrt{T})$.

To improve this regret bound and obtain a regret that is polynomial in $\log m$ rather than $m$, we carefully construct another translation matrix that is implementable using a more complex dataset of adversarial actions.
The translation matrix we design is quite
intuitive.  The row corresponding to an auction $\vec r$ contains a binary representation of its reserve
prices.  In this case, proving
admissibility of the matrix is simple. The challenge is to show
that this simple translation matrix is implementable using a dataset
of adversarial actions.

\paragraph{Construction of $\,\mGamma$:}
Let $\mGamma^{\textsc{VCG}}$ be an
$|\I_m| \times (n \lceil \log m \rceil)$ binary matrix, where the
$i^{\text{th}}$ collection of $ \lceil \log m \rceil$ columns contains the
binary encodings of the auctions' reserve prices for bidder $i$. More
formally, for any $i\le n$ and a bit position $\bit\le\lceil \log
m\rceil$, let $\col = (i-1) \lceil \log m \rceil + \bit$ and set
$\mGamma^{\textsc{VCG}}_{\vec r,\col }$ to be the $\bit^{\text{th}}$ bit of $mr_i$.

In Lemma~\ref{lem:impl_individual}, we prove that $\mGamma^{\textsc{VCG}}$ is implementable and admissible. But first, let us illustrate the main ideas of this proof  through a simple example.

\begin{figure}
\centering
\begin{tabular}{|p{1.2cm}|c|c|c|c|l}
\cline{2-5}
\multicolumn{1}{c|}{} & \multicolumn{4}{c|}{\scriptsize Binary encoding} & \\
\cline{1-5}
{\scriptsize Auction} $\vec r$ & \multicolumn{2}{c|}{\centering $r_1$} & \multicolumn{2}{c|}{\centering $r_2$} & \\
\cline{1-5}
$\Parens{\nicefrac{1}{3}, \nicefrac{1}{3}}$ & 0\tikzmark{m1} & 1 & 0 & 1 &\\
\cline{1-5}
$\Parens{\nicefrac{1}{3}, \nicefrac{2}{3}}$ & 0 & 1 & 1 & 0 & \scriptsize $\Delta=-1$ \\
\cline{1-5}
$\Parens{\nicefrac{1}{3}, \nicefrac{3}{3}}$ & 0 & 1 & 1 & 1 &\\
\cline{1-5}
$\Parens{\nicefrac{2}{3}, \nicefrac{1}{3}}$ & 1\tikzmark{m2} & 0 & 0 & 1 &\\
\cline{1-5}
$\Parens{\nicefrac{2}{3}, \nicefrac{2}{3}}$ & 1 & 0 & 1 & 0 & \scriptsize $\Delta=0$\\
\cline{1-5}
$\Parens{\nicefrac{2}{3}, \nicefrac{3}{3}}$ & 1 & 0 & 1 & 1 &\\
\cline{1-5}
$\Parens{\nicefrac{3}{3}, \nicefrac{1}{3}}$ & 1\tikzmark{m3} & 1 & 0 & 1\tikzmark{m4} & \scriptsize $\Delta'=1$\\
\cline{1-5}
$\Parens{\nicefrac{3}{3}, \nicefrac{2}{3}}$ & 1 & 1 & 1 & 0\tikzmark{m5} & \scriptsize $\Delta'=-1$\\
\cline{1-5}
$\Parens{\nicefrac{3}{3}, \nicefrac{3}{3}}$ & 1 & 1 & 1 & 1\tikzmark{m6} &\\
\cline{1-5}
\end{tabular}

\begin{tikzpicture}[overlay, remember picture, shorten >=.5pt, shorten <=.5pt]
    \draw[mygray] [<->] ({pic cs:m1}) to[out=350, in=10, looseness=6] ({pic cs:m2});
    \draw[mygray] [<->] ({pic cs:m2}) to[out=350, in=10, looseness=6] ({pic cs:m3});
    \draw[mygray] [<->] ({pic cs:m4}) to[out=340, in=20, looseness=3] ({pic cs:m5});
    \draw[mygray] [<->] ({pic cs:m5}) to[out=340, in=20, looseness=3] ({pic cs:m6});
\end{tikzpicture}
\caption{$\mGamma^{\textsc{VCG}}$ for $n=2$ bidders and discretization $m = 3$.}
\label{fig:table}
\end{figure} 

\begin{example}
Consider $\mGamma^{\textup{\textsc{VCG}}}$ for $n=2$ bidders and $m = 3$ discretization levels, as demonstrated in Figure~\ref{fig:table}. As an example, we show how one can go about implementing columns $1$ and $4$ of $\,\mGamma^{\textup{\textsc{VCG}}}$.

Consider the first column of $\,\mGamma^{\textup{\textsc{VCG}}}$. It
corresponds to the most significant bit of $r_1$.
To implement this column, we
need to find a weighted set of bid profiles that generate revenues
with the same differences as those between the column entries.
We consider bid profiles $\vec
v_h = \left( h/3 , 0\right)$ for $h\in \{1, 2, 3\}$, with
the revenue $\rev(\vec r, \vec v_h)=r_1\mathbf{1}_{(h/3\ge r_1)}$. To
obtain the weights $w_h$ for each $\vec v_h$ it is necessary (and sufficient) to match differences
between entries corresponding to reserve prices with $r_1 = \frac13$~vs~$\frac23$, and  $r_1 = \frac23$~vs~$\frac33$  (denoted by $\Delta$ in Figure~\ref{fig:table}),
corresponding to the following equations:
\begin{align*}
&\frac 13 \left( w_{1} + w_{2} + w_{3} \right) - \frac 23 \left( w_{2} + w_{3} \right) = -1, \\
&\frac 23 \left( w_{2} + w_{3} \right) - \frac 33 \left( w_{3} \right) = 0,
\end{align*}
where the left-hand sides are the differences in the revenues and the right-hand sides are the differences $\Delta$ between the corresponding column entries. Note that the weighted set $S_1 = \{ (3, \vec v_1), (2, \vec v_2), (4, \vec v_3) \}$ satisfies these equations and implements the first column. Similarly, for implementing the fourth column, we consider  bid profiles $\vec v'_h = \left( 0, h/3 \right)$ for $h\in \{1, 2, 3\}$ and equations dictated by the differences $\Delta'$. One can verify that $S_4 = \{ (6, \vec v'_1) , (0, \vec v'_2), (3, \vec v'_3) \}$ implements this column.

More generally, the proof of Lemma~\ref{lem:impl_individual} shows that $\mGamma^{\textsc{VCG}}$ is implementable
by showing that any differences in values in one column that solely depend on a single bidder's reserve price lead to a feasible system of linear equations.
\end{example}

\begin{lemma} \label{lem:impl_individual}
$\mGamma^{\textup{\textsc{VCG}}}$  is   $1$-admissible and
  implementable with complexity $m$.
\end{lemma}
\begin{proof}
  In the interest of readability, we drop the superscript and write
  $\mGamma$ for $\mGamma^{\textsc{VCG}}$ in this proof.

  For any $\vec r$, row $\mGamma_{\vec r}$ corresponds
  to the binary encoding of $r_1, \dots, r_n$. Therefore, for any two
  different auctions $\vec r \neq \vec r'$, we have
  $\mGamma_{\vec r} \neq \mGamma_{\vec r'}$.
  Since $\mGamma$ is a binary matrix, this implies that
  $\mGamma$ is $1$-admissible.

  Next, we will construct the sets $S_j$ that implement each column $j\leq n\lceil \log m\rceil$.
  Pick $i\le n$
  and $\bit\le\lceil\log m\rceil$, and the associated column
  index~$j$.
  The set $S_j$ includes exactly
  the $m$ profiles
  in which only the bidder $i$ has non-zero valuation,
  denoted as $\vec v_\ellOther\coloneqq(\ellOther/m)\vec e_i$ for $\ellOther\le m$.
  To determine their weights $w_\ellOther$, we use the definition of
  implementability. In particular, the weights must satisfy:
\[
\forall~ \vec r, \vec r'\in \I_m,\qquad \Gamma_{\vec r, \col} - \Gamma_{\vec r', \col}  = \sum_{\ellOther\le m}  w_\ellOther
\BigParens{\rev(\vec r, \vec v_\ellOther) - \rev(\vec r', \vec v_\ellOther)}
.
\]

In the above equation, $\Gamma_{\vec r, \col}$ and $\Gamma_{\vec r', \col}$
encode the $\bit^{\text{th}}$ bit of $r_i$ and $r'_i$, respectively, so the
left-hand side is independent of the reserve prices for bidders
$i'\neq i$.  Moreover,
$\rev(\vec r, \vec v_\ellOther) = r_i \mathbf{1}_{(\ellOther \geq mr_i )}$,
so the right-hand side of the above equation is also independent of
the reserve prices for bidders $i'\neq i$.  Let $z_\bit$ be the $\bit^{\text{th}}$ bit
of integer $z$. That is, $\Gamma_{\vec r, \col} = (m r_i)_\bit$. Substituting $z=m r_i$ and $z'=m r_i'$,
the above equation can be reformulated as
\footnote{Not including the reserve 0 is a crucial technical point for the proof of implementability.}
\begin{equation}\label{eq:claim}
\forall z, z'\in\{1, \dots, m\},\quad \left(z_\bit - z'_\bit \right)  = \sum_{\ellOther\le m}  w_\ellOther \left( \frac z m \mathbf{1}_{(\ellOther \geq z)} - \frac{z'}{m} \mathbf{1}_{( \ellOther \geq z')} \right).
\end{equation}
We recursively derive the weights $w_\ellOther$, and show that they are non-negative and satisfy \Eq{claim}. To begin,
let
\[
 w_m = \max\BigBraces{0,\;\max_z\bigBracks{ m\bigParens{z_\bit - (z-1)_\bit}}
 },
\]
and for all $z = m, m-1,\dotsc, 2$, define
\[
 w_{z-1} = \frac{1}{z-1} \left( \sum_{\ellOther = z}^m w_\ellOther  - m\bigParens{z_\bit - (z-1)_\bit} \right).
\]
Next, we show by induction that $w_h\geq 0$ for all $h$. For the base case of $h=m$, by definition $w_m \geq 0$. Now, assume that for all $\ellOther \geq
z$, $w_\ellOther \geq 0$. Then
\[
w_{z-1} \geq \frac{1}{z-1} \BigParens{
  w_m -  m\bigParens{z_\bit
  - (z-1)_\bit}
  }
  \geq 0.
\]
Therefore all weights are non-negative.
Furthermore, by rearranging the definition of $w_{z-1}$, we have
\begin{align*}
\bigParens{z_\bit - (z-1)_\bit}&= \frac 1m \left(   \sum_{\ellOther = z}^m  w_\ellOther  - (z-1) w_{z-1}  \right) =
\frac 1m \left(    z\sum_{\ellOther = z}^m  w_\ellOther  - (z-1) \sum_{\ellOther = z-1}^m  w_\ellOther \right) \\
&=
 \sum_{\ellOther\le m}  w_\ellOther \left( \frac zm \mathbf{1}_{(\ellOther \geq z)} - \frac{z-1}{m} \mathbf{1}_{( \ellOther \geq z-1)} \right),
\end{align*}
where in the second equality we simply added and subtracted the term $(z-1)\sum_{h=z}^{m}w_h$ and in the last equality, we grouped together common terms.

\Eq{claim} is proved for a particular pair $z>z'$
by summing the above expression for $\bigParens{\zeta_\bit - (\zeta-1)_\bit}$ over all
$\zeta \in (z', z]$ and canceling telescoping terms,
and if $z=z'$, the statement holds regardless of the weights chosen.

This shows that $\mGamma$ is implementable. Note that the  cardinality of each $S_j$ is $m$, so $\mGamma$ is implementable with complexity $m$.
\end{proof}

The next theorem follows immediately from Lemma~\ref{lem:impl_individual}, Theorems~\ref{thm:FTPL-U-approx} and~\ref{thm:imp}, and the fact that the maximum revenue is at most $R$.
Note that $R$ is bounded by the number of bidders that can be served simultaneously, which is at most $n$.
\begin{theorem}\label{thm:I_m}
Consider the online auction design problem for  the class of  VCG auctions with bidder-specific reserves, $\I_m$. Let $R = \max_{\vec r,\vec v} \rev(\vec r, \vec v)$ and let $D$ be the uniform distribution as described in Theorem~\ref{thm:FTPL-U}.
Then,
the Oracle-Based Generalized FTPL algorithm with $D$ and
datasets that implement  $\mGamma^{\textup{\textsc{VCG}}}$
is oracle-efficient with per-round complexity
$\order(T+n m\log m)$
and  has regret
\[ \E\left[ \max_{\vec r \in \I_m}  \sum_{t=1}^T \rev( \vec r, \vec v_t) - \sum_{t=1}^T \rev(\vec r_t, \vec v_t) \right] \leq  O(n R \sqrt{T}\log m).\]
\end{theorem}

Now we return to the infinite class $\I$ of all VCG auctions with
reserve prices $r_i\in [0,1]$. We show that $\I_m$ is a finite ``cover''
for this class when the family of feasible sets $\S$ consists of all
subsets of size at most $s$, corresponding to single-item single-unit
auctions (when $s = 1$) or more general single-item $s$-unit auctions.
In such auctions, the items are allocated to the $s$ highest
bids that are above their reserve, and each winner pays the larger of its reserve price and the $s+1^\text{st}$ highest bid that had cleared its respective reserve price. We assume that the ties are resolved
in favor of bidders with a lower index.
We prove in Appendix~\ref{app:Im=I} that for these auctions,
the optimal revenue of $\I_m$ compared with that of $\I$ can decrease by at most $s/m$ at each round. That is,
\begin{equation} \label{eq:Im=I}
  \max_{\vec r \in \I}  \sum_{t=1}^T \rev (\vec r, \vec v_t)  - \max_{\vec r \in \I_m}  \sum_{t=1}^T \rev (\vec r, \vec v_t) \leq \frac {T s}{m}.
\end{equation}
Setting $m = \sqrt{T}$ and using Theorem~\ref{thm:I_m}, we obtain the following result for the class of auctions $\I$.
\begin{theorem}\label{thm:I}
Consider the online auction design problem for  the class of VCG auctions with bidder-specific reserves, $\I$, in $s$-unit auctions. Let $D$ be the uniform distribution as described in Theorem~\ref{thm:FTPL-U}.
Then,
the Oracle-Based Generalized FTPL algorithm with $D$ and
datasets  that implement $\mGamma^{\textup{\textsc{VCG}}}$
is oracle-efficient with per-round complexity
$\order(T+n\sqrt{T}\log T)$
and has regret
\[ \E\left[ \max_{\vec r \in \I}  \sum_{t=1}^T \rev( \vec r, \vec v_t) - \sum_{t=1}^T \rev(\vec r_t, \vec v_t) \right] \leq  O(n s \sqrt{T}\log T).\]
\end{theorem}

	\subsection{Envy-free Item Pricing}
\label{sec:envy-free}

In this section, we consider envy-free item pricing~\cite{guruswami2005profit} in an environment with $k$ heterogeneous items with a supply of $s_\ell \geq 0$ units for each item $\ell\le k$.

\begin{definition}[Envy-free Item-Pricing Auction]\label{def:envy-free-item}
  An \emph{envy-free item-pricing auction} for $k$ heterogeneous items, given
  supply $s_\ell$ for $\ell=1,\dotsc,k$, is defined by a vector of prices $\vec a$,
  where $a_\ell$ is the price of item $\ell$. The mechanism considers
  bidders $i=1, \dots, n$ in order and allocates to bidder $i$ the
  bundle $\vq_i\in\set{0,1}^k$ that maximizes $v_i(\vq_i) - \va\cdot\vq_i$,
  among all feasible bundles, i.e., bundles that can be composed
  from the remaining supplies. Bidder $i$ is then charged the price $\va\cdot\vq_i$.
\end{definition}

Examples of such environments include \emph{unit-demand} bidders
and \emph{single-minded} bidders
in settings such as \emph{hypergraph pricing},
where bidders seek hyperedges in a hypergraph, and its variant
\emph{the highway problem}, where bidders seek hyperedges  between
sets of contiguous
vertices~\cite{balcan2006approximation,guruswami2005profit}.

We represent by $\mathcal{P}_m$ the class of all such envy-free item-pricing auctions where
all the prices are strictly positive multiples of $1/m$, i.e.,
$a_\ell\in \{1/m,\ldots, m/m\}$
for all $\ell$.
Next, we discuss the construction of an implementable and admissible translation matrix $\mGamma$. Consider a bid profile where one bidder has value $v$ for bundle $\vec e_\ell$ and all other bidders have value $0$ for all bundles.
The revenue of auction $\vec a$ on such a bid profile is $a_\ell \mathbf{1}_{(v \geq a_\ell)}$.
Note the similarity to the case of VCG auctions with bidder-specific reserve prices $\vec r$, where
bid profiles with a single non-zero valuation $v_i$ yielding the revenue $r_i \mathbf{1}_{(v_i\geq r_i)}$ were used to create an implementable construction for $\mGamma$.
We show that a similar construction works for $\P_m$.

\paragraph{Construction of $\,\mGamma$:}
Let $\mGamma^{\textsc{IP}}$ be a $|\P_m| \times (k \lceil \log m \rceil)$ binary matrix, where the $\ell^{\text{th}}$ collection of $ \lceil \log m \rceil$ columns correspond to the binary encoding of the auction's price for item $\ell$.
More
formally, for any $\ell\le k$ and $\beta\le\lceil\log m\rceil$,
$\Gamma^{\textsc{IP}}_{\vec a, \col}$ is the $\bit^{\text{th}}$ bit of (the integer)
$m a_\ell$, where $\col = (\ell-1) \lceil \log m \rceil +  \bit$.
Next, we show that  $\mGamma^{\textsc{IP}}$ is admissible and implementable.
The proof of the following lemma is analogous to that of Lemma~\ref{lem:impl_individual} and appears in Appendix~\ref{app:impl_envy} for completeness.

\begin{lemma} \label{lem:impl_envy}
$\mGamma^{\textup{\textsc{IP}}}$  is   $1$-admissible and
  implementable with complexity $m$.
\end{lemma}

Our main theorem follows immediately from Lemma~\ref{lem:impl_envy}, Theorems~\ref{thm:FTPL-U} and \ref{thm:imp}, and  the fact that the revenue of the mechanism at every step is at most $R$. In general, $R$ is at most $n$.

\begin{theorem}\label{thm:envy}
Consider the online auction design problem for the class of envy-free item-pricing auctions,  $\P_m$. Let $R = \max_{\vec a,\vec v} \rev(\vec a, \vec v)$ and let $D$ be the uniform distribution as described in Theorem~\ref{thm:FTPL-U}.
Then,
the Oracle-Based Generalized FTPL algorithm with $D$ and
datasets that implement  $\mGamma^{\textup{\textsc{IP}}}$
 is oracle-efficient with per-round complexity
$\order(T+km\log m)$
and has regret
\[ \E\left[ \max_{\vec a \in \mathcal{P}_m}  \sum_{t=1}^T \rev( \vec a, \vec v_t) - \sum_{t=1}^T \rev(\vec a_t, \vec v_t) \right] \leq  O\left(k R\sqrt{T}\log m \right).\]
\end{theorem}

Now consider the class of  all envy-free item-pricing auctions where $a_\ell \in[0,1]$ is a real number and denote this class by $\mathcal{P}$.
We show that $\P_m$ is a discrete  ``cover'' for $\P$ when there is an \emph{unlimited supply} of all items ($s_\ell = \infty$ for all $\ell$) and the bidders have \emph{single-minded} or \emph{unit-demand} valuations. In the single-minded setting, each bidder $i$ is interested in one particular bundle of items $\hat\vq_i$. That is, $v_i(\vq_i) = v_i(\hat{\vq}_i)$ for all $\vq_i \supseteq \hat{\vq}_i$ and $0$ otherwise. In the unit-demand setting, each bidder $i$ has valuation $v_{i}(\vec e_\ell)$ for item $\ell$, and wishes to purchase \emph{at most one item}, i.e., item $ \argmax_{\ell} \left( v_i(\vec e_\ell) - a_\ell \right)$. We show that in both settings,
discretizing item prices cannot decrease the revenue by much
(see Appendix~\ref{app:P-grid}).

\begin{lemma}\label{lem:P-grid}
For any $\vec a \in \P$ there is $\vec a'\in \P_m$, such that for any unit-demand  valuation profile $\vec v$ with infinite supply (the digital goods setting), $\rev(\vec a, \vec v) - \rev(\vec a', \vec v) \leq nk/m$.
Similarly, there is  $\vec a'\in \P_m$, such that for any single-minded  valuation profile $\vec v$ with infinite supply, $\rev(\vec a, \vec v) - \rev(\vec a', \vec v) \leq nk^2/m$.
\end{lemma}

These discretization arguments together with Theorem~\ref{thm:envy}
yield the following result for the class of auctions $\P$
(using the fact that $R\leq n$, and setting $m = \sqrt{T}$ for the unit-demand and $m = k\sqrt{T}$ for the single-minded setting):

\begin{theorem}\label{thm:envy-unit}\label{thm:envy-signle}
Consider the online auction design problem
for the class of envy-free item-pricing auctions, $\P$, with unit-demand bidders with infinite supply (the digital goods setting).
Let $D$ be the uniform distribution as described in Theorem~\ref{thm:FTPL-U}.
Then,
the Oracle-Based Generalized FTPL algorithm with $D$ and
datasets  that implement $\mGamma^{\textup{\textsc{IP}}}$
 is oracle-efficient with per-round complexity
$\order\bigParens{T+k\sqrt{T}\log T}$
 and has regret
\[ \E\left[ \max_{\vec a \in \mathcal{P}}  \sum_{t=1}^T \rev( \vec a, \vec v_t) - \sum_{t=1}^T \rev(\vec a_t, \vec v_t) \right] \leq O\left(nk\sqrt{T}\log T \right).\]
Similarly, for single-minded bidders with infinite supply,
the Oracle-Based Generalized FTPL algorithm with $D$ and
datasets  that implement $\mGamma^{\textup{\textsc{IP}}}$
is oracle-efficient with per-round complexity
$\order\bigParens{T+k^2\sqrt{T}\log (kT)}$
 and has regret
\[ \E\left[ \max_{\vec a \in \mathcal{P}}  \sum_{t=1}^T \rev( \vec a, \vec v_t) - \sum_{t=1}^T \rev(\vec a_t, \vec v_t) \right] \leq O\left(n k\sqrt{T}\log(kT) \right).\]
\end{theorem}

\subsection{Level Auctions}
\label{sec:level}

We next consider the class of
\emph{level auctions} introduced by~\citet{morgenstern2015pseudo},
who show that these auctions can achieve $(1{-}\epsilon)$-approximate revenue maximization if the valuations of the bidders are drawn independently (but not necessarily identically) from a distribution,
thus approximating Myerson's optimal auction~\cite{Myerson1981}.
Using our tools, we derive oracle-efficient no-regret algorithms for this auction class.

The $s$-level auctions realize a single-item single-unit allocation as follows:
\begin{definition}\label{def:s-level}
Given $s\ge 2$, an $s$-level auction $\vec \theta$ is defined by $s$ thresholds for each bidder $i$, $0 \leq \theta^i_0  \leq \dots \leq \theta^i_{s-1}\le 1$.
For any bid profile $\vec v$,
we let $b_i^{\vec \theta}(v_i)$ denote the largest index $b$ such that $\theta^i_b\le v_i$, or $-1$ if $v_i < \theta^i_0$.
If $v_i<\theta^i_0$ for all $i$, the item is not allocated.
Otherwise, the item goes to the bidder with the largest index $b^{\vec \theta}_i(v_i)$, breaking ties in favor of bidders with smaller $i$.
The winner pays the price equal to the minimum bid that he could have submitted and still won the item.
\end{definition}

When it is clear from the context, we omit $\vec \theta$ in $b^{\vec \theta}_i(v_i)$ and write just $b_i(v_i)$.
In the remainder of the section we assume that $n\ge 2$.
For $n=1$, the level auctions are equivalent to the second-price auctions with reserves, so we can just appeal to the bounds from previous sections (specifically, Theorems~\ref{thm:I_m} and~\ref{thm:I} from \Sec{VCG}).

We consider two classes of
$s$-level auctions, $\mathcal{R}_{s,m}$ and $\mathcal{S}_{s,m}$, where $\mathcal{R}_{s,m}$ is the set of all auctions described by Definition~\ref{def:s-level} with
thresholds in the set $\{0,\, 1/m,\, \dotsc,\, m/m\}$ and $\mathcal{S}_{s,m}$ is the subset of $\mathcal{R}_{s,m}$ containing the auctions in which the thresholds for each bidder $i$ are distinct.

We first consider  $\mathcal{S}_{s,m}$. To construct an admissible
and implementable $\mGamma$ for $\mathcal{S}_{s,m}$,
we begin with a matrix that is clearly implementable, with each column implemented by a single bid profile, and then show its admissibility.

We consider the bid profiles in
which the only non-zero bids are $v_n = \ell/m$ for some $0\le\ell\le m$,
and $v_i = 1$ for a single bidder $i < n$.  Note that bidder $i$
wins the item in any such profile and pays $\theta^i_b$ corresponding to $b = \max\{0, b_n(v_n)\}$.
We define a matrix $\mGamma$ with one column for every bid profile of this form and an additional column for the bid profile $\vec e_n$, with the entries in each row consisting of the revenue of the corresponding auction on the given bid profile. Clearly, $\mGamma$ is implementable.
As for
admissibility, take $\vec \theta \in \S_{s, m}$ and the corresponding
row $\mGamma_{\vec \theta}$. Note that as $v_n = \ell/m$ increases
for $\ell=0, \dots, m$, there is an increase in
$b_n(\ell/m) = -1,0, \dotsc, s-1$, possibly skipping the initial $-1$.
As the level $b_n(v_n)$ increases, the auction revenue attains the values $\theta^i_0$, $\theta^i_1,\dotsc,\theta^i_{s-1}$, changing exactly at those points where $v_n$ crosses thresholds $\theta^n_1, \dots, \theta^n_{s-1}$.
Since any
two consecutive thresholds of $\vec \theta$ are different,  the
thresholds of $\theta^i_b$ for $b\ge 0$ and $\theta^n_b$ for $b\ge 1$ can be reconstructed by analyzing the revenue
of the auction and the values of $v_n$ at which the revenue changes. The remaining threshold $\theta^n_0$
is equal to the revenue of the bid profile $\vec v=\vec e_n$.
Since all of the parameters of the auction can be recovered from the entries in the row $\mGamma_{\vec \theta}$,
this shows
that any two rows of $\mGamma$ are different and $\mGamma$ is
$1/m$-admissible. This reasoning is summarized in the following construction and the corresponding lemma, formally proved
in Appendix~\ref{app:admissible-Sm}. See Figure~\ref{fig:imp_SL} for more intuition.

\begin{figure}
\centering
\includegraphics[width =  \textwidth]{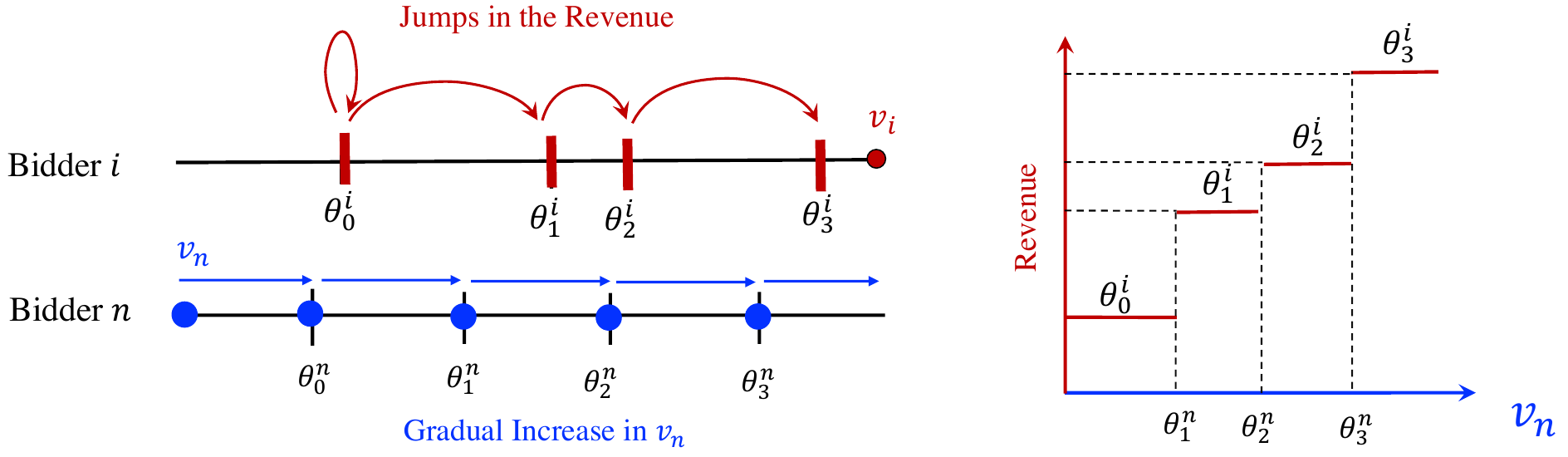}
\put(-416, 21){\small{$0$}}
\put(-209, 70){\small{$1$}}
\caption{Demonstration of how $\vec \theta$ can be reconstructed by its revenue on the bid profiles in $V = \set{ \vec v^{i,\ell}}_{i,\ell} \cup \{ \vec e_n \}$.
On the left, we show that as the value $v_{n}$ (blue circle) gradually increases from $0$ to $1$, the revenue of the auction (red vertical lines) jumps along the sequence of values $\theta^i_0, \theta^i_1, \dots,\theta^i_{s-1}$.
So by analyzing the revenue of an auction on all bid profiles $\set{ \vec v^{i,\ell}}_{i,\ell}$ one can reconstruct $\vec \theta^i$ for $i\neq n$ and $\theta^n_1, \dots, \theta^n_{s-1}$.
To reconstruct $\theta_0^n$, one only needs to consider the profile $\vec e_n$.
The figure on the right demonstrates the revenue of the same auction, where the horizontal axis is the value of $v_n$ and the vertical axis is the revenue of the auction when $v_i=1$ and all other valuations are $0$.
}
\label{fig:imp_SL}
\end{figure}

\paragraph{Construction of $\,\mGamma$:}
For $i\in\set{1,\dotsc,n-1}$ and $\ell\in\set{0,\dotsc,m}$, let $\vec v^{i,\ell}=\vec e_i + (\ell/m)\vec e_n$. Let $V = \set{ \vec v^{i,\ell}}_{i,\ell} \cup \{ \vec e_n \}$.
Let $\mGamma^{\textsc{SL}}$ be the matrix of size $|\S_{s,m}| \times\card{V}$ with entries indexed by $(\vec\theta,\vec v)\in\S_{s,m}\times V$,
such that
$\Gamma^{\textsc{SL}}_{\vec \theta, \vec v} = \rev( \vec \theta, \vec v)$.
\begin{lemma}\label{lem:admissible_Sm}
$\mGamma^{\textup{\textsc{SL}}}$  is   $1/m$-admissible and
  implementable with complexity $1$.
\end{lemma}

Our next theorem is an immediate  consequence of Lemma~\ref{lem:admissible_Sm}, Theorems~\ref{thm:FTPL-U} and \ref{thm:imp}, and the fact that the revenue of the mechanism in each round is at most $1$.

\begin{theorem}\label{thm:S_m,s}
Consider the online auction design problem
for the class $\mathcal{S}_{s,m}$ of $s$-level auctions.
Let $D$ be the uniform distribution as described in Theorem~\ref{thm:FTPL-U}.
Then,
the Oracle-Based Generalized FTPL algorithm with $D$  and
datasets that implement $\mGamma^{\textsc{SL}}$
is oracle-efficient with per-round complexity
$\order(T+nm)$
and has regret
\[ \E\left[ \max_{\vec \theta \in \S_{s,m}}  \sum_{t=1}^T \rev (\vec \theta, \vec v_t) - \sum_{t=1}^T \rev(\vec \theta_t, \vec v_t) \right] \leq  O(n m^2 \sqrt{T}).\]
\end{theorem}

Next, we turn our attention to  the class of auctions $\mathcal{R}_{s,m}$ and construct an admissible and implementable matrix for this auction class. As a warm-up, we demonstrate how the structure of $\mathcal{R}_{s,m}$ differs from $\mathcal{S}_{s,m}$ and argue that $\mGamma^{\textsc{SL}}$ is not admissible for $\mathcal{R}_{s,m}$. Take $\vec v$ such that $v_i = 1$ and $v_n$ is some multiple of $1/m$. Similarly as before, as $v_n$ increases there is also an increase in $b_n(v_n)$ and the revenue $\smash{\theta^i_{b_n(v_n)}}$.  However, the thresholds of bidder
$n$ are no longer required to be distinct, so $b_n(v_n)$ can skip certain values, and as a result $\theta^i_b$ is not revealed for all $b$. In addition, when the thresholds of bidder $i$ are not distinct, we have that even when $b_n(v_n)$ strictly increases, there might be no increase in the revenue, and as a result $\theta^n_b$  cannot be reconstructed.

We next show how to construct an admissible matrix for $\mathcal{R}_{s,m}$.
In preparation for this construction, we first define an equivalence relation among auctions which puts auctions with distinct parameters but identical outcomes in the same equivalence class. Specifically, we say that two auctions $\vec \theta$ and $\vec \theta'$ are \emph{equivalent} if for any bid profile $\vec v$, $\rev(\vec \theta, \vec v) = \rev(\vec \theta', \vec v)$, i.e., they receive the same revenue on all bid profiles. Note that it is sufficient to limit our attention to a subset of auctions $\barCalR_{s,m}\subseteq \mathcal{R}_{s,m}$ that includes exactly one auction from each equivalence class. The regret with respect to $\barCalR_{s,m}$ is the same as the regret with respect to $\mathcal{R}_{s,m}$. Also, when applying Algorithm~\ref{alg:oracleftpl}, any result $x_t$ returned by an offline oracle for $\mathcal{R}_{s,m}$ can be replaced by the equivalent $\bar{x}_t\in\barCalR_{s,m}$ without affecting the algorithm's regret. This means that an offline oracle for $\mathcal{R}_{s,m}$ can be viewed as an offline oracle for $\barCalR_{s,m}$ and used together with an admissible and implementable matrix $\mGamma^\barRL$ for the auction class $\barCalR_{s,m}$
to solve the online optimization problem for the larger but equivalent auction class $\mathcal{R}_{s,m}$.
In the following, we introduce an implementable and admissible construction for $\barCalR_{s,m}$.

\paragraph{Construction of $\,\mGamma$:}
Let $V = \{\vec v \mid \vec v\in \{0,  1/m, \dots,  m/m\}^n \text{ and } \|\vec v\|_0 \leq 3 \}$, where $\lVert\cdot\rVert_0$ denotes the number of non-zero entries of a vector. Let $\mGamma^\barRL$ be the matrix of size $|\barCalR_{s,m} | \times \card{V}$ with entries indexed by $(\vec \theta, \vec v) \in \barCalR_{s,m} \times V$, such that $\Gamma^\barRL_{\vec \theta, \vec v} = \rev( \vec \theta, \vec v)$.
\begin{lemma}\label{lem:admissible_Rm}
$\mGamma^\barRL$  is   $1/m$-admissible and
  implementable with complexity $1$ for the class of auctions $\barCalR_{s,m}$.
\end{lemma}
\begin{proof}
Clearly, $\mGamma^\barRL$ is implementable with complexity $1$ for the class  of auctions $\barCalR_{s,m}$.
We show that $\mGamma^\barRL$ is also admissible.
Take any two auctions $\vec \theta$ and $\vec \theta'$ in $\barCalR_{s,m}$. Since these auctions are not equivalent, there is a bid profile $\vec v$ such that $\rev(\vec \theta, \vec v) \neq \rev (\vec \theta', \vec v)$. Without loss of generality, assume that $\rev(\vec \theta, \vec v) < \rev (\vec \theta', \vec v)$.
In what follows, we construct a corresponding bid profile $\vec v' \in V$ such that  $\rev(\vec \theta, \vec v') \neq \rev (\vec \theta', \vec v')$. Let $i$ and $i'$ be the winners in the auctions $\vec \theta$ and $\vec \theta'$, respectively. Let $j$ and $j'$ be the bidders with second highest bucketed bids, i.e., the bidders who set prices in auctions $\vec \theta$ and $\vec \theta'$, respectively.  We consider two cases:

If $i, i', j$, and $j'$ are not distinct, we construct the bid profile $\vec v'$
that is the same as $\vec v$ on indices $i$, $i'$, $j$, and $j'$, and is 0 otherwise.
Note that $\vec v'$ has at most $3$ non-zero elements, hence $\vec v' \in V$. Moreover,
the bids of the winners and the price setters are the same in both bid profiles, so the allocations and payments of both auctions remain the same. That is,
\[ \rev(\vec \theta, \vec v')= \rev(\vec \theta, \vec v) \neq \rev (\vec \theta', \vec v)  =\rev(\vec \theta', \vec v').
\]
\indent If $i, i', j$, and $j'$ are all distinct,  we construct the bid profile $\vec v'\in V$ that
is the same as $\vec v$ on indices $i$, $i'$, and $j'$, and is 0 otherwise (including on index $j$).
On auction $\vec \theta'$, since the bids of the winner and the price setter, $i'$ and $j'$ are the same in both bid profiles, the allocation and payment of both auctions remain the same, and we have
$\rev(\vec \theta', \vec v')= \rev(\vec \theta', \vec v)$.
On auction $\vec \theta$, the winner's bid remains the same and the price setter's bid is equal or lowered, so
$\rev(\vec \theta, \vec v') \leq  \rev(\vec \theta, \vec v)$. Hence,
\[ \rev(\vec \theta, \vec v') \leq  \rev(\vec \theta, \vec v) < \rev (\vec \theta', \vec v)  =\rev(\vec \theta', \vec v').
\]
Since the revenue of any auction is a multiple of $1/m$ and any two rows of $\mGamma^\barRL$ differ in at least one entry by $1/m$, we obtain that $\mGamma^\barRL$ is
$1/m$-admissible.
\end{proof}

Our next theorem is an immediate  consequence of Lemma~\ref{lem:admissible_Rm}, Theorems~\ref{thm:FTPL-U} and \ref{thm:imp}, the equality of regrets with respect to $\barCalR_{s,m}$ and $\mathcal{R}_{s,m}$,
and the fact that the revenue of the mechanism in each round is at most $1$. As discussed above, Algorithm~\ref{alg:oracleftpl} can use an offline oracle for $\mathcal{R}_{s,m}$ in place of
an offline oracle for $\barCalR_{s,m}$.
\begin{theorem}\label{thm:R_m,s}
Consider the online auction design problem
for the class $\mathcal{R}_{s,m}$ of $s$-level auctions.
Let $D$ be the uniform distribution as described in Theorem~\ref{thm:FTPL-U}.
Then,
the Oracle-Based Generalized FTPL algorithm with $D$  and
datasets that implement $\mGamma^\barRL$
is oracle-efficient with per-round complexity
$\order(T+ n^3m^3)$
and has regret
\[ \E\left[ \max_{\vec \theta \in \mathcal{R}_{s,m}}  \sum_{t=1}^T \rev (\vec \theta, \vec v_t) - \sum_{t=1}^T \rev(\vec \theta_t, \vec v_t) \right] \leq  O(n^3 m^4 \sqrt{T}).\]
\end{theorem}

\section{Stochastic Adversaries and Universal Benchmarks}
\label{sec:overallopt}

So far our results apply to general adversaries, where the sequence of adversary actions is arbitrary, and where we can achieve the payoff that is close to the payoff of the best action in our class. For auctions, we might be interested in the comparison with an arbitrary auction rather than an auction that is in our class. In this section, we show that in certain cases this can be achieved if we impose distributional assumptions on the sequence of the adversary.

We start with the easier setting where the actions of the adversary are drawn i.i.d. across all rounds  and then we analyze the slightly more complex setting where the actions of the adversary follow a fast-mixing Markov chain. For both settings we show that the average payoff of the learning algorithm is close to the optimal expected payoff, where expectation is with respect to the adversary's distribution in the i.i.d.\ setting and the adversary's stationary distribution in the Markovian setting.

In online auction design, we can combine these results with approximate optimality results of simple auctions, such as
$s$-level auctions or VCG with bidder-specific reserves, to prove universal optimality of our online learning algorithms in these distributional settings
rather than only the optimality among the auctions within the class over which our algorithms are learning.

\subsection{Stochastic Adversaries}

\paragraph{I.I.D.\ Adversary.} When
the adversary actions are drawn independently from the same unknown distribution $F$, we can use Chernoff-Hoeffding bound to show that the average payoff of a no-regret learner converges to the best payoff one could achieve in expectation over the distribution $F$ (see Appendix~\ref{app:iid} for the proof):

\begin{lemma}\label{lem:iid}
Suppose that $y_1,\ldots,y_T$ are i.i.d.\ draws from a distribution
$F$. Then for any no-regret learning algorithm, with probability at least $1-\delta$,
\[
\frac{1}{T}\sum_{t=1}^T \E_{x_t}[f(x_t,y_t)] \geq \sup_{x\in \X}
\E_{y\sim F}[f(x,y)]- \sqrt{\frac{\log(1/\delta)}{ 2 T}} - \frac{\regret}{T}.
\]
\end{lemma}

\paragraph{Markovian Adversary.} Suppose that the choice of the adversary $y_t$ follows a stationary and reversible Markov process based on some transition matrix $P(y,y')$ with a stationary distribution $F$. Moreover, consider the case where the set $\Y$ is finite. For any Markov chain, the spectral gap $\gamma$ is defined as the difference between the first and the second largest eigenvalue of the transition matrix $P$ (the first eigenvalue always being $1$). We will assume that this gap is bounded away from zero. The spectral gap of a Markov chain is strongly related to its mixing time. In this work we will specifically use the following result of \citet{Paulin2015}, which is a Bernstein concentration inequality for sums of dependent random variables sampled from a stationary Markov chain with spectral gap bounded away from zero. A Markov chain $y_1,\ldots, y_T$ is stationary if $y_1\sim F$ where $F$ is the stationary distribution, and is reversible if for any $y, y'$, $F(y) P(y, y') = F(y') P(y', y)$.
For simplicity, we focus on stationary chains, though similar results hold for non-stationary chains (see \citet{Paulin2015} and references therein).

\begin{theorem}[\citet{Paulin2015}, Theorem 3.8]\label{thm:conc-markov}
Let $X_1,\ldots,X_T$ be a stationary and reversible Markov chain on a state space $\Omega$, with stationary distribution $F$ and spectral gap $\gamma$. Let $g:\Omega\rightarrow [0,1]$, then
\begin{equation*}
\Pr \left[ \left| \frac{1}{T}\sum_{t=1}^T g(X_t)- \E_{X\sim F}[g(X)] \right|>\epsilon \right] \leq 2 \exp\left(-\frac{T \gamma \epsilon^2 }{4+10 \epsilon}\right).
\end{equation*}
\end{theorem}

Applying this result, we obtain the following lemma
(see Appendix~\ref{app:Markov} for the proof):

\begin{lemma}\label{lem:Markov}
Suppose that the adversary actions $y_1,\ldots,y_T$ form a stationary and reversible Markov
chain with stationary distribution $F$ and spectral gap $\gamma$. Then for any no-regret learning
algorithm, with probability at least $1-\delta$,
\[
\frac{1}{T}\sum_{t=1}^T \E_{x_t}[f(x_t,y_t)] \geq \sup_{x\in \X} \E_{y\sim F}[f(x,y)]- \sqrt{\frac{14 \log(2/\delta)}{ \gamma T}} - \frac{\regret}{T}.
\]
\end{lemma}

\begin{example}[Sticky Markov Chain] Consider a Markov chain where at each round $y_t$ is equal to $y_{t-1}$ with some probability $\rho\geq 1/2$ and with the remaining probability $(1-\rho)$ it is drawn independently from some fixed distribution $F$. It is clear that the stationary distribution of this  chain is equal to $F$. We can bound the spectral gap of this Markov chain by the Cheeger bound~\cite{Cheeger1969}. The Cheeger constant for a finite state, reversible Markov chain is defined, and in this case bounded, as
\begin{align*}
\Phi =~& \min_{\states \subseteq \Omega: F(\states )\leq 1/2}\frac{\sum_{y\in \states } \sum_{y'\in \states ^c} F(y) P(y,y')}{F(\states )} = \min_{\states \subseteq \Omega: F(\states )\leq 1/2}\frac{\sum_{y\in \states } \sum_{y'\in \states ^c} F(y) (1-\rho) F(y')}{F(\states )}\\
=~& \min_{\states \subseteq \Omega: F(\states )\leq 1/2}(1-\rho) \frac{F(\states )\cdot F(\states ^c)}{F(\states )} =\min_{\states \subseteq \Omega: F(\states )\leq 1/2} (1-\rho) F(\states ^c) \geq \frac{1-\rho}{2}
\enspace.
\end{align*}
Moreover, by the Cheeger bound we know that $\gamma \geq \frac{\Phi^2}{2} \geq \frac{(1-\rho)^2}{8}$. Thus we get that for such a sequence of adversary actions, with probability $1-\delta$,
\[
\frac{1}{T}\sum_{t=1}^T \E_{x_t}[f(x_t,y_t)] \geq \sup_{x\in \X} \E_{y\sim F}[f(x,y)]- \frac{4}{1-\rho}\sqrt{\frac{7 \log(2/\delta)}{T}} - \frac{\regret}{T}
\enspace.
\]
\end{example}

\subsection{Implications for Online Optimal Auction Design}
\label{subsec:overallOPT}

Consider the online optimal auction design problem for a single item and $n$ bidders. Suppose that the adversary who picks the valuation profiles $\vec v_1,\ldots,\vec v_T$ is Markovian and that the stationary distribution $F$ of the chain is independent across players, i.e., the stationary distribution is a product distribution $F=F_1\times\dotsb\times F_n$.
Then we know that the optimal auction for this setting is \emph{Myerson's (optimal) auction} \cite{Myerson1981}, which translates the players' values based on some monotone function $\phi$, known as the \emph{ironed virtual value function}, and then allocates the item to the bidder with the highest virtual value, charging payments so that the mechanism is dominant-strategy truthful.
In this section, we use the Generalized FTPL algorithm to compete with Myerson's optimal auction in this Markovian environment.
This extends the prior work that only gave such guarantees for the i.i.d.\ setting.\footnote{If we know ahead of the time that the stationary distribution is symmetric across bidders, i.e., $F_1=F_2=\cdots=F_n$, then the optimal auction is a second-price auction with a reserve. In that case, we can appeal to \Thm{I} to obtain a better regret bound than those presented in this section.}

A natural approach for approximating the overall optimal auction in a Markovian environment is through level auctions. As discussed in Section~\ref{sec:level}, \citet{morgenstern2015pseudo} show that level auctions with repeated thresholds and arbitrarily fine discretization approximate Myerson's optimal auction in terms of revenue. In more detail, when distributions $F_i$ are bounded in $[1,H]$, the class of  auctions $\mathcal{R}_{s, \infty}$ with $s=\Omega\Parens{\frac1\epsilon+\log_{1+\epsilon}H}$ achieves expected revenue of at least a factor of $(1 - \epsilon)$ of the expected optimal revenue of Myerson's auction. Since the runtime and regret of our Generalized FTPL algorithm for the class of auctions $\mathcal{R}_{s, m}$ scale with the discretization level $m$, we require $m$ to be finite.
Therefore, we cannot use this characterization of \citet{morgenstern2015pseudo} directly.
Analogously to these results, we prove and use an additive approximation guarantee for the class of discretized level auctions $\mathcal{R}_{n/\epsilon, 1/\epsilon}$, showing that
\begin{equation} \label{eq:markov}
 \max_{\vec \theta \in \mathcal{R}_{n/\epsilon, 1/\epsilon}} \E_{\vec v \sim F}\left[ \rev(\vec \theta, \vec v)  \right] \geq \opt(F) - \epsilon,
\end{equation}
where $\opt(F)$ is the optimal revenue achievable by any dominant-strategy truthful mechanism for valuation vector distribution $F$.

 At a high level, we prove Equation~\eqref{eq:markov} using a three-step approach.
We first consider the discretization of the bidders' valuations.
That is, for any bid profile $\vec v$, we consider $\floors{\vec v}_\epsilon$ that denotes the bid profile where each entry of $\vec v$ is  rounded down to the nearest whole multiple of $\epsilon$.
As the first step, it is not hard to show  that for any $\vec \theta \in \mathcal{R}_{n/\epsilon, 1/\epsilon}$, $\rev(\vec \theta, \floors{\vec v}_\epsilon) =  \rev(\vec \theta, \vec v)$.
In the second step, we show that Myerson's optimal auction on the discretized valuations indeed belongs to the class $\mathcal{R}_{n/\epsilon, 1/\epsilon}$. That is, when $F'$ represents the distribution over valuations $\floors{\vec v}_\epsilon$, we have
\[ \max_{\vec \theta \in \mathcal{R}_{n/\epsilon, 1/\epsilon}} \E_{\vec v \sim F}\left[\rev(\vec \theta, \floors{\vec v}_\epsilon )  \right] = \opt(F').
\]
To prove this we use a characterization of Myerson's optimal auction on discrete (and not necessarily regular) distributions provided by \citet{elkind2007designing}. Much like level auctions, this characterization assigns a bucket number to each possible valuation of each bidder and allocates the item to the valuation with the highest bucket number. We show that one can create $n/\epsilon$ thresholds for each bidder to mimic this bucketing effect. Therefore, the class of auctions $\mathcal{R}_{n/\epsilon, 1/\epsilon}$ includes an optimal auction for the discretized valuations.
In the last step, we use a result of  \citet{DHP16} that establishes that $\opt(F') \geq \opt(F) - \epsilon$.
Putting the above ingredients together, Theorem~\ref{thm:markov-opt} shows that the Genearlized FTPL algorithm for $\mathcal{R}_{n/\epsilon, 1/\epsilon}$ approximates Myerson's auction for a given Markov chain.
We defer the detailed proof of Theorem~\ref{thm:markov-opt} to Section~\ref{app:opt:level} after giving an example of this setting.

\begin{theorem}[Competing with Universally Optimal Auction] \label{thm:markov-opt}
Consider the online auction design problem for a single item among $n$ bidders, where the sequence of valuation vectors $\vec v_1,\dotsc, \vec v_T$ is Markovian, following a stationary and reversible Markov process with a spectral gap of $\gamma>0$ and with a stationary distribution $F$ which is a product distribution across bidders, meaning $F=F_1\times\cdots\times F_n$. Then, the Generalized FTPL algorithm for  $\mathcal{R}_{n/\epsilon,1/\epsilon}$, where  $\epsilon = \Theta(n^{3/5}T^{-1/10})$,
guarantees the following bound with probability at least $1-\delta$:
\begin{align*}
\frac{1}{T}\sum_{t=1}^T \E \left[ \rev ({\vec \theta}_t, \vec v_t) \right] \geq \opt(F) - \sqrt{\frac{14 \log(2/\delta)}{\gamma T}} - O\left( n^{3/5} T^{ - 1/10}    \right).
\end{align*}

\end{theorem}

\begin{example}[Valuation Shocks] Consider the setting where valuations of players in the beginning of time are drawn from some product distribution $F=F_1\times\cdots\times F_n$. Then in each round with some probability $\rho$ the valuations of all players remain the same as in the previous round, while with some probability $1-\rho$, there is a shock in the market and the valuations of the players are re-drawn from distribution $F$. As we analyzed in the previous section, the spectral gap of the Markov chain defined by this problem is at least $\frac{(1-\rho)^2}{8}$. Thus we get a regret bound which depends inversely on the quantity $1-\rho$.

Hence, our online learning algorithm achieves revenue that is close to the optimal revenue achievable by any dominant-strategy truthful mechanism for the distribution $F$. More importantly, it achieves this guarantee even if the valuations of the players are not drawn i.i.d.\ at every iteration and even if the learner does not know what the distribution $F$ is, or when the valuations of the players are going to be re-drawn, or what the rate $\rho$ of shocks in the markets is.
\end{example}

\subsection{Proof of Theorem~\ref{thm:markov-opt}}
\label{app:opt:level}

\newcommand{\pos}{\mathrm{pos}}

We set out to prove that
\[ \frac 1T \sum_{t=1}^T \E\left[ \rev(\vec \theta_t, \vec v_t) \right] \geq \opt(F) - \epsilon - \sqrt{\frac{14 \log(2/\delta) }{\gamma T}} - O \left(n^{3/5} T^{-1/10}  \right).
\]
We first show that $\rev({\vec \theta}, \vec v) = \rev(\vec \theta, \floors{\vec v}_\epsilon)$ for any $\vec v$ and any $\vec \theta \in \mathcal{R}_{n/\epsilon, 1/\epsilon}$.
To start, note that the thresholds in $\vec \theta$ are multiples of $\epsilon$, so bid profiles $\vec v$ and $\floors{\vec v}_\epsilon$ are bucketed identically by $\vec \theta$.
Therefore, these two bid profiles have the same winner and the same payment amount (equal to the smallest bid with which the bidder still wins the item). Thus,
we have that $\rev({\vec \theta}, \vec v) =  \rev(\vec \theta, \floors{\vec v}_\epsilon)$,
and by Lemma~\ref{lem:Markov}, we can bound the revenue of the Generalized FTPL algorithm as
\begin{align*}
\frac{1}{T}
\sum_{t=1}^T \E_{{\vec\theta}_t}\BigBracks{\rev({\vec \theta}_t, \vec v_t)}
 & \geq \sup_{\vec\theta\in \mathcal{R}_{n/\epsilon, 1/\epsilon}} \E_{\vec v\sim F}\BigBracks{\rev(\vec \theta, \vec v)} - \sqrt{\frac{14 \log(2/\delta) }{\gamma T}} - O\Parens{n^{3/5} T^{-1/10}} \\
 & \geq \sup_{\vec\theta\in \mathcal{R}_{n/\epsilon, 1/\epsilon}} \E_{\vec v\sim F}\BigBracks{\rev\bigParens{\vec \theta, \floors{\vec v}_\epsilon}} - \sqrt{\frac{14 \log(2/\delta) }{\gamma T}} - O\Parens{n^{3/5} T^{-1/10}}
.
\numberthis
\label{eq:opti-1st}
\end{align*}

Next, we show that the expected revenue of the optimal auction $\vec \theta \in \mathcal{R}_{n/\epsilon, 1/\epsilon}$ on the rounded valuation $\vec v'=\floors{\vec v}_\epsilon$ is close to $\opt(F')$, where $F'$
is the product distribution of $\vec v'$.

\begin{lemma}\label{lem:R=OPT(F')}
Let $F' = F'_1 \times \dots \times F'_n$ be a product distribution over $\vec v'$ that are  supported on multiples of $\epsilon$. Let $\opt(F')$ be  the revenue of Myerson's optimal auction on $F'$. Then,
\[   \max_{\vec \theta \in \mathcal{R}_{n/\epsilon, 1/\epsilon}} \E_{\vec v'\sim F'} \left[  \rev(\vec \theta, \vec v') \right] = \opt(F').
\]
\end{lemma}

To prove Lemma~\ref{lem:R=OPT(F')}, we first use the following result of \citet{elkind2007designing} which describes Myerson's optimal auctions on arbitrary discrete distributions with a finite support.

\begin{lemma}[Theorem 3.2 of \citet{elkind2007designing}] \label{lem:elkind}
Consider a product distribution
$F' = F'_1 \times \dots \times F'_n$, where each $F'_i$ is a discrete distribution
supported on $m+1$ values indexed in an increasing order by $\ell\in\set{0,1,\dotsc,m}$.
There exist values $\bar{c}_i^\ell$ associated with all possible bid levels of each bidder, satisfying
$\bar{c}_i^{\ell} \leq \bar{c}_i^{\ell+1}$ for all $i\in[n]$ and $\ell\in\set{0,1,\dots,m-1}$, giving rise to the optimal auction for $F'$ according to the following mechanism: Sort bidders in the order of decreasing values $\bar{c}_i^{\ell}$ that correspond to their submitted bids, breaking ties in favor of bidders with smaller indices. Allocate the item to the first bidder whose corresponding $\bar{c}_i^{\ell}$ is non-negative and charge the bidder the minimum bid that he could have placed and still won the item.
\end{lemma}

\begin{proof}[Proof of Lemma~\ref{lem:R=OPT(F')}]
Since each bid is of the form $\ell\epsilon$ where $\ell\in\set{0,\dotsc,m}$, by Lemma~\ref{lem:elkind},
there exist suitable values $\bar{c}_i^{\ell}$ that implement the optimal auction. We will show that
there is a level auction $\vec \theta \in \mathcal{R}_{n/\epsilon, 1/\epsilon}$ that implements the same auction.

First sort the values $\bar{c}_i^{\ell}$ from smallest to largest, breaking ties in favor of smaller $\ell$ and larger $i $ (same as done in Elkind's  mechanism, but in the reverse order). Let $\pos(i, \ell)$ be as
follows: Assign $\pos(i, \ell) = -1$ for all leading negative $\bar{c}_i^{\ell}$ and then let $\pos(i, \ell)$ be the position of $\bar{c}_i^{\ell}$ in this
list, starting with index $0$ after the leading $-1$s. The allocation rule of Elkind's mechanism is equivalent to replacing each bid by its
corresponding $\pos(i, \ell)$ and giving the item to the bidder with the highest non-negative position. We now inductively construct a
level-auction mechanism $\vec \theta$ in which $b_i^{\vec \theta}(\ell \epsilon) = \pos(i, \ell)$.

Take any $i \in [n]$. We begin with an empty set of thresholds and introduce new thresholds gradually for $\ell=0,\dotsc,m$.
For ease of exposition let $\pos(i, -1) = -1$ and  $b_i^{\vec \theta}(-\epsilon) = -1$. This guarantees that
$b_i^{\vec \theta}(\ell \epsilon) = \pos(i, \ell)$ holds for $\ell=-1$. Now assume that the thresholds introduced so far guarantee
that $b_i^{\vec \theta}(\ell \epsilon) = \pos(i, \ell)$ holds for $\ell<k$. For $\ell=k$, we only create additional thresholds with the value $k/m$, which will preserve this
property for $\ell<k$, but also allows us to extend it to $\ell=k$. Specifically, we add $\pos(i, k)- \pos(i, k- 1)$ many thresholds with the value $k/m$.
Using the induction hypothesis, we have
\[ b_i^{\vec \theta}(k \epsilon)  = b_i^{\vec \theta}\bigParens{(k-1)\epsilon} + \pos(i, k) - \pos(i, k- 1) = \pos(i,k).
\]
Hence, by induction $b_i^{\vec \theta}(\ell \epsilon) = \pos(i, \ell)$ for all $\ell$.
This completes the claim that we have created a level auction $\vec \theta\in \mathcal{R}_{n/\epsilon, 1/\epsilon}$ that implements the optimal auction described in Lemma~\ref{lem:elkind}.
\end{proof}

We next use the result of \citet{DHP16} that shows that $\opt(F')$ is almost as large as $\opt(F)$.

\begin{lemma}[Lemma 4.3 of \citet{DHP16arXiv}]
Given any product distribution $F = F_1 \times \dots \times F_n$, let $F'$ be the distribution obtained by rounding down the values from $F$  to the nearest multiple of $\epsilon$. Then $\opt(F') \geq \opt(F) - \epsilon$.
\end{lemma}

Combining the above proves Theorem~\ref{thm:markov-opt}.

\section{Contextual Online Learning: Learning with Side Information}
\label{sec:contexts}

We now consider a generalization of the online learning setting of Section~\ref{sec:ftpl}, where in each round the learner also observes some side information,
called \emph{context}. The context $\sigma_t$ in round $t$ comes from some abstract context space $\Sigma$. The learner wants to use this contextual information to improve his performance. Specifically, the learner's goal is to compete with a set of policies $\Pi$, where each policy $\pi\in\Pi$ is a mapping from contexts $\sigma\in \Sigma$ to actions $\pi(\sigma)\in \X$.

The adversary, in each round, chooses both a context $\sigma_t\in \Sigma$ and an action
$y_t\in \Y$. The payoff of the learner if he chooses a policy $\pi_t$ is then $f(\pi_t(\sigma_t),y_t)$. The regret of an algorithm is given by:
\[
\regret = \E\left[ \max_{\pi \in \Pi} \sum_{t=1}^T f(\pi(\sigma_t), y_t) -
  \sum_{t=1}^T f(\pi_t(\sigma_t),
  y_t) \right] .
\]
An offline oracle in the contextual setting is an algorithm that takes as input a distribution over pairs of contexts and adversary actions and returns the best policy in the policy space $\Pi$ for this distribution.

The contextual learning problem, described by $\X$, $\Sigma$, $\Pi$, $\Y$ and $f$, can be viewed as a specific instance of a (non-contextual) learning problem described in Section~\ref{sec:ftpl}, with the learner's action space $\X_c = \Pi$, adversary action space $\Y_c=\Sigma \times \Y$ and payoff function $f_c(\pi,(\sigma,y))=f(\pi(\sigma),y)$. Moreover, any offline oracle for the contextual problem $(\X,\Sigma,\Pi,\Y,f)$ is also an offline oracle for the non-contextual problem $(\X_c,\Y_c,f_c)$. Below,
we show how to obtain oracle-efficient no-regret algorithms for $(\X_c,\Y_c,f_c)$ if we have access to an admissible and implementable matrix $\mGamma$ for
the simpler non-contextual problem $(\X,\Y,f)$.
We will conflate the problems $(\X,\Sigma,\Pi,\Y,f)$ and $(\X_c,\Y_c,f_c)$, and jointly refer
to them as the \emph{contextual problem}, and refer to the simpler problem $(\X,\Y,f)$ as the \emph{non-contextual problem}.

\begin{example}[Contextual Online Auction Design] In contextual online auction design, the auctioneer gets to see some side information about the bidders before they place their bids. The goal of the auctioneer is to compete with a set of policies that map such contextual information to an auction. Given a class of auctions~$\A$, a set of possible valuations $\V$, a policy class $\Pi$ and an unknown sequence of contexts and bid profiles $(\sigma_1,\vec v_1), \dots, (\sigma_T,\vec v_T)\in \Sigma\times\V^n$, the goal of an online algorithm is to pick in each round a policy $\pi_t\in \Pi$ such that the algorithm's total revenue is close to the revenue of the best policy $\pi\in \Pi$ in hindsight:
\[ \regret = \E\left[ \max_{\pi \in \Pi}  \sum_{t=1}^T \rev (\pi(\sigma_t), \vec v_t) - \sum_{t=1}^T \rev (\pi_t(\sigma_t), \vec v_t) \right] \leq  o(T).
\]
\end{example}

\subsection{From Non-Contextual to Contextual Learning for Separable Policy Spaces}
\label{sec:separable}

Assume we are given an admissible and implementable matrix $\mGamma$
for $(\X,\Y,f)$.
We next show how to construct an admissible and implementable matrix $\mGamma^\sep$ for the contextual problem $(\X_c,\Y_c,f_c)$. Our construction builds on the notion of a separator of a policy space, also known as the \emph{universal identification sequence}, introduced by~\citet{goldman1993exact}; the term \emph{separator} is due to \citet{Syrgkanis2016}.

\begin{definition}[Separator] A set $\sep\subseteq \Sigma$ of contexts is a separator for a policy space $\Pi$, if for any two policies $\pi,\pi'\in \Pi$, there exists a context $\sigma\in \sep$ such that $\pi(\sigma)\neq \pi'(\sigma)$.
\end{definition}

\begin{definition}[Contextual $\sep$-extension of a matrix $\mGamma$] For any matrix $\mGamma \in [0,1]^{|\X|\times N}$, we define its contextual $\sep$-extension $\mGamma^\sep$ as an $|\Pi|\times (|\sep|\cdot N)$ matrix, where each column $j_c$ is associated with a pair $(\sigma,j)$ for $\sigma \in \sep$ and $j\in [N]$ and the entry of $\mGamma^\sep$ at coordinates $(\pi,(\sigma,j))$ is equal to the entry of matrix $\mGamma$ at coordinates $(\pi(\sigma),j)$, i.e.,
\begin{equation}
\Gamma_{\pi,(\sigma,j)}^\sep = \Gamma_{\pi(\sigma), j}.
\end{equation}
\end{definition}

The next two lemmas, whose proofs appear in Appendices~\ref{app:contexts-admissible} and \ref{app:contexts-implementable}, show that if a matrix is admissible and implementable for
$(\X,\Y,f)$, then its contextual extension with respect to a separator is admissible and implementable for $(\X_c,\Y_c,f_c)$.

\begin{lemma}[Contextual Admissibility] \label{lem:contextual_admissibility}
If $\mGamma$ is a $\delta$-admissible matrix for $(\X,\Y,f)$ and $\sep$ is a separator for the policy space $\Pi$, then the contextual $\sep$-extension $\mGamma^\sep$ is
$\delta$-admissible for the contextual problem $(\X_c,\Y_c,f_c)$.
\end{lemma}

\begin{lemma}[Contextual Implementability] \label{lem:contextual_imp}
If matrix $\mGamma$ is implementable with complexity $M$ for $(\X,\Y,f)$, then matrix $\mGamma^\sep$ is implementable with complexity $M$ for the contextual problem $(\X_c,\Y_c,f_c)$.
\end{lemma}

Given the above two lemmas, we can now invoke Corollary~\ref{cor:oracle-ftpl} to obtain oracle-efficient online learning for the contextual problem $(\X_c,\Y_c,f_c)$:
\begin{corollary}
\label{cor:oracle-ftpl:sep}
Assume that $\mGamma\in[0,1]^{\card{\X}\times N}$ is implementable with complexity $M$ and $\delta$-admissible for $(\X,\Y,f)$. Let $\sep$ be a separator of the policy space $\Pi$ of size $d$. Let $D$ be the uniform distribution on $[0, 1/\eta]$ for $\eta  =  \delta /\sqrt{2T(1+2T^{-1/2})(1+\delta)}$.
Then Algorithm~\ref{alg:oracleftpl} applied to the contextual problem $(\X_c,\Y_c,f_c)$ with the contextual $\sep$-extension $\mGamma^\sep$ of matrix $\mGamma$ achieves the following regret:
\[
\regret \leq  O(N d \sqrt{T}/\delta).
\]
Moreover, if there exists a contextual offline oracle $\opt(\cdot, \frac{1}{\sqrt{T}})$ which runs in time $\poly(N,d,M,T)$, then Algorithm \ref{alg:oracleftpl} can be implemented efficiently in time $\poly(N,d,M,T)$.
\end{corollary}

Several examples of policy spaces with small separators, i.e., separators whose size is poly-logarithmic in the number of contexts, are presented by \citet{goldman1993exact} and \citet{Syrgkanis2016}. In their examples, contexts are binary vectors and the task is binary classification; the examples include Boolean conjunctions, Boolean disjunctions, logarithmic-depth read-once majority formulas, and logarithmic-depth read-once positive NAND formulas.
We next present an online auction design example where a Boolean disjunction is used to discriminate between two vectors of reserve prices.

\begin{example}[Price Discrimination based on ORs of Binary Features]
\label{ex:or}
We consider the class of single-item second price-auctions with discretized bidder-specific reserves, $\I_m$, where the arriving bidders are described by binary feature vectors in $\{0,1\}^K$, constituting the context $\sigma$. In \Sec{VCG}, our goal was to identify a single vector of reserves $\vec r\in\I_m$ that would optimize the revenue. Here, we would like to achieve a higher revenue by splitting the market into two segments according to the contextual information and using a separate vector of reserves for each market segment. Formally, we define the set of policies $\Pi = \{\pi_{\vec r_0, \vec r_1,S} \mid \vec r_0 \in \I_m, \vec r_1\in \I_m, S\subseteq [K]\}$, where $\vec r_0$ and $\vec r_1$ are the reserve vectors for the two respective segments and $S$ determines the two segments as follows: $\Sigma_0=\Set{\sigma\bigm{\vert}\bigvee_{i\in S}\sigma_i = 0}$, $\Sigma_1=\Set{\sigma\bigm{\vert}\bigvee_{i\in S}\sigma_i = 1}$. Thus,
\[
  \pi_{\vec r_0,\vec r_1, S}(\sigma)=
\begin{cases}
  \vec r_0 & \text{if $\sigma\in\Sigma_0$}
\\
  \vec r_1 & \text{if $\sigma\in\Sigma_1$.}
\end{cases}
\]
Our contextual online learning algorithm over this policy space competes with a variety of
market splitting strategies. For instance, we can do as well as an auctioneer that picks some high-reserve vector $\vec r_1$ and a low-reserve vector $\vec r_0$, learns which of the Boolean coordinates are indicative of a high valuation for the players, and
when any one of them is 1, it uses the high reserves $\vec r_1$, and otherwise the low reserves $\vec r_0$.

To show that we can indeed efficiently compete with the class $\Pi$ we just
need to show that it has a small separator.
We argue that the set $V = \{e_i \mid  i\in [K]\} \cup \{\boldsymbol{1}, \boldsymbol{0}\}$ is a separator for $\Pi$.
Consider any two policies
$\pi_{\vec r_0, \vec r_1, S}$ and $\pi_{\vec r'_0, \vec r'_1, S'}$ that differ on some context $\sigma$. Then these policies must also differ on one of the following $\sigma'\in V$:
\begin{align*}
 \sigma' = \begin{cases}
	\boldsymbol{1} & \text{if } \bigvee_{i\in S}\sigma_i = \bigvee_{i\in S'}\sigma_i =1
\\
	\boldsymbol{0} & \text{if } \bigvee_{i\in S}\sigma_i = \bigvee_{i\in S'}\sigma_i =0
\\
	\vec e_j, \text{ for } j\in S\setminus S' & \text{if } \bigvee_{i\in S}\sigma_i = 1,\,\bigvee_{i\in S'}\sigma_i =0
\\
	\vec e_j, \text{ for } j\in S'\setminus S & \text{if } \bigvee_{i\in S}\sigma_i = 0,\,\bigvee_{i\in S'}\sigma_i =1.
 \end{cases}
\end{align*}

Therefore, the size of the separator is $d=O(K)$, while the size of the policy space is exponential in $K$.
Similar analysis can be done when the policies are using the AND of Boolean features rather than the OR.
\end{example}

\subsection{Transductive Contextual Learning}
\label{sec:context-transductive}

We now consider the case where the learner knows a priori the set of contexts that could potentially arise, i.e., we assume that he knows that all contexts $\sigma_1,\ldots,\sigma_T$ come from a set $\tsep$, which we will refer to as the \emph{transductive set}. We do not require that the learner knows the multiplicity of each context or the sequence under which contexts arrive. Moreover, the set $\tsep$ could be of size
as large as $T$, and in fact our analysis allows $\card{\tsep}=o(T^2)$.

In this setting, ignoring some technical details, we can treat the transductive set $\tsep$ as a separator defined in the previous section. However, using the analysis of the previous section, we would obtain the regret guarantee that would grow linearly or super-linearly in $T$ when the size of the set $\tsep$ is $\Omega(\sqrt{T})$. Thus in order to guarantee sub-linear regret
we need a tighter analysis. To achieve this we will leverage the fact that in the transductive setting $\tsep$ is the whole set of contexts that will arise, which is a stronger property than being a separator. This will allow us to prove a stronger stability result than obtained via Lemma \ref{lem:stability-approx} in the previous section. Moreover, we will use a perturbation distribution $D$ that has mean zero, leading to cancellations in the cumulative error term, which make that term not grow linearly with the size of $\tsep$ but rather as the square root of the size of $\tsep$. The combination of these two improvements leads to sub-linear regret
when $|\tsep|=o(T^2)$.

We begin with a non-contextual learning problem $(\X,\Y,f)$ admitting a translation matrix $\mGamma\in [0,1]^{|\X|\times N}$ that is implementable with complexity $M$ and $\delta$-admissible. We study the \emph{transductive contextual problem} $(\X_c,\Y_c,f_c,\tsep)$, i.e., the contextual problem $(\X_c,\Y_c,f_c)$ with the additional restriction that $\sigma_t\in\tsep$, and analyze the regret of Algorithm \ref{alg:oracleftpl}
with the contextual extension $\mGamma^\tsep$ of matrix $\mGamma$.

We first show an improved stability property of this algorithm. Even though the number of columns of matrix $\mGamma^\tsep$ is $|\tsep|\cdot N$, we can show that the stability of the algorithm does not depend on $|\tsep|$.
The proof of the next lemma appears in Appendix~\ref{app:context-stability}.
\begin{lemma}\label{lem:transductive-stability}
Let $\mGamma\in [0,1]^{|\X|\times N}$ be $\delta$-admissible for the non-contextual problem $(\X,\Y,f)$.
Then the Generalized FTPL algorithm for the transductive contextual problem $(\X_c,\Y_c,f_c,P)$ with the contextual extension $\mGamma^\tsep$ of $\mGamma$ as the translation matrix and with a $\left(\rho,\frac{1+2\epsilon}{\delta}\right)$-dispersed distribution $D$ achieves the stability bound
  \[ \E\left[ \sum_{t=1}^T f_c(\pi_{t+1},(\sigma_{t}, y_t))- f_c(\pi_t, (\sigma_t,y_t)) \right] \leq 2
  T N \rho(1+\delta^{-1}).\]
\end{lemma}

We next show that the error term from Equation \eqref{eq:stab+noise_approx}, in our case,
\begin{equation}
\E\left[ \valpha\inprod( \mGamma_{\pi_1}^\tsep  -  \mGamma_{\pi^*}^\tsep ) \right]
\enspace,
\end{equation}
does not grow linearly with the size of the transductive set $|\tsep|$ since we allow $|\tsep|=\Omega(T)$. We achieve this by using a mean-zero distribution $D$, rather than a non-negatively supported one. For such mean-zero distributions we can show that the error term grows as the square root of the number of columns of the contextual-extension matrix, rather than linearly. Combining the two improvements we get the following theorem (see Appendix \ref{app:contexts} for the proof), which is a more refined version of Theorem \ref{thm:FTPL-U}.

\begin{theorem}\label{thm:transductive-FTPL-G}
Let $\mGamma\in [0,1]^{|\X|\times N}$ be $\delta$-admissible for the non-contextual problem $(\X,\Y,f)$ and $D$ be the uniform distribution on $[-\nu,\nu]$. Then the regret of the Generalized FTPL algorithm for the transductive contextual problem $(\X_c,\Y_c,f_c,P)$ with translation matrix $\mGamma^\tsep$ can be
  bounded as
\[
\regret \leq
\frac{T N(1+2\epsilon)(1+\delta)}{ \nu\delta^2} + 2 \nu\sqrt{2N |\tsep| \ln |\Pi|} + \epsilon T.
\]
For $\nu  =  \frac{\sqrt{T(1+2\epsilon)(1+\delta)}}{\delta}\Parens{\frac{N}{|\tsep|\ln|\Pi|}}^{\!\frac14}$
the above bound becomes $O\BigParens{N^{\frac34}\bigParens{|\tsep|\ln|\Pi|}^{\frac14}\,\frac{\sqrt{T(1+2\epsilon)(1+\delta)}}{\delta}+\epsilon T}$.
Setting $\epsilon=T^{-1/2}$ yields the regret $O\bigParens{T^{1/2} N^{\frac34}(|P|\ln|\Pi|)^{\frac14}/\delta}$, which becomes $O\bigParens{(NT)^{\frac34}(\ln|\Pi|)^{\frac14}/\delta}$ when $|\tsep|=O(T)$.
\end{theorem}

It remains to argue the oracle efficiency of the Generalized FTPL algorithm. However, the implementability condition in Section \ref{sec:ftpl} only allows for non-negatively supported distributions $D$. If we want to allow arbitrary distributions over reals, our translation matrix must be able to implement negative weights~$\alpha_j$. This means we need to be able to simulate the negative of the difference between any two entries of a column by a dataset of adversary actions. This leads to the following additional condition on the translation matrix.
\begin{definition}
\label{def:imp-neg}
  A matrix $\mGamma$ is \emph{negatively implementable} with complexity $M$ if
  for each $j\in[N]$ there exist a weighted dataset $S_j^{-}$, with $|S_j^-|\leq M$, such that
\begin{align*}
\hspace{1in}
    &\text{for all $x,x'\in\X$:}
    &-\left(\Gamma_{xj}-\Gamma_{x'j}\right) &= \sum_{(w,y)\in S_j^-} w\bigParens{f(x,y)-f(x',y)}.
&\hspace{1in}
\end{align*}
In this case, we say that weighted datasets $S_j^{-}$, $j\in[N]$, negatively implement $\mGamma$ with complexity $M$.
\end{definition}

Using a matrix $\mGamma$ that is both implementable and negatively implementable gives rise to an oracle-efficient variant of the Generalized FTPL provided in Algorithm~\ref{alg:oracleftpl-neg}, which allows the distribution $D$ to be
over both positive and negative reals. Similar to Theorem~\ref{thm:imp}, it is easy to see that the output of this algorithm is equivalent to the output of Generalized FTPL and therefore has the same regret guarantee (we omit this proof as it is identical to the proof of Theorem~\ref{thm:imp}).

\begin{algorithm}[t]
  \begin{algorithmic}
  \STATE Input:
 		    datasets $S_j$ and  $S_j^-$, $j\in[N]$, that implement and negatively implement $\mGamma\in[0,1]^{\card{\X}\times N}$,
  \STATE \hspace{\algorithmicindent}%
                distribution $D$ over $\R$,
  \STATE \hspace{\algorithmicindent}%
                an offline oracle \opt.
    \STATE{Draw $\alpha_j \sim D$ independently for $j = 1, \ldots, N$.}
    \FOR{$t=1, \ldots, T$}
        \FOR{$j=1, \ldots, N$}
            \STATE{If $\alpha_j<0$, let $Q_j\coloneqq\{(-\alpha_j w, y): (w,y) \in S_j^-\}$ be the scaled version of $S_j^-$.}
            \STATE{If $\alpha_j\geq 0$, let $Q_j\coloneqq\{(\hphantom{-}\alpha_j w, y): (w,y) \in S_j^{\hphantom{-}}\}$ be the scaled version of $S_j^{\hphantom{-}}$.}
        \ENDFOR
        \STATE{Set $S=\bigSet{(1,y_1),\dotsc,(1,y_{t-1})}\cup\bigcup_{j\le N} Q_j$.}
        \STATE{Play
          $x_t =  \opt(S,\frac{1}{\sqrt{T}})$.
        }
        \STATE{Observe $y_t$ and receive payoff $f(x_t, y_t)$.}
    \ENDFOR
\end{algorithmic}
\caption{Oracle-Based Generalized FTPL with distributions over positive and negative reals}
  \label{alg:oracleftpl-neg}
\end{algorithm}

\begin{theorem}
\label{thm:imp-neg}
\label{thm:adm:imp-neg}
If $\mGamma$ is implementable and negatively implementable with complexity $M$,
  then Algorithm~\ref{alg:oracleftpl-neg} is an oracle-efficient implementation
  of Algorithm~\ref{alg:ftpl-approximate} with $\epsilon = 1/\sqrt{T}$ and has per-round complexity $\order\bigParens{T+NM}$.
\end{theorem}

As an immediate corollary, we obtain that the existence of a polynomial-time offline contextual oracle implies the existence of a polynomial-time transductive contextual online learner with regret $\order(T^{3/4})$, whenever we have
access to an implementable, negatively implementable, and admissible matrix.

\begin{corollary} \label{cor:oracle-ftpl-neg}
Assume that $\mGamma\in[0,1]^{\card{\X}\times N}$ is implementable and negatively implementable with complexity $M$ and $\delta$-admissible for the non-contextual problem $(\X,\Y,f)$. Also,
assume there exists an offline contextual oracle $\opt(\cdot, \frac{1}{\sqrt{T}})$ which runs in time $\poly(N,M,T)$. Let $D$ be the uniform distribution as defined in \Thm{transductive-FTPL-G}, with $\epsilon=T^{-1/2}$. Then Algorithm~\ref{alg:oracleftpl-neg}, with $D$, applied to the transductive contextual problem $(\X_c,\Y_c,f_c,P)$
with $\card{P}\le T$ using the translation matrix $\mGamma^\tsep$ runs in time $\poly(N,M,T)$ and achieves regret $O\BigParens{(NT)^{\frac34}(\ln|\Pi|)^{\frac14}/\delta}$.
\end{corollary}

Observe that the proofs of implementability for VCG auctions with bidder-specific reserves (Lemma~\ref{lem:impl_individual}) and envy-free item-pricing auctions (Lemma~\ref{lem:impl_envy}) can be easily modified to show that
the translation matrices for these applications are not only implementable, but also negatively implementable.
Thus, using the results in this section, we immediately obtain no-regret oracle-efficient algorithms for these auctions in transductive contextual setting.

\subsection{Extensions and Connections}
\label{sec:context:ext}

\paragraph{Negative Implementability for Contextual Problems with Small Separators.}

We introduced negative implementability and Algorithm~\ref{alg:oracleftpl-neg} for transductive learning, but the algorithm and the analysis also apply to general contextual problems with small separators whenever the translation matrix is both implementable and negative implementable (see Appendix~\ref{app:contexts} for the proof):
\begin{corollary}\label{cor:sep:neg}
Assume that $\mGamma\in[0,1]^{\card{\X}\times N}$ is implementable and negatively implementable with complexity $M$ and $\delta$-admissible for the non-contextual problem $(\X,\Y,f)$.
Let $\sep$ be a separator of the policy space $\Pi$ of size $d$.
Let $D$ be the uniform distribution on $[-\nu,\nu]$ for $\nu=\frac{\sqrt{T(1+2T^{-1/2})(1+\delta)}}{\delta}\Parens{\frac{Nd}{\ln|\Pi|}}^{\!\frac14}$. Then Algorithm~\ref{alg:oracleftpl-neg}, with $D$,
applied to the contextual problem $(\X_c,\Y_c,f_c)$ with the contextual $\sep$-extension $\mGamma^\sep$ of matrix $\mGamma$ achieves the following regret:
\[
\regret \leq  O\BigParens{(Nd)^{3/4}(\ln|\Pi|)^{1/4}\sqrt{T}/\delta}.
\]
Moreover, if there exists a contextual offline oracle $\opt(\cdot, \frac{1}{\sqrt{T}})$ which runs in time $\poly(N,d,M,T)$, then Algorithm~\ref{alg:oracleftpl-neg} can be implemented efficiently in time $\poly(N,d,M,T)$.
\end{corollary}

This result improves on \Cor{oracle-ftpl:sep} whenever $\ln|\Pi|=o(Nd)$. This was the case in Example~\ref{ex:or}, where $N=O(n\ln m)$, $d=O(K)$, so $Nd=O(Kn\ln m)$, whereas $\ln|\Pi|=O(K+n\ln m)$. In general we can only argue that $\ln|\Pi|\le d\ln|\X|\le Nd/\delta$, so in some cases \Cor{oracle-ftpl:sep} might be preferable.

\paragraph{Separators for Non-contextual Problems.}

In Sections~\ref{sec:separable} and~\ref{sec:context-transductive}, we showed how to turn a non-context\-u\-al translation matrix into a contextual translation matrix when we have access to a separator.
It turns out that the concept of a separator
is also instrumental in constructing non-contextual translation matrices when the payoff function takes values in $\set{0,1}$. Consider a non-contextual problem described by $(\tilde\X,\tilde\Y,\tilde f)$, where $\tilde f$ only takes values in $\set{0,1}$. To derive a translation matrix that is admissible and implementable, we consider the class of policies $\pi_{\tilde x}$ that implement learner's actions $\tilde x$, by defining $\pi_{\tilde x}(\tilde y)=\tilde f(\tilde x,\tilde y)$. Assume that the class of policies $\Pi=\set{\pi_{\tilde x}}_{\tilde x\in\tilde\X}$ has a separator $Q\subseteq\tilde\Y$. We can use it to construct a translation matrix $\tilde\mGamma$, with rows indexed by $\tilde x\in\tilde\X$ and columns indexed by $\tilde y\in Q$ such that $\tilde\Gamma_{\tilde x,\tilde y}=\tilde f(\tilde x,\tilde y)$.
From the definition of the separator,
this matrix is $1$-admissible for $(\tilde\X,\tilde\Y,\tilde f)$, because each row uniquely identifies $\tilde x$ and the matrix is 0/1-valued. It is also implementable with complexity $1$ using datasets $\set{(1, \tilde y)}$.
Thus, the size of the separator of $\Pi$ determines the number of columns of the translation matrix for $(\tilde\X,\tilde\Y,\tilde f)$.

Our notions of admissibility and implementability in a sense generalize the notion of a separator
by allowing arbitrary real-valued payoffs (rather than just binary) and allowing implementation
by weighted datasets (rather than just single examples). In our examples, we have leveraged both of these extensions to design oracle-efficient no-regret algorithms for various problem classes with real-valued losses.

\paragraph{Comparison with the algorithm of \citet{Syrgkanis2016}.}

Our contextual learning results extend the algorithm and guarantees of \citet{Syrgkanis2016}, who consider contextual problems
where the learner's actions are $K$-dimensional binary vectors
$x\in\X\subseteq\set{0,1}^K$, the adversary in each step picks a payoff function $y\in\Y\subseteq\set{\X\to[0,1]}$, and
the actual payoff corresponds to the evaluation of $y$ at $x$, i.e., $f(x,y)=y(x)$. \citet{Syrgkanis2016} assume that $\Y$ contains all linear functions. Their algorithm can be viewed as a variant
of Algorithm~\ref{alg:oracleftpl-neg} with
the non-contextual translation matrix $\mGamma$ of size $\card{\X}\times K$ whose columns correspond to the coordinates of $x$. This matrix is $1$-admissible, and also implementable and negatively implementable with complexity $1$. In our notation, $N=K$ and $\delta=1$.

\citet{Syrgkanis2016} consider an additional parameter $m=\max_{x\in\X} \norm{x}_1$ and prove contextual
regret bounds of $O\bigParens{m^{1/4}d^{3/4}\sqrt{KT\ln|\Pi|}}$ and
                 $O\bigParens{m^{1/4}d^{1/4}\sqrt{KT\ln|\Pi|}}$,
respectively, for the small-separator and transductive setting, where $d$ denotes
either the size of the separator $Q$ or the size of the transductive set $P$. Our
corresponding bounds, implied by Corollary~\ref{cor:sep:neg} and Theorem~\ref{thm:transductive-FTPL-G}, are
                 $O\bigParens{d^{3/4}K^{3/4}\sqrt{T}(\ln|\Pi|)^{1/4}}$ and
                 $O\bigParens{d^{1/4}K^{3/4}\sqrt{T}(\ln|\Pi|)^{1/4}}$.
Thus, our bounds match or improve the prior bounds whenever $\ln|\Pi|=\Omega(K/m)$.\footnote{%
  Given the knowledge of $m$, it is possible to use the probabilistic method to construct an admissible translation matrix $\mGamma'$ with $N'=O\bigParens{m\log(1+\frac{K}{m})}$ and $\delta'=\frac1m$, implementable with complexity $1$. The algorithm with this matrix matches or improves the bounds of \citet{Syrgkanis2016} if $\ln|\Pi|=\Omega\bigParens{\frac{m^6}{K^2}\ln^3(1+\frac{K}{m})}$. The algorithm with $\mGamma'$ also improves over our regret bounds for the matrix $\mGamma$ if $m$ is small; specifically, if $m=o\bigParens{(K/\ln K)^{3/7}}$.
}
More importantly, our work extends the prior
results to generic actions (rather than just binary vectors) and allows implementation by weighted datasets (rather than just single examples).

\section{Approximate Oracles and Approximate Regret} \label{sec:approximate}

The Oracle-Based Generalized FTPL algorithm requires an oracle that chooses an action whose payoff is within a small additive error of the best action's payoff.
In this section, we extend our analysis of Oracle-Based Generalized FTPL to work with oracles that return an action whose payoff is only a constant-factor approximation of the best payoff. Formally,
we consider the following oracles:
\begin{definition}[Offline $C$-Approximation Oracle]
Given $C\le 1$, an \emph{offline $C$-approximation oracle} for a payoff function $f$ and a learner action set $\X$, denoted $C\dash\opt_{(f, \X)}$, is any algorithm
 that receives
  as input a weighted set of adversary actions
 $S = \{ (w_\ell,y_\ell) \}_{\ell\in\mathcal{L}}$, $w_\ell\in\R^+$, $y_\ell\in \Y$,
and an accuracy parameter $\epsilon$,
and returns $x\in \X$, such that
 \[
\sum_{(w,y) \in S} w f(x, y) \geq C \max_{x\in \X} \sum_{(w,y) \in S} w f(x, y)
-\epsilon.
 \]
\end{definition}
Similarly, we will write $\opt_{f,\X}$ for an offline oracle $\opt$ defined in Section~\ref{sec:imp},
making the dependence on $f$ and $\X$ explicit. It corresponds to a $C$-approximation oracle with $C=1$.

As discussed earlier, a minimal assumption for computationally efficient no-regret online learning is the
existence of an efficient offline oracle $\opt_{f,\X}$ with a small additive error.
When there exists a \emph{fully polynomial-time approximation scheme} (FPTAS) for the offline optimization problem, i.e., an algorithm achieving $C=1-\xi$ for any $\xi>0$, while running in time
polynomial in the input size and $1/\xi$, then choosing $\xi=\epsilon/T$ recovers the additive sub-optimality for offline optimization
and hence yields a no-regret guarantee for Oracle-Based Generalized FTPL.
However, when the best polynomial-time approximation for a problem is a $C$-approximation for some constant $C<1$, there is no hope for simultaneously achieving no-regret and computational efficiency.
Instead, we can obtain performance guarantees
for an alternative notion of regret, introduced by \citet{kakade2009playing}, called $C\dash\regret$. Formally, the $C\dash\regret$ of an online maximization problem is defined as
\[
 C\dash\regret = \E\left[  C \max_{x \in \X}  \sum_{t=1}^T f(x, y_t)
  -  \sum_{t=1}^T f(x_t, y_t) \right].
\]

Simply running Oracle-Based Generalized FTPL with an arbitrary $C$-approximation oracle for $C<1$ does not automatically yield $C\dash\regret=o(T)$, but there are several important cases
where this happens.
Below we show that constant-factor approximation oracles obtained through relaxation of the objective and through Maximal-in-Range (MIR) algorithms yield
sublinear $C\dash\regret$.

\subsection{Approximation through Relaxation}
A large class of approximation algorithms achieve their approximation guarantees by
optimizing a relaxation of the payoff function. More formally, if there is a function $F:\X\times \Y \rightarrow \R$, such that
$C f(x, y) \leq F(x,y) \leq  f(x,y)$  and there is an offline oracle $\opt_{F, \X}$,
then it is clear that any online algorithm with sublinear regret for $F$ also achieves a sublinear $C\dash\regret$ for $f$. More formally:
\begin{theorem}\label{thm:FTPL-U-Relax-O}
Let $F$ be a function such that for any $x\in \X$ and $y\in \Y$, $f(x, y) \geq F(x,y) \geq C f(x,y)$.
Let
$\mGamma^F\in [0,1]^{|\X| \times N}$ be $\delta$-admissible and implementable with complexity $M$ for the payoff function $F$.
Then Algorithm~\ref{alg:oracleftpl} with distribution $D$ as defined in Theorem~\ref{thm:FTPL-U},
datasets that implement $\mGamma^F$,
and $\opt_{F, \X}$ as an oracle is oracle-efficient with per-round complexity $\order(T+NM)$ and
$C\dash\regret=\order(N\sqrt{T}/\delta)$ for $f$.
\end{theorem}
A similar observation was made by Balcan and Blum~\cite{balcan2006approximation} regarding approximation algorithms that use linear optimization as a relaxation and therefore can be efficiently optimized by the standard FTPL algorithm of Kalai and Vempala~\cite{KV05}. Our work extends this observation to any $C$-relaxation of the payoff function $f$ that has an FPTAS and an admissible and implementable translation matrix.

\paragraph{\citet{roughgarden2016minimizing} as a Relaxation.} The approach of \citet{roughgarden2016minimizing} for achieving a $1/2$-regret for single-item second-price auctions with bidder-specific reserves falls exactly in the relaxation approximation framework. They give a relaxed objective which admits a polynomial-time offline oracle and which is always within a factor of two of the original objective. Then they run an oracle-based online learning algorithm for the relaxed objective. In their case, the relaxed objective yields a linear optimization problem and can be solved with the standard FTPL algorithm of \citet{KV05}. The theorem above shows that the same approach works even when the relaxed objective does not reduce to a linear optimization problem but to a problem that can be tackled by our Oracle-Based Generalized FTPL, providing a potential avenue for obtaining sublinear $C\dash\regret$ for values of $C\geq 1/2$.

\subsection{Approximation by Maximal-in-Range Algorithms}

Another  interesting  class of approximation algorithms is Maximal-in-Range (MIR) algorithms.
An MIR algorithm commits to a set of actions $\X'\subseteq \X$, independently of the input, and outputs the best $x\in \X'$. The set of actions $\X'$ has the property that an optimum over $\X'$ is a $C$-approximation of an optimum over $\X$.  Thus, an MIR $C$-approximation algorithm forms an approximation oracle $C\dash\opt_{f,\X}\coloneqq\opt_{f, \X'}$.
Thanks to the $C$-approximate optimality of $\X'$, any online algorithm with sublinear regret relative to $\X'$ also achieves a sublinear $C\dash\regret$ relative to $\X$:
\begin{theorem}\label{thm:FTPL-U-MIDR-O}
Let $\mGamma^{\X'}\in [0,1]^{|\X'| \times N}$ be $\delta$-admissible and implementable with complexity $M$ for the action set $\X'$.
Then Algorithm~\ref{alg:oracleftpl} with distribution $D$ as defined in Theorem~\ref{thm:FTPL-U},
datasets that implement $\mGamma^{\X'}$,
and an MIR approximation oracle $\opt_{f, \X'}$ is oracle-efficient with per-round complexity $\order(T+NM)$ and
$C\dash\regret=\order(N\sqrt{T}/\delta)$ relative to the action set $\X$.
\end{theorem}

\section{Additional Applications}
\label{sec:additional}

In this section, we work out two additional applications of our oracle-efficient online learning approach: online welfare maximization in multi-unit auctions and no-regret learning in simultaneous second-price auctions. For readability, we use a parenthesized superscript instead of a plain subscript to denote objects appearing in the round $t$, e.g., $v_i^{(t)}$.

    \subsection{Polynomial-Time Algorithm for Online Welfare Maximization in Multi-Unit Auctions} \label{sec:knapsack}

In this section, we consider the class of single-item multi-unit auctions that has a $1/2$-approximation  Maximal-in-Range (MIR)  algorithm. We show how the results of Section~\ref{sec:approximate} can be applied to this problem to achieve a \emph{truly (oracle-free) polynomial-time} online algorithm with vanishing $1/2\dash\regret$.

We consider an online learning variant of an $n$-bidder $s$-unit environment, where the goal is to allocate $s$ identical items.
Each bidder $i$ has a monotone non-decreasing valuation function $v_i:\mathbb{N}\to[0,1]$, with $v_i(0)=0$, describing the utility for receiving any given number of items.
Equivalently, the bidder $i$ has a non-negative \emph{marginal valuation $\mu_i(\ell)=v_i(\ell)-v_i(\ell-1)$} for receiving its $\ell^{\text{th}}$ item, and his total utility for receiving $q_i$ items is $v_i(q_i) = \sum_{\ell = 1}^{q_i} \mu_i(\ell)$.
The goal is to find an allocation $\vec q \in \mathbb{N}^n$, subject to $\sum_{i=1}^n q_i = s$, that maximizes the total welfare $\sum_{i=1}^n v_i(q_i)$. Since we would like to allow settings where $s$ is much larger than $n$, we seek algorithms running in time polynomial in $n$ and independent of $s$.\footnote{%
Recall that we assume that numerical computations are $O(1)$.}
In the online learning setting, every round $t$ a fresh set of bidders arrive with new valuations $v_i^{(t)}$ and the learner commits to an allocation of the units of the item to the players, prior to seeing the valuations.
The goal of the learner is to pick an allocation each day that competes with the best allocation in hindsight.

It is not hard to see that the offline welfare maximization problem in this setting corresponds to the Knapsack problem, where each player has a valuation equal to the average value in hindsight, i.e., $\frac{1}{T} \sum_{t=1}^T v_i^{(t)}(\cdot)$.
Thus, dynamic programming can be used to compute a welfare-maximizing allocation in time polynomial in $n$ and $s$. Dobzinski and Nisan~\cite{dobzinski2007mechanisms} introduced a $1/2$-approximation MIR algorithm for this problem. If $s\le n^2$, the algorithm uses the dynamic programming to directly maximizes the welfare (without any approximation). Otherwise, the algorithm divides $s$ items into $\floors{s/Q}$ bundles
of size $Q\coloneqq\floors{s/n^2}$, and one bundle of size $Q'<Q$ containing all the remaining items (if $s$ in not a multiple of $Q$).
Then, the MIR algorithm chooses the best allocation from the set of allocations (range of the algorithm) where all the items in one bundle are allocated to the same bidders. This is
effectively a knapsack problem over $\floors{s/Q}< 2n^2$ identical items and possibly one additional and distinct item, which can be solved in time polynomial in $n$.

We show how to construct a matrix $\mGamma^{\textsc{MU}'}$ that is admissible and implementable for the allocations in the range of this MIR algorithm and then
use Theorem~\ref{thm:FTPL-U-MIDR-O} to obtain an online algorithm with vanishing $1/2\dash\regret$, running in time $\poly(n, T)$.

\paragraph{Construction of $\,\mGamma$:}
 Let $\G$ by the feasible set of the number of items assigned to any trader under
allocations in the range of the MIR algorithm of Dobzinski and Nisan~\cite{dobzinski2007mechanisms} described above, that is,
$\G\coloneqq\bigBraces{aQ+bQ':\:a\in\set{0,1,\dotsc,\floors{s/Q}},\,b\in\set{0,1}}$. The case $s\le n^2$ is included by setting $Q=1$ and $Q'=0$.
Denote the cardinality of $\G$ as $r$ and let $g_0,\dotsc,g_{r-1}$ be the elements of $\G$ in the sorted order (note that $g_0=0$). Since $\floors{s/Q}< 2n^2$, we have $r\le 4n^2$.

Let $\mGamma^{\textsc{MU}'}$ be a matrix with $n$ columns, such that for any allocation $\vec q=(g_{\iota_1},\dotsc,g_{\iota_n})\in\X'$ and any column $j$, $\Gamma^{\textsc{MU}'}_{\vec q,j} = \iota_j / r$.
Clearly, for any $\vec q \neq \vec q'\in\X'$ we have $\mGamma^{\textsc{MU}'}_{\vec q} \neq \mGamma^{\textsc{MU}'}_{\vec q'}$, so $\mGamma^{\textsc{MU}'}$ is $1/r$-admissible. Since $r\le 4n^2$, $\mGamma^{\textsc{MU}'}$ is also $1/(4n^2)$-admissible.

It is not hard to see that $\mGamma^{\textsc{MU}'}$ is also implementable with complexity $1$. For any column $j$, consider the bid profile $\vec v^j$ where bidder $j$'s marginal valuation is $1/r$ at the items
reaching the allocation levels $g_\iota\in\G$ and other bidders have $0$ valuation for any number of items. That is, $\mu_j(\ell) =\frac1r\mathbf{1}_{(\ell\in\G)}$ and $\mu_{i}(\ell) = 0$ for all $\ell$ and $i\neq j$. The welfare of any allocation $\vec q=(g_{\iota_1},\dotsc,g_{\iota_n})\in\X'$ on this bid profile is the utility of bidder $j$, which is $\iota_j/r = \Gamma^{\textsc{MU}'}_{\vec q,j}$. Therefore, $\mGamma^{\textsc{MU}'}$ is implementable with complexity $1$ using datasets $S_j=\{(1, \vec v^j)\}$ for all $j\in[n]$. By Theorem~\ref{thm:FTPL-U-MIDR-O} we therefore obtain the following:
\begin{theorem} \label{thm:knapsack}
Consider the problem of welfare maximization in $s$-unit auctions. Then Algorithm~\ref{alg:oracleftpl} with distribution $D$ as defined in Theorem~\ref{thm:FTPL-U},
datasets that implement $\mGamma^{\textsc{MU}'}$,
and the $1/2$-approximate MIR algorithm of \cite{dobzinski2007mechanisms} as an oracle,
runs in per-round time $\poly(n, T)$ and
plays the sequence of allocations $\vec q^{(1)}, \dots, \vec q^{(T)}$, such that
\[
\frac 12\dash\regret
=
\E\!\left[\frac12\left(
    \max_{\vec q \in \mathbb{N}^n :\:\norm{\vec q}_1 = s}
    \;
    \sum_{t=1}^T \sum_{i=1}^n v^{(t)}_i(q_i)
    \right)
-
    \sum_{t=1}^T \sum_{i=1}^n v^{(t)}_i(q^{(t)}_i)
\right]
\leq O(n^4\sqrt{T}).%
\]
\end{theorem}
In the above theorem,
the extra factor of $n$ in the regret as compared to that implied by Theorem \ref{thm:FTPL-U-MIDR-O} is due to the fact that the maximum welfare in this problem is upper bounded by $n$.

	\subsection{Oracle-Efficient No-Regret Learning in Simultaneous Second-Price Auctions}
\label{sec:sispas}

In this section, we answer an open problem raised by Daskalakis and Syrgkanis~\citep{DS16} regarding the existence of an oracle-based no-regret algorithm for optimal bidding in Simultaneous Second-Price Auctions.
We show that our Oracle-Based Generalized FTPL algorithm used with an appropriate
implementable and admissible translation matrix can be used to obtain such an algorithm.

A  \emph{Simultaneous Second-Price Auction (SiSPA)}~\cite{christodoulou2008bayesian, bhawalkar2011welfare, feldman2013simultaneous} is a mechanism for allocating $k$ items to $n$ bidders. Each bidder $i\le n$ submits $k$ simultaneous bids denoted by a vector of bids $\vec b_i$. The mechanism allocates each item using a second-price auction based on the bids  solely submitted for this item, while breaking ties in favor of bidders with lower indices. For each item $j$, the winner is charged $p_j$, the second highest bid for that item (in the presence of ties, the numerical value of $p_j$ may coincide with the value of the winning bid).
Each bidder $i$ has a fixed combinatorial valuation function $v_i: \set{0,1}^k \rightarrow [0,1]$ over bundles of items. Then, the total utility of bidder $i$ who is allocated the bundle $\vq_i\in\set{0,1}^k$ is $v_i(\vq_i) - \vec p\inprod\vq_i$, where $\vec p$ is the vector of second largest bids across all items.

We consider the problem of optimal bidding in a SiSPA from the perspective of the last bidder. Hereafter, we drop the indices of the bidders from the notation.
From this perspective, the utility of the bidder only depends on its bid $\vec b$ and  the threshold vector $\vec p$ of the second largest bids. The  online optimal bidding problem is defined as follows.

\begin{definition}[Online  Bidding in SiSPAs~\cite{Syrgkanis2016}]
At each round $t$, the bidder picks a bid vector $\vec b^{(t)}$ and an adversary picks a threshold vector $\vec p^{(t)}$. The bidder wins the bundle of items $\vq(\vec b^{(t)}, \vec p^{(t)})$, with $q_j(\vec b^{(t)}, \vec p^{(t)})= \mathbf{1}_{(b^{(t)}_{j} > p_j^{(t)})}$ and gets the utility
\[ u(\vec b^{(t)}, \vec p^{(t)}) = v\left( \vq(\vec b^{(t)}, \vec p^{(t)}) \right) - \vec p^{(t)}\inprod\vq(\vec b^{(t)}, \vec p^{(t)}).
\]
\end{definition}
We consider this problem under the \emph{no-overbidding condition} that requires that for any bundle $\vq$, the sum of bids over items in $\vq$ does not exceed the bidder's valuation for $\vq$, i.e., $\vec b^{(t)}\inprod\vq\le v(\vq)$, for all $\vq\in\set{0,1}^k$.
Similar no-overbidding assumptions are used in the previous work to prove that  no-regret learning in second-price auctions has good welfare guarantees~\cite{christodoulou2008bayesian,feldman2013simultaneous}.

We consider the online bidding problem where all the bids $b_j$ are multiples of $1/m$ and the valuation function $v(\cdot)$ only takes values in $\{0, 1/m, \dots,  m/m\}$.
We represent by $\B_m$ the class of all such bid vectors that satisfy the no-overbidding condition for $v(\cdot)$.
The assumption on the valuation function is not restrictive, because any valuation function $v(\cdot)$ can be rounded down to the closest multiples of $1/m$ while losing at most $1/m$ utility, and the set $\B_m$ derived for the rounded-down
valuation contains exactly the discretized bid vectors that satisfy no-overbidding condition for the original valuation.
Moreover, a similar  discretization for the bid vectors was used by Daskalakis and Syrgkanis~\cite{Syrgkanis2016} for studying offline and online optimal  bidding in SiSPAs.

Next, we show how to construct an implementable and admissible translation matrix for $\B_m$.

\paragraph{Construction of $\,\mGamma$:}
Let $\mGamma^{\textsc{OB}}$ be a matrix with $k$ columns and $|\B_m|$ rows that are equal to bid vectors, i.e., $\mGamma^{\textsc{OB}}_{\vec b}=\vec b$.

The next lemma shows that $\mGamma^{\textsc{OB}}$ is admissible and implementable.

\begin{lemma} \label{lem:impl-sispa}
$\mGamma^{\textsc{OB}}$ is $1/m$-admissible and
  implementable with complexity $m$.
\end{lemma}
\begin{proof}
Since $\mGamma^{\textup{\textsc{OB}}}_{\vec b} = \vec b$ and $\vec b$ is discretized, $\mGamma^{\textsc{OB}}$ is $1/m$-admissible.
Next, we show that the $j$th column of $\mGamma^{\textsc{OB}}$ can be implemented by a weighted dataset $S_j$ containing $m$ threshold vectors
where all but the $j$th threshold are set to 1.
Specifically, for $\ell=0,1,\dotsc,m-1$,
let $\vec  p_\ell=(\ell/m)\vec e_j +\sum_{j'\ne j}\vec e_{j'}$. Note that the utility of playing a bid $\vec b$ against $\vec  p_\ell$ is $u(\vec b, \vec  p_\ell) =\bigParens{v(\vec e_j) - \ell/m}\mathbf{1}_{(b_j >\ell/m)}$.
We set the weight corresponding to $\vec  p_\ell$ to
\[
 w_{\ell} = \begin{cases}
\frac1m\cdot\frac{1}{v(\vec e_j) - \ell/m}   &\text{if $\ell/m < v(\vec e_j)$,}  \\
0 & \text{otherwise.}
\end{cases}
\]
Since $b_j \leq v(\vec e_j)$ for any $\vec b$, we have
\begin{align*}
\sum_{\ell=0}^{m-1} w_\ell\,u(\vec b, \vec  p_\ell) &= \sum_{\ell=0}^{m-1} \frac{1}{m}\cdot\frac{1}{v(\vec e_j) - \ell/m} \cdot\BigParens{v(\vec e_j) - \ell/m}\mathbf{1}_{(b_j > \ell/m)} \\
&= \sum_{\ell=0}^{m-1} \frac{1}{m} \mathbf{1}_{(b_j > \ell/m)}
 = \sum_{\ell=0}^{mb_j-1} \frac{1}{m} \\
&= b_j = \Gamma^{\textsc{OB}}_{\vec b,j}.
\end{align*}
Thus, indeed $S_j = \set{(w_\ell,\vec p_\ell)}_{\ell}$ implements $\mGamma^{\textsc{OB}}$. Note that $|S_j|= m$, so $\mGamma^{\textsc{OB}}$ is implementable with complexity $m$.
\end{proof}

The next theorem is a direct consequence of Lemma~\ref{lem:impl-sispa} and Theorems~\ref{thm:FTPL-U} and \ref{thm:imp}.

\begin{theorem} \label{thm:sispa}
Consider the problem of  Online Bidding in SiSPAs.
Let $D$ be the uniform distribution as described in Theorem~\ref{thm:FTPL-U}. Then,
the Oracle-Based Generalized FTPL algorithm with $D$ and  datasets that implement $\mGamma^{\textup{\textsc{OB}}}$ is oracle-efficient with per-round complexity
$O(T+km)$ and has regret
\[
  \E\left[ \max_{\vec b \in \B_m}  \sum_{t=1}^T u(\vec b, \vec p^{(t)}) - \sum_{t=1}^T u(\vec b^{(t)}, \vec p^{(t)}) \right]  \leq  O(km\sqrt{T})
  .
\]
\end{theorem} 

\section{Conclusion}
\label{sec:conclusion}

In this paper, we have studied oracle-efficient online learning algorithms and applied them to online auction design.
Several previous works have also studied oracle-efficient online learning, including \citet{KV05}, whose Follow-the-Perturbed-Leader algorithm is oracle-efficient when the payoff is linear, and \citet{DS16}, whose algorithm is oracle-efficient when the adversary's action set is of polynomial size.
On the other hand, \citet{HK16} have shown that oracle-efficient online learning is not always achievable.
In this paper, we have advanced the study of conditions under which oracle-efficient online learning is possible and introduced a general-purpose oracle-efficient algorithm that works under weaker structural assumptions than previous approaches.
We also showed that many problems in auction design demonstrate such a structure.
As a result, our framework gave oracle-efficient online algorithms for a large number of problems in auction design when the environment, e.g., bidders' valuations, changes adversarially over a period of time.

Many open problems remain. First and more broadly, our work provides sufficient  conditions for existence of oracle-efficient online learning algorithms. Knowing that oracle-efficient algorithms do not always exist, a natural and important direction is to study necessary and sufficient conditions for their existence.
Second, our oracle-efficient algorithm obtains a regret of $O\bigParens{N\sqrt{T}/\delta}$ when there is a $\delta$-admissible and implementable translation matrix with $N$ columns. Forgoing implementability, any 0/1-valued payoff function has an admissible translation matrix with $N = \log|\X|$ columns, so the Generalized FTPL algorithm can always achieve the regret $O(\sqrt{T}\log|\X|)$, albeit not efficiently. However, this means that there is a gap of $\sqrt{\log|\X|}$ between the optimal information theoretic regret bound, i.e., $O(\sqrt{T \log|\X|})$,  and the one obtained by our algorithm. Are there oracle-efficient algorithms that obtain the optimal information theoretic regret bounds for problems that have admissible and implementable translation matrices? Answering these questions will enable a deeper understanding of the inherent tradeoffs between the computational and information theoretic aspects of online learning.

\bibliographystyle{apalike}
\bibliography{Nika_bib}

\clearpage
\appendix

\section{Further Related Work} \label{app:related}

\paragraph{Online Learning}
Study of online no-regret algorithms goes back to the seminal work of Hannan~\cite{Han57}, who was the first to develop algorithms with regret $\poly( |\X|) o(T)$. Motivated by settings in machine learning, game theory, and optimization, algorithms with regret $\log(|\X|) o(T)$ were later developed  by
\citet{LW94} and \citet{FS97}. From the computational perspective, however, algorithms with runtime that has sub-linear dependence on $|\X|$ remained elusive. \citet{HK16} recently showed such algorithms cannot exist  without additional assumptions on the structure of the optimization task.
Our work introduces general structural conditions under which online oracle-efficient learning is possible and provides oracle-efficient no-regret results in this space.

For the case of linear objective functions, \citet{KV05} proposed the first oracle-efficient online algorithm. Their algorithm added perturbations that are independent across dimensions  to the historical cumulative payoff of each  action of the learner and then picked the action with the largest perturbed historical payoff. However, for non-linear functions, this algorithm runs in time $\poly(|\X|)$.

\citet{DS16} proposed an approach of augmenting the history of an online algorithm with additional ``fake'' historical samples of the adversary actions.
In particular, the authors show that when the set of adversary's actions, $\Y$,  is small, one can add a random number of copies of each of the actions of the adversary to the history and obtain a regret of $O(|\Y| \sqrt{T})$ with oracle-efficient per-round complexity of $\poly(|\Y|, T)$.
However, the adversary's set of actions is often large in many applications, e.g., applications to auction design,
where this approach does not lead to a no-regret oracle-efficient algorithm.

Our work fills in this gap, by generalizing and extending the algorithm of \citet{DS16} with arbitrary set of actions of the adversary that meet the admissibility requirement.
Additionally, our work introduces what can be thought of as a measure of statistical and computational complexity of the problem that takes into account the effective size of a decomposition of the set of learner's actions, $\X$, on the actions of the adversary, $\Y$, not just the size of the learner's set of actions, as in \cite{FS97,KV05,  LW94}, or the adversary's set of actions, as in \cite{DS16}.

Prior to our work, stylized algorithms with improved regret bounds or oracle-efficient learning have existed in multiple domains.
\citet{HazanKLM16} study the regret bound of the problem of learning from experts, when the losses of experts are restricted to a low dimensional subspace. They show that one can obtain a regret bound that only depends on the  rank of this space, rather than the number of experts.
In another work, \citet{HK12} study online submodular minimization. They show that it is not necessary to add a perturbation to each action, but it suffices to add a perturbation on each underlying element of the set. This is a special case of our approach where $\mGamma$ is implemented by datasets that include single-element sets.

Contextual learning in the transductive setting using an optimization oracle was previously studied by~\citet{Kakade2005}, whose work was later extended and improved by~\citet{CS11}.
\citet{Syrgkanis2016} further extended the results to more general settings (including the bandit setting) based on the concept of small separator, which appears in prior work under the name \emph{universal identification sequence}~\citep{goldman1993exact}. Our definitions of admissibility and implementability can be viewed as non-contextual generalizations of the concept of the separator (see \Sec{context:ext}). We also show how to use the separator framework of \citet{Syrgkanis2016} to extend our algorithm to contextual settings.
Contextual bandits with an oracle have also been heavily studied in recent years~\citep{LangfordZh08, DHKKLRZ11, AgarwalHsKaLaLiSc14, RS16, SyrgkanisLuKrSc16}.

\paragraph{Auction Design}

The seminal work of \citet{Myerson1981} gave a recipe for designing the optimal truthful auction when the distribution over bidder valuations is completely known to the auctioneer. More recently, a series of works have focused on what can be done if the the distribution is not fully known to the auctioneer, but the auctioneer has
access to some historical data.
These approaches have focused on the sample complexity \citep{Cole2014, morgenstern2015pseudo} and computationally efficiency~\citep{DHP16} of designing optimal auctions from sample valuations. That is, they answer how much data is needed for one to  design a near optimal auction for a distribution.
They use tools from supervised learning, such as pseudo-dimension~\cite{Pol84} (a variant of VC dimension for real-valued functions) and compression bounds~\cite{littlestone1986relating}. These tools are specially appropriate for when historical data in form of i.i.d samples from a distribution is available to the learner and their use  does not extend to the adversarial online learning setting which we study.
Our work drops any distributional assumptions and considers learning the best auction among a class of auctions in the adversarial setting. However, we show that our algorithms can be used to obtain sample complexity results for more general stochastic settings, such as fast mixing markov processes.

On a different front, a series of works have considered online learning in auctions and pricing mechanisms, when the size of the auction class is itself small. Variations of this problem have been studied by \citet{kleinberg2003value,BH05,CGM13}.
In this case, standard no-regret algorithms (e.g., \cite{KV05,FS97}) obtain some no-regret results in polynomial time. However, additional effort is needed to design stylized algorithms for individual auction classes that obtain optimal regret, possibly even when the learner receives limited historical feedback. Our approach on the other hand works with classes of auctions that have exponentially large cardinality, thereby allowing one to use more expressive auctions to gain higher revenue.

More recently, \citet{roughgarden2016minimizing} studied  online learning in the class
of single-item second-price auctions with bidder-specific reserves. The authors introduce an algorithm with performance that approaches a constant factor of the optimal revenue in hindsight.
Our work goes beyond this auction class and introduces a general adversarial framework for auction design.

\section{Omitted Proofs from Section~\ref{sec:ftpl}}
\label{app:ftpl}

\subsection{Proof of Lemma~\ref{lem:stab+noise_approx}}
\label{app:approx}

We first prove an approximate variant of Be-the-Leader Lemma.
\begin{lemma}[Be-the-Approximate-Leader Lemma]\label{lem:BTL_approx}
In the Generalized FTPL algorithm, for any $x \in \X$,
\[ \sum_{t=1}^T  f(x_{t+1}, y_t)  +  \vec \alpha \cdot  \mGamma_{x_{1}}\geq
    \sum_{t=1}^T  f(x, y_t) +  \vec \alpha \cdot  \mGamma_{x} - \epsilon(T+1)
\enspace.
\]
\end{lemma}
\begin{proof}
For $T = 0$ the inequality holds, because by $\epsilon$-optimality of $x_1$, we have
$\vec \alpha \cdot  \mGamma_{x_{1}}\geq \vec \alpha \cdot  \mGamma_{x} - \epsilon$ for all~$x$.
Assume that the claim holds for some $T$. Then, for all $x$,
\begin{align*}
\sum_{t=1}^{T+1}  f(x_{t+1}, y_t) +  \vec \alpha \cdot  \mGamma_{x_{1}}
  &= \sum_{t=1}^T   f(x_{t+1}, y_t) +  \vec \alpha \cdot  \mGamma_{x_{1}}+ f(x_{T+2}, y_{T+1})  \\
   &\geq \sum_{t=1}^T   f(x_{T+2}, y_t) +  \vec \alpha \cdot  \mGamma_{x_{T+2}}  - \epsilon (T+1) +  f(x_{T+2}, y_{T+1})   \tag{by induction hypothesis} \\
     &= \sum_{t=1}^{T+1}  f(x_{T+2}, y_t) +  \vec \alpha \cdot  \mGamma_{x_{T+2}} - \epsilon (T+1)\\
     &\geq \sum_{t=1}^{T+1}   f(x, y_t) +  \vec \alpha \cdot  \mGamma_{x} - \epsilon (T+2)\enspace,  \tag{by $\epsilon$-optimality of $x_{T+2}$}
\end{align*}
proving the lemma.
\end{proof}

\begin{proof}[Proof of Lemma~\ref{lem:stab+noise_approx}]
Let $ x^* = \argmax_{x\in \X} \sum_{t=1}^T f(x, y_t)$. Then by Lemma~\ref{lem:BTL_approx},
\begin{align*}
\regret &= \E\left[ \sum_{t=1}^T  f(x^*, y_t)   -  \sum_{t=1}^T  f(x_t, y_t) \right]
\\
& = \E\left[ \sum_{t=1}^T  f(x^*, y_t) -  \sum_{t=1}^T  f(x_{t+1}, y_t) \right]  + \E\left[  \sum_{t=1}^T  f(x_{t+1}, y_t)  -  \sum_{t=1}^T  f(x_t, y_t) \right]
\\
\tag*{\qedhere}
& \leq \E\left[  \vec \alpha \cdot \left( \mGamma_{x_1} -  \mGamma_{x^*} \right)  \right]  + \E\left[  \sum_{t=1}^T  f(x_{t+1}, y_t)  -  \sum_{t=1}^T  f(x_t, y_t) \right] + \epsilon (T+1)
\enspace.
\end{align*}
\end{proof}

\section{Omitted Proofs from Section~\ref{sec:auctions}}
\label{app:auctions}

\subsection{Proof of Equation~\eqref{eq:Im=I}}  \label{app:Im=I}

Consider any vector of reserves $\vec r\in \I$ and let
$\vec r'\in \I_m$ be the vector obtained by rounding each reserve
price down to the nearest multiple of $1/m$, except for reserves that lie in the range $[0,1/m)$, which are rounded up to $1/m$.

First we show that by increasing all the reserves in $\vec r$ that are below $1/m$ to the value $1/m$, we cannot decrease the revenue by more than $(\card{W}-\card{W''})/m$ where $W$ is the set of winners under $\vec r$ and $W''$ is the set of winners under the new reserves.
Let $\vec r''$ be the vector of new reserves and let $p_i$ and $p_i''$ be the payments of bidder $i$ in auctions $\vec r$ and $\vec r''$. Let $v_i$ be the valuation of bidder $i$. Let $b$ and $b''$ denote the highest unserviced bids that cleared their reserve under $\vec r$ and $\vec r''$, respectively, or $0$ if all the bidders that cleared their reserve were serviced.
Recall that $W$ and $W''$ are the sets of winners under $\vec r$ and $\vec r''$, respectively.
First, note that the set of bidders that clear the reserves under $\vec r''$ is always a subset of those that clear the reserves under $\vec r$. However,
if $b\ge 1/m$ then all the bids in $W$ must be at least $1/m$, and so all the bidders in $W$ will clear their reserves under $\vec r''$, and therefore $W''=W$. Moreover, also the highest unserviced bidder under $\vec r$ clears the new reserve, and so $b''=b$. Thus, for all the winners $i$, $p''_i=\max\set{r''_i,b''}=\max\set{r_i,1/m,b}=\max\set{r_i,b}=p_i$, so the revenue under new reserves is the same as under old reserves.
Now consider the case where $b<\frac 1m$. This means that all the unserviced bids that cleared the reserves $\vec r$ are lower than $\frac 1m$, and so they do not clear the new reserves $\vec r''$. As a result, the set of bidders that clear the new reserves is a subset of $W$, so $b''=0$ and $W''\subseteq W$. In fact, $W''$ in this case consists exactly of the bidders that clear the new reserves. For a bidder $i\in W''$ we have
\[
   p''_i = \max\{r''_i, b''\} = \max\set{r_i,1/m} \ge \max\set{r_i,b} = p_i,
\]
so the revenue obtained from $i\in W''$ can only increase under $\vec r''$ compared with $\vec r$. The bidders $i\in W\backslash W''$, that is those who win an item under $\vec r$ but not under $\vec r'$, must have a bid (and therefore a payment) smaller than $1/m$. Therefore, the overall revenue decreases by at most $(\card{W}-\card{W''})/m$.

Now we assume that we start with a reserve price $\vec r$, such that $r_i\geq 1/m$ for all $i$, and let $\vec r'$ be the reserve vector where all reserves in $\vec r$ are rounded down to the nearest multiple of $1/m$. If $v_i > r_i$, then
$v_i > r'_i$, so any bidder who would have been included in the basic VCG auction using reserves $\vec r$ is still included with
$\vec r'$. This can only increase the number of bidders who are serviced and therefore pay a charge. We now show also that the total decrease in payment from bidders with value at least $1/m$ is at most $s/m$.

Consider the amount that serviced bidder $i$ is charged. This is
the maximum of $r_i$ and the highest bid of a bidder in the basic VCG
auction who was not serviced (or 0 if all bidders were serviced); let
$b$ denote this highest unserviced bid in the basic VCG auction under $\vec r$,
and similarly let $b'$ denote such a bid under $\vec r'$. Since the set
of bidders entering the basic VCG auction increases from $\vec r$ to $\vec r'$,
we must have $b' \geq b$.

Let $U$ be the set of bidders serviced under $\vec r$, and $U'$ the set
under $\vec r'$. The difference in revenue is
\begin{align}
\notag
&\sum_{i \in U}\max\{r_i, b\} - \sum_{i \in U'} \max\{r'_i, b'\}
\\[-3pt]
\label{eq:l:item}
&\qquad\qquad{}
=\!\!\!\sum_{i \in U \cap U'}\!\!\!(\max\{r_i, b\} - \max\{r'_i, b'\})
+\!\!\!\sum_{i \in U\setminus U'}\!\!\!\max\{r_i, b\}
-\!\!\!\sum_{i' \in U'\setminus U}\!\!\!\max\{r'_{i'}, b'\}.
\end{align}
We begin by analyzing the last two terms.
For any $i \in U\setminus U'$ and $i' \in U'\setminus U$,
\[
r'_{i'} + 1/m > r_{i'} > v_{i'} \geq v_i \mathrel{\ge} r_i,
\]
where $r_{i'}>v_{i'}$, because $i'$ did not enter the basic VCG auction for $\vec r$ (otherwise it would have
won), and $v_{i'}\geq v_i$, because $i$ enters the basic VCG auction for $\vec r'$, but is not
allocated the item.
Therefore,
\[
\max\{r'_{i'}, b'\} \geq \max\bigSet{r_i - 1/m,\,b} \geq
\max\{r_i, b\} - 1/m.
\]
Since $\card{U\setminus U'}\le\card{U'\setminus U}$, we can pick $V\subseteq U'\setminus U$ such that
$\card{V}=\card{U\setminus U'}$ and obtain
\begin{align*}
  \sum_{i \in U\setminus U'}\!\!\!\max\{r_i, b\}
-\!\!\!\sum_{i' \in U'\setminus U}\!\!\!\max\{r'_{i'}, b'\}
\le
  \sum_{i \in U\setminus U'}\!\!\!\max\{r_i, b\}
-\!\sum_{i' \in V}\!\max\{r'_{i'}, b'\}
\le
\frac{\card{U\setminus U'}}{m}
\enspace.
\end{align*}
Note that each term in the first sum of \Eq{l:item} is at most $1/m$, because
\[
  \max\set{r'_i,b'}\ge\max\bigSet{r_i-1/m,\,b}\ge\max\set{r_i,b}-1/m.
\]
Thus, we have
\[
  \rev(\vec r,\vec v)-\rev(\vec r',\vec v) \le
  \frac{\card{U\cap U'}}{m}
+ \frac{\card{U\setminus U'}}{m}
  = \frac{\card{U}}{m}
\enspace.
\]

Thus if we start from any reserve price vector $\vec r$, we can first round up to $1/m$ all reserves that are below $1/m$ and then round down any other reserve to the nearest multiple of $1/m$. Denoting the initial set of winners $W$,
the intermediate set of winners $U$ and the final set of winners $U'$, we have shown that the revenue can drop by at most $(\card{W}-\card{U})/m$ in the first step and by at most $\card{U}/m$ in the second step, so the overall drop
is at most $\card{W}/m\le s/m$. This yields the approximation result
\begin{equation}
\tag*{\qed}
  \max_{\vec r \in \I}  \sum_{t=1}^T \rev (\vec r, \vec v_t)  - \max_{\vec r \in \I_m}  \sum_{t=1}^T \rev (\vec r, \vec v_t) \leq \frac {T s}{m}.
\end{equation}

\subsection{Proof of Lemma~\ref{lem:impl_envy}}  \label{app:impl_envy}

We will describe an isomorphism with the setting in the proof of
Lemma~\ref{lem:impl_individual}, which will allow us to directly apply the analysis of $\mGamma^{\textsc{VCG}}$
to $\mGamma^{\textsc{IP}}$.
The isomorphism from the IP setting to VCG setting maps
items $\ell$ in IP to bidders $i$ in VCG, and price vectors $\vec a$ to reserve price vectors $\vec r$.
We therefore assume that $n$ in VCG equals $k$ in IP,
and the values of $m$ in both settings are equal. Then, indeed $\mGamma^{\textsc{VCG}}$
equals $\mGamma^{\textsc{IP}}$.

Next we need to show how to construct $S_j$ for all $j$ in the $\mGamma^{\textsc{IP}}$ setting. Assume that $j$
corresponds to the bidder $i$ and the bit $\beta$ in VCG setting, and the item $\ell$
and the bit $\beta$ in IP setting.
In VCG, we considered the bid profiles $\vec v_\ellOther=(\ellOther/m)\vec e_i$,
and the revenue of any auction $\vec r$ is
\[
  \rev^{\textsc{VCG}}(\vec r, \vec v_\ellOther) = r_i \mathbf{1}_{(\ellOther \geq mr_i )}
.
\]
In IP setting, we consider profiles $\vec v'_\ellOther$
of combinatorial valuations over bundles $\vq\in\set{0,1}^k$, in which
all bidders have values zero on all bundles and one bidder has value $\ellOther/m$ for
bundle $\vec e_\ell$ and zero on all other bundles.\footnote{%
Note that a simple variation of this bid profile can be used in settings where the valuations need to satisfy additional assumptions, such as (sub-)additivity  or  free disposal. In such cases, we can
use a similar bid profile where one bidder has valuation $\ellOther/m$ for any bundle that includes item $\ell$ and all other valuations are $0$.}
In this case, we have
\[
  \rev^{\textsc{IP}}(\vec a, \vec v'_\ellOther) = a_i \mathbf{1}_{(\ellOther \geq ma_i )}
.
\]
Thus, both the translation matrices $\mGamma^{\textsc{VCG}}$ and $\mGamma^{\textsc{IP}}$ as well as
the revenue functions $\rev^{\textsc{VCG}}$ and $\rev^{\textsc{IP}}$ are isomorphic (given these choices of
the profiles). Therefore, we can set the weights $w'_\ellOther$ in IP setting
equal to the weights $w_\ellOther$ in VCG setting and obtain admissibility and implementability
with the same constants and complexity.
\qed

\subsection{Proof of Lemma~\ref{lem:P-grid}}\label{app:P-grid}

\paragraph{Single-minded setting.}
In the single-minded setting, each bidder $i$ is interested in one particular bundle of items $\hat\vq_i$. That is, $v_i(\vq_i) = v_i(\hat{\vq}_i)$ for all $\vq_i \supseteq \hat{\vq}_i$ and $0$ otherwise. Consider any vector of prices $\vec a \in \P$ and let $\vec a'\in \P_m$ be the vector obtained by rounding up all prices that are below $k/m$ to the nearest multiple of $1/m$ and reducing all prices that are at least $k/m$ by $(k-1)/m$ and then rounding down to the nearest multiple of $1/m$. More formally,
\[ a'_\ell = \begin{cases}
			 \ceils{a_\ell}_{\!\frac1m}
                  &\text{if }a_\ell < \frac km,\\
			 \Floors{a_\ell - \frac{k-1}{m}}_{\!\frac1m}
                  &\text{if }a_\ell \geq \frac km,
		  \end{cases}
\]
where $\lceil \cdot \rceil_{\epsilon}$ and $\lfloor \cdot \rfloor_{\epsilon}$ indicate rounding up and down to the nearest multiple of $\epsilon$, respectively. Note that $a'_\ell\in \{1/m, \dots, (m-1)/m, 1\}$ by definition.
Since $a_\ell\ge 0$ for all $\ell$, we can assume without loss of generality that bidder $i$ receives either the bundle $\vec 0$ or $\hat\vq_i$. Only
the bidders that receive $\hat\vq_i$ contribute towards the revenue under $\vec a$. For each such bidder $i$, there are two cases:
\begin{enumerate}
\item All items $\ell$ in bundle $\hat\vq_i$ are such that $a_\ell< k/m$. Since there are at most $k$ items included in $\hat\vq_i$, we have that $\hat\vq_i \cdot \vec a < k^2/m$, i.e., the original payment for $\hat\vq_i$ is at most $k^2/m$.
Since the payment of the bidder $i$ under $\vec a'$ is non-negative, regardless of the bundle it receives, the drop in the payment amount by the bidder $i$ is at most $k^2/m$.
\item There is an item $j$ in bundle $\hat\vq_i$ such that $a_j \geq k/m$. This means that there can be at most $k-1$ items whose original price was less than $k/m$, and each of these prices has increased by at most $1/m$.
These items increase the price of the bundle $\hat\vq_i$ by at most $(k-1)/m$. Since $a'_j\le a_j - (k-1)/m$, item $j$ decreases the price of the bundle by at least $(k-1)/m$. Altogether, we therefore have that $\hat\vq_i \cdot \vec a' \leq\hat\vq_i \cdot \vec a$, so the bundle $\hat\vq_i$ is assigned to the bidder $i$ under $\vec a'$. We also have that $a'_\ell\ge a_\ell-k/m$ for all $\ell$, and since there are only $k$ items, the price of the bundle $\hat\vq_i$
drops by at most $k^2/m$.
\end{enumerate}
Thus, in both cases, the payment of bidder $i$ for the bundle it receives under $\vec a'$ is lower by at most $k^2/m$ than the payment under $\vec a$. Therefore, for any valuation profile of single-minded bidders $\vec v$,
\[ \rev (\vec a, \vec v)  - \rev(\vec a', \vec v) \leq \frac {nk^2}{m}.
\]

\paragraph{Unit-demand setting.}

In the unit-demand setting with infinite supply, each bidder $i$ has  $v_{i}(\vec e_\ell)$ for item $\ell$, and wishes to purchase \emph{at most one item}, i.e., item $ \argmax_{\ell} \left( v_i(\vec e_\ell) - a_\ell \right)$.
We show that for any $\vec a\in \P$ there is $\vec a'' \in \P_m$ such that for any valuation profile $\vec v$, $\rev(\vec a, \vec v) - \rev(\vec a'', \vec v) \leq  O(nk/m)$.
At a high level, we first choose $\vec a'$, with $a'_\ell \in \{0, 1/m, \cdots, m/m\}$, as discounted prices such that items with higher prices are discounted by larger amounts. It is not hard to see that under this condition, no bidder purchases a less expensive item in auction $\vec a'$. So, the loss in the revenue of the auction is bounded by the discount on the items.
Using this intuition,  it is sufficient to find $\vec a'$, with $a'_\ell \in \{0, 1/m, \cdots, m/m\}$,  such that $a_\ell \geq a'_\ell \geq a_\ell - O(k/m)$ for all  $\ell\le k$, and $a_\ell - a'_\ell \ge a_{\ell'} - a'_{\ell'}$ when  $a_\ell \ge a_{\ell'}$. We then show how to derive $\vec a'' \in \P_m$ whose revenue is at least as good as $\vec a'$.

Without loss of generality, assume $a_1\le a_2 \le\dotsb\le a_k$. For ease of exposition,
let $\epsilon = 1/m$. To begin, let $a'_1$ be the largest multiple of $\epsilon$ less than or equal
to $a_1$.  For $\ell = 2, \ldots, k$, let $a'_\ell$ be the largest
  multiple of $\epsilon$ less than or equal to $a_\ell$ such that
  $a_\ell - a'_\ell \geq a_{\ell-1} - a'_{\ell-1}$. Note that
  $a'_\ell$ is well defined, since we can begin by considering $a'_\ell=a'_{\ell-1}$
  and then increase by $\epsilon$ until the condition is violated.
  This construction means that pricing of items with a larger $\ell$
  is more attractive in $\vec a'$ than it was in $\vec a$.
  Thus, no bidder that prefers an item $\ell$ under $\vec a$, would prefer any item $\ell'<\ell$ under $\vec a'$.
  Also, $a'_1\le a'_2\le\dotsb\le a'_k$.
  Therefore, the revenue obtained from the bidder $i$ who prefers $\ell$ under $\vec a$ is at least $a'_\ell$ under $\vec a'$,
  which implies the bound
  \[
    \rev (\vec a, \vec v) - \rev ( \vec a', \vec v)
    \le
    n\max_{\ell\le k}(a_\ell-a'_\ell).
  \]
  To complete the proof, we argue by induction that $a_\ell-a'_\ell\le\ell\epsilon$. This clearly holds for $\ell=1$.
  For $\ell\ge 2$, the definition of $a'_\ell$, in the absence of discretization to $\epsilon$, would yield $a'_\ell=a_\ell-(a_{\ell-1} - a'_{\ell-1})$. With
  the discretization, we have
  \[
    a'_\ell\ge a_\ell-(a_{\ell-1} - a'_{\ell-1})-\epsilon
           \ge a_\ell-(\ell-1)\epsilon-\epsilon
           = a_\ell-\ell\epsilon,
  \]
  where we used the inductive hypothesis at $\ell-1$.

Next, we construct $\vec a''$ by setting $a''_\ell = \epsilon$ when $a'_\ell = 0$, and $a''_\ell =a'_\ell$ otherwise.
We show that  any bidder $i$ that purchased some item $\ell$ in auction $\vec a'$ with  price $a'_{\ell}\geq \epsilon$, purchases the item $\ell$ in auction $\vec a''$ as well.

Assume to the contrary that bidder $i$ purchases another item $\ell'$. Since there is infinite supply, bidder $i$ could have purchased item $\ell'$ in auction $\vec a'$, but preferred to purchase item $\ell$. Therefore,
\[ v_i(\vec e_\ell) - a'_\ell \geq v_i(\vec e_{\ell'}) - a'_{\ell'}.
\]
Note that $a''_\ell = a'_\ell$ and $a''_{\ell'} \geq a'_{\ell'}$ for all $\ell'\neq \ell$. So,
 \[ v_i(\vec e_\ell) - a''_\ell
   \geq v_i(\vec e_{\ell'}) - a''_{\ell'}.
\]
Therefore, bidder $i$ purchases item $\ell$ in auction $\vec a''$ as well.

Since only the bidders who receive items at  price $0$ might not be allocated the same items, we have that $\rev(\vec a'', \vec v) \geq \rev(\vec a', \vec v)$. This completes the proof.
\qed

\subsection{Proof of Lemma~\ref{lem:admissible_Sm}}\label{app:admissible-Sm}

Since $\Gamma^{\textsc{SL}}_{\vec \theta, \vec v} = \rev( \vec \theta, \vec v)$,
$\mGamma^{\textsc{SL}}$ can be implemented by datasets
$S_{\vec v} = \{ (1, \vec v)\}$ for $\vec v\in V$.
So, $\mGamma$ is implementable with complexity  $1$.

\begin{figure}[h]
\begin{subfigure}[b]{0.49\textwidth}
\setlength{\unitlength}{0.65cm}
\thicklines
\begin{picture}(12,4)(1,0)
\thicklines
\put(2,4){\line(1,0){8}}
\put(10.2,3.9){\scriptsize{Bidder $i$}}
\put(2,3){\line(1,0){8}}
\put(10.2,2.9){\scriptsize{Bidder $n$}}
\put(2,1.5){\line(1,0){8}}
\put(10.2,1.4){\scriptsize{Bidder $i$}}
\put(2,.5){\line(1,0){8}}
\put(10.2,0.4){\scriptsize{Bidder $n$}}

\put(1.2,3.4){$\mGamma_{\vec\theta}$}

\put(2.5,3.8){\line(0,1){.4}}
\put(3.5,3.8){\line(0,1){.4}}
\put(6,3.8){{\color{mygray} \linethickness{3pt}\line(0,1){.6}}}
\put(5.85,3.4){\footnotesize{$\vec\theta_b^i$}}
\put(9,3.8){\line(0,1){.4}}
\put(10,4){{\color{mygray}\circle*{0.2}}}
\put(9.85,3.55){{\color{mygray}\footnotesize{$v_i$}}}

\put(3,2.8){\line(0,1){.4}}
\put(4,2.8){\line(0,1){.4}}
\put(7,2.8){\line(0,1){.4}}
\put(9.5,2.8){\line(0,1){.4}}
\put(7,3){{\color{mygray}\circle*{0.2}}}
\put(6.9,2.5){{\color{mygray}\footnotesize{$v_n$}}}

\put(1.2,0.9){$\mGamma_{\vec\theta'}$}

\put(2.5,1.3){\line(0,1){.4}}
\put(3.5,1.3){\line(0,1){.4}}
\put(6.6,1.3){{\color{mygray} \linethickness{3pt}\line(0,1){.6}}}
\put(6.45,0.9){\footnotesize{$\vec\theta'^i_b$}}
\put(9,1.3){\line(0,1){.4}}
\put(10,1.5){{\color{mygray}\circle*{0.2}}}
\put(9.85,1.05){{\color{mygray}\footnotesize{$v_i$}}}

\put(3,0.3){\line(0,1){.4}}
\put(4,0.3){\line(0,1){.4}}
\put(7,.3){\line(0,1){.4}}
\put(9.5,.3){\line(0,1){.4}}
\put(7,.5){{\color{mygray}\circle*{0.2}}}
\put(6.9,0){{\color{mygray}\footnotesize{$v_n$}}}
\end{picture}
\caption{Case 1}
\end{subfigure}
~
\begin{subfigure}[b]{0.49\textwidth}
\setlength{\unitlength}{0.65cm}
\begin{picture}(12,4)(1,0)\thicklines
\put(2,4){\line(1,0){8}}
\put(10.2,3.9){\scriptsize{Bidder $1$}}
\put(2,3){\line(1,0){8}}
\put(10.2,2.9){\scriptsize{Bidder $n$}}
\put(2,1.5){\line(1,0){8}}
\put(10.2,1.4){\scriptsize{Bidder $1$}}
\put(2,.5){\line(1,0){8}}
\put(10.2,0.4){\scriptsize{Bidder $n$}}

\put(1.2,3.4){$\mGamma_{\vec\theta}$}

\put(2.5,3.8){\line(0,1){.4}}
\put(6,3.8){\line(0,1){.4}}
\put(3.5,3.8){{\color{mygray} \linethickness{3pt}\line(0,1){.6}}}
\put(3.35,3.4){\footnotesize{$\vec\theta_b^1$}}
\put(9,3.8){\line(0,1){.4}}
\put(10,4){{\color{mygray}\circle*{0.2}}}
\put(9.85,3.55){{\color{mygray}\footnotesize{$v_1$}}}

\put(3,2.8){\line(0,1){.4}}
\put(5,2.8){\line(0,1){.4}}
\put(7,2.8){\line(0,1){.4}}
\put(9.5,2.8){\line(0,1){.4}}
\put(5,3){{\color{mygray}\circle*{0.2}}}
\put(5.15,3.15){{\color{mygray}\footnotesize{$v_n$}}}
\put(4.8,2.4){\footnotesize{$\vec\theta_b^n$}}

\put(1.2,0.9){$\mGamma_{\vec\theta'}$}

\put(2.5,1.3){{\color{mygray} \linethickness{3pt}\line(0,1){.6}}}
\put(6,1.3){\line(0,1){.4}}
\put(3.5,1.3){\line(0,1){.4}}
\put(2.35,0.9){\footnotesize{$\vec\theta'^1_{b-1}$}}
\put(9,1.3){\line(0,1){.4}}
\put(10,1.5){{\color{mygray}\circle*{0.2}}}
\put(9.85,1.05){{\color{mygray}\footnotesize{$v_1$}}}

\put(3,0.3){\line(0,1){.4}}
\put(6,0.3){\line(0,1){.4}}
\put(7,.3){\line(0,1){.4}}
\put(9.5,.3){\line(0,1){.4}}
\put(5,.5){{\color{mygray}\circle*{0.2}}}
\put(5.15,.65){{\color{mygray}\footnotesize{$v_n$}}}
\put(5.8,-0.05){\footnotesize{$\vec\theta'^n_b$}}
\end{picture}
\caption{Case 2}
\end{subfigure}
\caption{Examples of cases 1 and 2 in the proof of Lemma~\ref{lem:admissible_Sm}. Bidder valuations are depicted as gray circles on the real line and the revenues of the two auctions $\vec \theta$ and $\vec \theta'$ are depicted as gray solid vertical lines.}
\label{fig:cases1-2}
\end{figure}

It remains to show admissibility. Take any two different auctions $\vec \theta$ and $\vec \theta'$. We will show that $\mGamma^{\textsc{SL}}_{\vec \theta} \neq \mGamma^{\textsc{SL}}_{\vec \theta'}$.
Let $b$ be the smallest level at which there is $i\in[n]$ such that $\theta_b^i \neq \theta'^i_b$ and among such $i$ choose the largest. There are three cases (see Figure~\ref{fig:cases1-2} for cases 1 and 2):

\begin{enumerate}
\item  $i<n$: Consider the bid profile $\vec v^{i,\ell}$ for $\ell= m \theta^n_b$, in which bidder $n$'s bid equals $\theta^n_b=\theta'^n_b$, meaning that his level in both $\vec\theta$ and $\vec\theta'$ is $b\ge 0$. Since bidder $i$'s bid equals $1$, his level is $s-1\ge 1$ under any auction, and bidders other than $n$ and $i$ have levels equal to $0$ or $-1$. Thus, bidder $i$ wins the item in both $\vec\theta$ and $\vec\theta'$,
    and pays the $b^{\text{th}}$ threshold. So, $\rev(\vec \theta, \vec v^{i,\ell}) = \theta^i_b  \neq  \theta'^i_b = \rev(\vec \theta', \vec v^{i,\ell})$.

\item  $i = n$ and $b \geq 1$:
Without loss of generality, assume that $\theta^n_b  <  \theta'^n_b$. Let $\ell = m \theta^n_b$ and consider $\vec v^{1,\ell}$. Bidder $n$'s bid equals $\theta^n_b$, so his level under $\vec\theta$ is $b$, whereas his level under $\vec\theta'$ is $b-1\ge 0$ (since all thresholds below $b$ are the same in both auctions). So, bidder $1$ wins the item in both auctions and pays the threshold that corresponds to bidder $n$'s level. Therefore, $\rev(\vec \theta, \vec v^{1, \ell}) = \theta^1_b  \neq  \theta^1_{b-1} = \theta'^1_{b-1} = \rev(\vec \theta', \vec v^{1, \ell})$.

\item $i = n$ and $b = 0$: Consider  bid profile $\vec e_n$. In this profile,  bidder $n$ wins and pays the reserve price, which is equal to his threshold at $b=0$. Therefore, $\rev(\vec \theta, \vec e_n) = \theta^n_0  \neq  \theta'^n_0 = \rev(\vec \theta', \vec e_n)$.
\end{enumerate}

Therefore, $\mGamma^{\textsc{SL}}_{\vec \theta} \neq \mGamma^{\textsc{SL}}_{\vec \theta'}$. Since any element of $\mGamma^{\textsc{SL}}$ is a multiple of $1/m$, $\mGamma^{\textsc{SL}}$ is $1/m$-admissible.
\qed 

\section{Omitted Proofs from Section~\ref{sec:overallopt}} \label{app:overallOPT}

\subsection{Proof of Lemma~\ref{lem:iid}}\label{app:iid}

Let $x^* = \arg\max_{x\in \X} \E_{y\sim F}\left[f(x,y)\right]$. By the
definition of regret we have that for any $y_1, \ldots, y_T$,
\[
\sum_{t=1}^T \E_{x_t}[f(x_t,y_t)] = \sup_{x\in \X} \sum_{t=1}^T f(x,y_t)-\regret \geq
\sum_{t=1}^T f(x^*,y_t) - \regret .
\]
Observe that the random variables $Z_t = f(x^*,y_t)$ are drawn i.i.d.\ with expected value $\E_{y\sim F}[f(x^*,y)]$ and are bounded in $[0,1]$. Hence, by the Hoeffding bound, we get that with probability at least $1-e^{-2 T \epsilon^2}$:
\begin{equation}
\frac{1}{T}\sum_{t=1}^T f(x^*,y_t)\geq \E_{y\sim F}[f(x^*,y)]- \epsilon .
\end{equation}
By setting $\epsilon=\sqrt{\frac{\log(1/\delta)}{2T}}$ and combining the two bounds we get the result.
\qed

\subsection{Proof of Lemma~\ref{lem:Markov}}\label{app:Markov}

Let $x^* = \arg\max_{x\in \X} \E_{y\sim F}\left[f(x,y)\right]$. As in the proof Lemma~\ref{lem:iid},
\[
\sum_{t=1}^T \E_{x_t}[f(x_t,y_t)] = \sup_{x\in \X} \sum_{t=1}^T f(x,y_t)-\regret \geq
\sum_{t=1}^T f(x^*,y_t) - \regret .
\]
Since $y_1,\ldots,y_T$ are a Markov chain, by applying Theorem \ref{thm:conc-markov} to this chain and to the function $f(x^*,\cdot)$, we obtain that with probability at least $1-2 \exp\left(-\frac{T \gamma \epsilon^2}{4+10\epsilon}\right)$,
\begin{equation}
\frac{1}{T}\sum_{t=1}^T f(x^*,y_t)\geq \E_{y\sim F}[f(x^*,y)]-\epsilon.
\end{equation}
If we set $\epsilon=\sqrt{\frac{14 \log(2/\delta)}{\gamma T}}$ then we have, either $\epsilon>1$, in which case the inequality is trivial, since $f(x,y)\in [0,1]$, or if $\epsilon \leq 1$, then $\epsilon=\sqrt{\frac{14 \log(2/\delta)}{\gamma T}} \geq \sqrt{\frac{(4+10\epsilon) \log(2/\delta)}{\gamma T}}$, which implies that the inequality holds with probability $1-\delta$.
\qed

\section{Omitted Proofs from Section~\ref{sec:contexts}}
\label{app:contexts}

\subsection{Proof of Lemma~\ref{lem:contextual_admissibility}}
\label{app:contexts-admissible}
First we argue that $\mGamma^\sep$ has distinct rows. We will show that for any two rows $\pi$, $\pi'$, there exists a column $(\sigma,j)$ on which they differ. Since $\sep$ is a separator, there exists a context $\sigma^*\in \sep$ on which $\pi(\sigma^*)\neq \pi'(\sigma^*)$. Now by the admissibility of the original matrix $\mGamma$, we know that for any two $x,x'\in \X$, there exists a column $j$ of the original matrix such that $\Gamma_{x,j}\neq \Gamma_{x',j}$. Applying the latter for $\pi(\sigma^*), \pi'(\sigma^*)\in \X$, we get that there exists a $j^*$, such that $\Gamma_{\pi(\sigma^*), j^*}\neq \Gamma_{\pi'(\sigma^*),j^*}$. By the definition of $\mGamma^\sep$, the latter implies that $\Gamma_{\pi, (\sigma^*,j^*)}^\sep \neq \Gamma_{\pi', (\sigma^*,j^*)}^\sep$. Thus the two rows $\pi,\pi'$ of matrix $\mGamma^\sep$ differ at column $(\sigma^*,j^*)$.

We now argue that if the quantity $\delta$ from Definition~\ref{defn:admissible} is valid for $\mGamma$, it is also valid for $\mGamma^\sep$. We remind the reader that $\delta$ is required to lower bound the minimum positive difference between any two elements of a column. Since the values appearing in the column $(\sigma,j)$ of matrix $\mGamma^\sep$ are all taken from the column $j$ of matrix $\mGamma$, any lower bound $\delta$ that is valid for $\mGamma$ is
also valid for $\mGamma^\sep$. Thus $\mGamma^\sep$ is also $\delta$-admissible.
\qed

\subsection{Proof of Lemma~\ref{lem:contextual_imp}}  \label{app:contexts-implementable}
The intuition of the proof is as follows: if we know that we can simulate every column in $\mGamma$ with a sequence of polynomially many weighted inputs $\{(w,y)\}$, then we can simulate each column of $\mGamma^\sep$ associated with context $\sigma \in \sep$ and column $j$ of $\mGamma$, by a sequence of weighted contextual inputs $\braces{\,\parens{w,(\sigma,y)}\,}$, which is essentially a contextually annotated copy of the sequence of inputs we used to simulate column $j$ of $\mGamma$. We now show this intuition more formally.

Since $\mGamma$ is implementable with complexity $M$, we have that for any  column
$j\in \{1,\ldots, N\}$ of matrix $\mGamma$ there exists a weighted dataset $S_j=\{(w,y)\}$, with $|S_j|\leq M$, such that:
\begin{align*}
\hspace{1in}
    &\text{for all $x,x'\in\X$:}
    &\Gamma_{xj}-\Gamma_{x'j} &= \sum_{(w,y)\in S_j} w\cdot \bigParens{f(x,y)-f(x',y)}.
&\hspace{1in}
\end{align*}

For the contextual problem, we need to argue that for any
column $(\sigma,j)\in \sep\times [N]$ of matrix $\mGamma^\sep$ there exists a
weighted contextual dataset $S_{\sigma,j}^c =\braces{\,\parens{w,(\sigma,y)}\,}$, with $|S_{\sigma,j}^c|\leq M$, such that:
\begin{align*}
\hspace{1in}
    &\text{for all $\pi,\pi'\in\Pi$:}
    &\Gamma_{\pi,(\sigma,j)}^\sep-\Gamma_{\pi',(\sigma,j)}^\sep &= \sum_{(w,(\sigma,y))\in S_{\sigma,j}^c} w\cdot \bigParens{f(\pi(\sigma),y)-f(\pi'(\sigma),y)}.
&\hspace{1in}
\end{align*}
The construction of such a contextual dataset is straightforward: for each $(w,y)\in S_j$, create a data point $\bigParens{w,(\sigma,y)}$ in $S_{\sigma,j}^c$. To finish the proof,
observe that
\begin{align*}
\Gamma_{\pi,(\sigma,j)}^\sep-\Gamma_{\pi',(\sigma,j)}^\sep =~& \Gamma_{\pi(\sigma),j}-\Gamma_{\pi'(\sigma),j} = \sum_{(w,y)\in S_j} w\cdot \bigParens{f(\pi(\sigma),y)-f(\pi'(\sigma),y)}
\\
\tag*{\qed}
=~&\sum_{(w,(\sigma,y))\in S_{\sigma,j}^c} w\cdot \bigParens{f(\pi(\sigma),y)-f(\pi'(\sigma),y)}.
\end{align*}

\subsection{Proof of Lemma~\ref{lem:transductive-stability}} \label{app:context-stability}

We will show that for each $t \leq T$,
\begin{equation}
\E[f_c(\pi_{t+1},(\sigma_{t}, y_t))- f_c(\pi_t, (\sigma_t,y_t))]\leq 2N\rho(1+\delta^{-1}).
\end{equation}
Since $f_c(\pi,(\sigma_t,y_t))=f(\pi(\sigma_t),y_t)$ and since $f(x,y)\in [0, 1]$, it suffices to show that $\Pr[\pi_{t+1}(\sigma_t)\neq \pi_t(\sigma_t)]\leq 2N\rho(1+\delta^{-1})$.

Observe that if $\pi_{t+1}(\sigma_t)\neq\pi_{t}(\sigma_t)$ then it must be that $\Gamma_{\pi_{t+1}(\sigma_t),j}\neq \Gamma_{\pi_t(\sigma_t),j}$ for some $j\in [N]$, by the admissibility of matrix $\mGamma$ for the non-contextual problem. Thus we need to show that the probability that there exists a $j$ such that $\Gamma_{\pi_{t+1}(\sigma_t),j}\neq \Gamma_{\pi_t(\sigma_t),j}$ is at most $2N\rho(1+\delta^{-1})$. By the union bound it suffices to show that for any given $j$, the probability that $\Gamma_{\pi_{t+1}(\sigma_t),j}\neq \Gamma_{\pi_t(\sigma_t),j}$ is at most $2\rho(1+\delta^{-1})$.

The proof of this fact then follows by identical arguments as in the proof of the non-contextual stability in Lemma \ref{lem:stability-approx} and we omit it for succinctness.
\qed

\subsection{Proof of Theorem~\ref{thm:transductive-FTPL-G}}

Using Lemma~\ref{lem:transductive-stability}, and since $D$ is uniform over $[-\nu,\nu]$ and is  $\left(\frac{1+2\epsilon}{2\nu \delta}, \frac{1+2\epsilon}{\delta}\right)$-dispersed, we have:
\[
\E\left[ \sum_{t=1}^T  f(x_t, y_t) - f(x_{t+1}, y_t) \right] \leq \frac{T N (1+2\epsilon)(1+\delta)}{\nu\delta^2}
\enspace.
\]

It remains to bound the term $\E\left[ \sum_{\sigma\in \tsep}\sum_{j=1}^{N} \alpha_{\sigma,j}( \Gamma_{\pi^*,(\sigma,j)}^\tsep  -  \Gamma_{\pi_1,(\sigma,j)}^\tsep ) \right]$,
which is at most
\[
  2\E\left[ \max_\pi \sum_{\sigma\in \tsep}\sum_{j=1}^{N} \alpha_{\sigma,j}\Gamma_{\pi,(\sigma,j)}^\tsep\right].
\]
Let $\beta_\pi = \sum_{\sigma\in \tsep}\sum_{j=1}^N \alpha_{\sigma,j} \Gamma_{\pi,(\sigma,j)}^\tsep$. We therefore have for any $\lambda > 0$,
\begin{align*}
\E\left[ \max_\pi \beta_\pi \right]
&= \tfrac{1}{\lambda} \ln \left( \exp\left(\lambda\E\left[ \max_\pi \beta_\pi \right] \right)\right) \\
&\leq  \tfrac{1}{\lambda} \ln \left( \E\left[ \exp\left(\lambda \max_\pi \beta_\pi  \right) \right] \right)  \tag{Jensen's inequality}\\
&\leq  \tfrac{1}{\lambda} \ln \left( \sum_\pi \E\left[ \exp\left(\lambda\beta_\pi  \right) \right] \right) \\
&=  \tfrac{1}{\lambda} \ln \left( \sum_\pi \prod_{\sigma, j} \E\left[ \exp\left(\lambda\alpha_{\sigma,j} \Gamma^\tsep_{\pi,(\sigma,j)} \right) \right] \right)  \tag{by independence of $\alpha_{\sigma,j}$}\\
&\leq  \tfrac{1}{\lambda} \ln \left( \sum_\pi  \prod_{\sigma, j} \exp\left(\tfrac{\nu^2 \lambda^2}{2} \right) \right) \tag{Hoeffding's lemma for bounded r.v.}\\
&= \frac{\ln|\Pi|}{\lambda} + \frac{\nu^2 N |\tsep| \lambda}{2} \\
&= \nu\sqrt{2N|\tsep| \ln |\Pi|}
\enspace.
\tag{by picking the optimal $\lambda$}
\end{align*}
The theorem now follows by combining the stability and the error bound above and invoking Lemma \ref{lem:stab+noise_approx}.
\qed

\subsection{Proof of Corollary~\ref{cor:sep:neg}}

According to Lemma~\ref{lem:stab+noise_approx},
it suffices to bound the stability term $\E\bigBracks{ \sum_{t=1}^T f_c(\pi_{t+1},(\sigma_{t}, y_t))- f_c(\pi_t, (\sigma_t,y_t)) }$ and the perturbation error term
$\E\bigBracks{ \valpha\inprod( \mGamma_{\pi_1}^\sep  -  \mGamma_{\pi^*}^\sep ) }$.
Since the distribution $D$ is $\left(\frac{1+2\epsilon}{2\nu \delta}, \frac{1+2\epsilon}{\delta}\right)$-dispersed and also $\mGamma^\sep$ is $\delta$-admissible according to Lemma~\ref{lem:contextual_admissibility},
we invoke Lemma~\ref{lem:stability-approx} and plug in $\epsilon=1/\sqrt{T}$ to conclude that the stability term is bounded by $\frac{TNd (1+2T^{-1/2})(1+\delta^{-1})}{\nu \delta}$.
On the other hand, by the exact same proof as in Theorem~\ref{thm:transductive-FTPL-G}, the perturbation error term is bounded by $2\nu\sqrt{2Nd \ln |\Pi|}$.
Finally plugging in the value of~$\nu$ (which optimizes the regret bound) finishes the proof.
\qed

\section{Pseudo-Polynomial and Integer-Weighted Oracles}
\label{app:weak_oracle}

\subsection{Pseudo-Polynomial Oracles}
\label{app:pseudo}

If the offline optpmization oracle is only a pseudo-polynomial-time approximation scheme, such as some combinatorial optimization methods or optimization methods based on gradient descent,
then its running time will depend not only on the cardinality of the weighted data set $S$, but also on the actual magnitude of the weights.
Therefore, the magnitude of the weights that are needed to implement matrix $\mGamma$ will influence the final running time of the learning algorithm. To capture such types of oracles, we introduce the notion of
the \emph{pseudo-complexity}:

\begin{definition}
The \emph{pseudo-complexity} of a weighted data set $S = \{ (w_\ell,y_\ell) \}_{\ell\in\mathcal{L}}$ with $w_\ell\in\R^+$ is
\[
   \pseudo{S}\coloneqq\max\BigBraces{\card{S},\;\textstyle\sum_{\ell\in\mathcal{L}}w_\ell}.
\]
\end{definition}

\begin{definition}
\label{def:imp-pseudo}
  A matrix $\mGamma$ is implementable with \emph{pseudo-complexity} $W$ if for each $j\in[N]$ there exists a weighted dataset $S_j$,
  with $\pseudo{S_j}\leq W$,
  such that
\begin{align*}
\hspace{1in}
    &\text{for all $x,x'\in\X$:}
    &\Gamma_{xj}-\Gamma_{x'j} &= \sum_{(w,y)\in S_j} w\bigParens{f(x,y)-f(x',y)}.
&\hspace{1in}
\end{align*}
In this case, we say that weighted datasets $S_j$, $j\in[N]$,
implement $\mGamma$ with pseudo-complexity $W$.
\end{definition}

As an immediate corollary, we obtain that the existence of a pseudo-polynomial-time offline oracle implies the existence of a polynomial-time online learner with regret $\order(\sqrt{T})$, whenever we have
access to an implementable and admissible matrix. The following is an immediate consequence of Theorem~\ref{thm:FTPL-U} used with pseudo-polynomial oracles.

\begin{corollary}
\label{cor:oracle-ftpl:pseudo:range}
Consider the problem $(\X, \Y, f)$ where $f: \X\times\Y\rightarrow [0,1]$.
Assume that $\mGamma\in[0,1]^{\card{\X}\times N}$ is implementable with pseudo-complexity $W$ and $\delta$-admissible,
and there exists an approximate offline oracle
$\opt\bigParens{S, \frac{1}{\sqrt{T}}}$ which runs in time $\poly(\pseudo{S},T)$.
Then Algorithm~\ref{alg:oracleftpl} with distribution $D$ as defined in Theorem~\ref{thm:FTPL-U}
runs in time $\poly(N,W,T,1/\delta)$ and achieves
regret $\order(N\sqrt{T}/\delta)$.
\end{corollary}

Note that for some problem classes $(\X, \Y, f)$ discussed in the main body of the paper,  such as multi-item or multi-unit auctions, we have $f: \X\times \Y\rightarrow [0,R]$ for $R > 1$.
As we briefly discussed in Section~\ref{sec:auctions}, to apply our results we first implicitly re-scale function $f$ to $f'$ such that $f'(x, y) = f(x, y)/R$ for all $x \in \X$ and $y\in \Y$. We then apply Theorem~\ref{thm:FTPL-U-approx} to the problem $(\X, \Y, f')$ and obtain a $O(N \sqrt{T}/ \delta)$ regret for the normalized problem.
Lastly, we scale up $f'$ back to $f$ and get a regret bound that is $R$ times the regret for the normalized problem, i.e., $O(R N \sqrt{T} / \delta)$ for the problem $(\X, \Y, f)$. Let $\mGamma$ be the $\delta$-admissible matrix for the original problem $(\X, \Y, f)$ that is implemented by  $S = \{ (w_\ell,y_\ell) \}_{\ell\in\mathcal{L}}$. Then,   $S' = \{ (w_\ell R,y_\ell) \}_{\ell\in\mathcal{L}}$ implements $\mGamma$ for the re-scaled problem $(\X, \Y, f')$. That is, the pseudo-complexity of implementing $\mGamma$ through this re-scaling process is increased by at most a factor of $R$. Since $R$ is polynomial in the parameters of interest in all of our applications, this does not affect the running time of our algorithms substantially. Formally, we rely on the following result.

\begin{corollary}
\label{cor:oracle-ftpl:pseudo}
Consider the problem $(\X, \Y, f)$ where $f: \X\times\Y\rightarrow [0,R]$, $R\ge 1$.
Assume that $\mGamma\in[0,1]^{\card{\X}\times N}$ is implementable with pseudo-complexity $W$ for $(\X, \Y, f)$ and $\delta$-admissible,
and there exists an approximate offline oracle
$\opt\bigParens{S, \frac{1}{\sqrt{T}}}$ which runs in time $\poly(\pseudo{S},T)$.
Then there is an algorithm that runs in time $\poly(N,R,W,T,1/\delta)$ and achieves
regret $\order(RN\sqrt{T}/\delta)$.
\end{corollary}

In Appendix~\ref{app:pseudo-examples}, we derive the pseudo-complexity of implementing matrices $\mGamma$ for the (un-normalized) applications discussed in this paper. In all cases, the pseudo-complexity ($W$) of the unnormalized problem multiplied by its range ($R$) either agrees with complexity ($M$) of the problem or is a small function of the complexity and the number of columns.
Thus, in all these applications we obtain online algorithms with runtime $\poly(N,M,T,1/\delta)$ with only requiring pseudo-polynomial offline algorithms.
Furthermore, in all auction applications discussed in this paper $1/\delta = \poly(N,M)$ leading to
online algorithms with runtime $\poly(N,M,T)$ that only require pseudo-polynomial offline algorithms.

\subsection{Integer-Weighted Oracles}
\label{app:integral}

In the main body of the paper, we consider oracles $\opt$ that take as input a \emph{real-weighted} dataset and a precision parameter $\epsilon$. In this section, we focus on a seemingly more restricted oracle, $\weakopt$, that only takes integer-weighted datasets. We show that all of our results in Section~\ref{sec:ftpl} extend to this class of oracles.

\begin{definition}[Integer-Weighted Offline Oracle]
 An \emph{integer-weighted offline oracle} $\weakopt$ receives
 as input a set of adversary actions with non-negative integer weights
 $S = \{ (w_\ell,y_\ell) \}_{\ell\in\mathcal{L}}$ with $w_\ell\in\mathbb{N}$, $y_\ell\in \Y$ and a desired precision $\epsilon$,
 and returns an action $\hat{x}=\weakopt(S,\epsilon)$
 such that
 \[
    \sum_{(w,y) \in S} w f(\hat{x}, y)
    \ge
    \max_{x\in \X} \sum_{(w,y) \in S} w f(x, y) - \epsilon.
 \]
\end{definition}

We show that \weakopt can be used to implement a real-weighted oracle $\opt$. This implies that online learning can be efficiently reduced to offline optimization with \weakopt oracles under the same conditions as studied in Section~\ref{sec:ftpl}.
Next lemma outlines the construction of \opt from \weakopt. The resulting variant of the the Oracle-based Generalized FTPL algorithm (Algorithm~\ref{alg:oracleftpl}) is provided in Algorithm~\ref{alg:ftpl-integral}. Note
that the constructed \opt inherits the polynomial or the pseudo-polynomial efficiency of \weakopt. Pseudo-polynomial running times may arise naturally for those integer-weighted oracles, in which the integer weights are implemented
simply by replicated examples.

\begin{algorithm}[t]
  \begin{algorithmic}
  \STATE Input:
                datasets $S_j$, $j\in[N]$, that implement a matrix $\mGamma\in[0,1]^{\card{\X}\times N}$,
  \STATE \hspace{\algorithmicindent}
                distribution $D$ over $\R^+$,
  \STATE \hspace{\algorithmicindent}
                an integer-weighted offline oracle \weakopt.
    \STATE{Draw $\alpha_j \sim D$ independently for $j = 1, \ldots, N$.}
    \FOR{$t=1, \ldots, T$}
    \STATE{For all $j$, let $\alpha_j S_j$ denote the scaled version of $S_j$, i.e., $\alpha_j S_j \coloneqq\{(\alpha_j w, y): (w,y) \in S_j\}$}.
    \STATE{Set $S=\bigSet{(1,y_1),\dotsc,(1,y_{t-1})}\cup\bigcup_{j\le N} \alpha_j S_j$.}
    \STATE{Set $S'=\Set{
                   (w',y):\: {w'=\floors{w\cdot2\card{S}\sqrt{T}}}\text{ and }(w,y)\in S
                   }$.}
    \STATE{Play
      $x_t =  \weakopt(S',\card{S'})$.
    }
    \STATE{Observe $y_t$ and receive payoff $f(x_t, y_t)$.}
    \ENDFOR
\end{algorithmic}
\caption{Oracle-Based Generalized FTPL with Integer-Weighted Oracle}
  \label{alg:ftpl-integral}
\end{algorithm}

\begin{lemma}[Integer-Weighted Offline Oracle to Offline Oracle]
\label{lemma:weak:to:approx}
Given any real-weighted dataset $S$,  any precision parameter $\epsilon > 0$, and an integer-weighted offline oracle $\weakopt$, let
\[\textstyle
  S'=
  \Set{
                    (w',y):\:w'=\floors{w\cdot 2\card{S}/\epsilon}\text{ and }(w,y)\in S
                   }
,
\quad
  \epsilon'=\card{S'}
.
\]
For all $x\in\X$, we have
\[
\sum_{(w,y) \in S} w f(\weakopt(S',\epsilon'), y)
\geq \sum_{(w,y) \in S} w f(x, y) - \epsilon.
\]
Furthermore, $\card{S'}=\card{S}$ and $\pseudo{S'}\le 2\bigParens{\pseudo{S}}^2/\epsilon$. Therefore, if $\weakopt\bigParens{S', \card{S'}}$ runs in time $\poly(\card{S'})$ or $\poly(\pseudo{S'})$,
then the above construction with $\epsilon=1/\sqrt{T}$ implements $\opt\bigParens{S, \frac{1}{\sqrt{T}}}$ that runs in time $\poly(\card{S},T)$ or $\poly(\pseudo{S},T)$, respectively.
\end{lemma}
\begin{proof}
Let $\delta = \epsilon/(2|S|)$. Thus, elements $(w,y)\in S$ are transformed
into $(\lfloor w/\delta\rfloor,y)\in S'$.
To show the approximate optimality,
we use the fact that $0 \leq a - \lfloor a \rfloor \leq 1$ for any $a \geq 0$, so we have $0 \leq w - \delta\lfloor w/\delta\rfloor \leq \delta$ and thus for any $x \in \X$,

\begin{align*}
\sum_{(w,y) \in S} w f(\weakopt(S',\epsilon'), y)
&\geq \delta \sum_{(w,y) \in S} \floors{w/\delta} f(\weakopt(S',\epsilon'), y)  \\
&= \delta \sum_{(w',y) \in S'} w' f(\weakopt(S',\epsilon'), y)  \tag{by construction} \\
&\geq \delta \sum_{(w',y) \in S'} w' f(x, y)-\delta\epsilon'  \tag{from the definition of \weakopt} \\
&\geq \sum_{(w,y) \in S} (w - \delta) f(x, y)-\delta\epsilon' \\
&\geq \sum_{(w,y) \in S} w f(x, y)  - \card{S}\delta-\card{S}\delta\\
&= \sum_{(w,y) \in S} w f(x, y)  - \epsilon.
\tag*{\qedhere}
\end{align*}
\end{proof}

As an immediate consequence of Lemma~\ref{lemma:weak:to:approx} and Corollaries~\ref{cor:oracle-ftpl} and~\ref{cor:oracle-ftpl:pseudo} we obtain the following:

\begin{theorem}
\label{thm:integral}
Assume that $\mGamma\in[0,1]^{\card{\X}\times N}$ is $\delta$-admissible.
Then Algorithm~\ref{alg:ftpl-integral} with distribution $D$ as defined in Theorem~\ref{thm:FTPL-U}
achieves
regret $\order(N\sqrt{T}/\delta)$.

Furthermore, if $\,\mGamma$ is implementable with real-weighted datasets with complexity $M$ and the integer-weighted
oracle $\weakopt(S',\card{S'})$ runs in time $\poly(\card{S'})$ then Algorithm~\ref{alg:ftpl-integral} runs in time
$\poly(N,M,T)$. If $\mGamma$ is implementable with real-weighted datasets with pseudo-complexity $W$ and the integer-weighted
oracle $\weakopt(S',\card{S'})$ runs in time $\poly(\pseudo{S'})$ then Algorithm~\ref{alg:ftpl-integral} runs in time
$\poly(N,W,T,1/\delta)$.
\end{theorem}

\subsection{A Note on Numerical Computations}

For the ease of exposition, we have assumed throughout the paper that all numerical computations take $\order(1)$ time.
However, to invoke a real-weighted oracle on a given weighted dataset requires to communicate the weights, which may, in principle, be of an unbounded description length.
This problem can be circumvented by discretizing real weights to an appropriate accuracy as described in Appendix~\ref{app:integral} and instantiated in Algorithm~\ref{alg:ftpl-integral}.
Specifically, during the run of Algorithm~\ref{alg:ftpl-integral} with $D$ as defined in Theorem~\ref{thm:FTPL-U} we generate integer-weighted datasets $S$ with the pseudo-complexity $\poly(N,W,T,1/\delta)$ for
a $\delta$-admissible matrix $\mGamma$ with $N$ columns implementable with pseudo-complexity $W$.
Therefore, each individual weight can be described by $\order(\log(NWT/\delta))$ bits. Similarly, when the function is in range $[0, R]$, as discussed in Section~\ref{app:pseudo},
we generate integer-weighted datasets $S$ with the pseudo-complexity $\poly(N,R,W,T,1/\delta)$ for
a $\delta$-admissible matrix $\mGamma$ with $N$ columns implementable with pseudo-complexity $W$.
In all of our applications, the pseudo-complexity of implementing $\mGamma$, i.e., $W$,  multiplied by the range of the utility function, i.e., $R$,  is polynomial in the complexity $M$ of implementing $\mGamma$ and the number of columns $N$. Furthermore, the admissibility constant $\delta$ satisfies $1/\delta=\poly(N,M)$. As a result, each weight requires only $\order(\log(NMT))$ bits, and
the oracle-based running time of Algorithm~\ref{alg:ftpl-integral} remains $\poly(N,M,T)$ regardless of whether we assume that the numerical computation and storage of reals are $\order(1)$ or polynomial in the size of the bit representation.

\subsection{Pseudo-Complexity of Applications} \label{app:pseudo-examples}

In this section, we provide the pseudo-complexity of implementing matrices that are used by the applications in Sections~\ref{sec:VCG}, \ref{sec:envy-free}, \ref{sec:level}, \ref{sec:knapsack} and~\ref{sec:sispas}. The results are summarized in Table~\ref{tab:pseudo}. In our derivations below, recall that $m$ denotes the discretization level used in these auction classes. To show that $\mGamma$ is implementable with pseudo-complexity $W$, we need to show that the datasets $S_j$ implementing its columns satisfy $\pseudo{S_j}\le W$, i.e., $|S_j|\leq W$ and $\sum_{(w,y)\in S_j} w \leq W$.

\renewcommand{\arraystretch}{1}
\begin{table*}[t]
\caption{\label{tab:pseudo} The pseudo-complexity  and complexity of the datasets that implement translation matrices for our (un-normalized) applications. In this table, $n$ refers to the number of bidders, $m$ refers to the bid's discretization level, and $k$ refers to the number of items.
The payoff range $[0,R]$ for $R> 1$ affects the true pseudo-complexity of querying the oracle by $R$. The payoff range of $[0,n]$ for VCG auction with bidder-specific reserves is stated for the general case where every bidder's valuation is in range $[0,1]$. For $s$-unit VCG auction with bidder-specific reserves the payoff range is $[0,s]$ for $s\leq n$.}
\centering
\resizebox{\columnwidth}{!}{%
\begin{tabular}{|c|c|c|c|c|c|c|c|}
\hline
Problem class & Matrix  & \makecell{Pseudo-\\complexity} & Complexity & \makecell{Number of\\columns} & \makecell{Admissi-\\bility $\delta$} &\makecell{Payoff\\range} & Section \\
\hline
VCG with bidder-specific reserves & $\mGamma^{\textsc{VCG}}$ &
$O(m^2)$ & $m$ & $n\lceil \log m \rceil$ & $1$ & $[0, n]$&
\ref{sec:VCG} \\
\hline
envy-free item pricing & $\mGamma^{\textsc{IP}}$ &
$O(m^2)$ & $m$ & $k\lceil \log m \rceil$ & $1$ & $[0, n]$ &
\ref{sec:envy-free} \\
\hline
level auctions (without repetitions) & $\mGamma^{\textsc{SL}}$ &
$1$ & $1$ & $nm$ & $1/m$ & $[0,1]$ &
\ref{sec:level} \\
\hline
level auctions (with repetitions) & $\mGamma^\barRL$ &
$1$& $1$ & $O(n^3 m^3)$ & $1/m$ & $[0,1]$ &
\ref{sec:level} \\
\hline
multi-unit welfare maximization & $\mGamma^{\textsc{MU}}$ &
$1$ & $1$ & $n$ & $1/(4n^2)$ & $[0,n]$ &
\ref{sec:knapsack} \\
\hline
SiSPAs&  $\mGamma^{\textsc{OB}}$ &
$m$ & $m$ & $k$ & $1/m$ &  $[0,k]$ &
\ref{sec:sispas} \\
\hline
\end{tabular}
}
\end{table*}

\paragraph{VCG with Bidder-Specific Reserves.}
We show that the pseudo-complexity of implementing $\mGamma^{\textsc{VCG}}$  is $O(m^2)$.

Recall that we used datasets $S_j$ for each column $j\leq n\lceil \log m \rceil$  that included the $m$ bid profiles in which only the bidder corresponding to column $j$ (call it bidder $i$) has  non-zero valuation,
  denoted as $\vec v_\ellOther\coloneqq(\ellOther/m)\vec e_i$ for $\ellOther\le m$. Specifically, we used $S_j = \{(w_\ellOther, \vec v_\ellOther)\}_{\ellOther\in [m]}$ with the weights
  defined recursively as follows:
  \[
 w_m = \max\BigBraces{0,\;\max_z\bigBracks{ m\bigParens{z_\bit - (z-1)_\bit}}
 },
\]
and for all $z = m, m-1,\dotsc, 2$,
\[
 w_{z-1} = \frac{1}{z-1} \left( \sum_{\ellOther = z}^m w_\ellOther  - m\bigParens{z_\bit - (z-1)_\bit} \right).
\]

We show that $\sum_{z=1}^m w_z \leq 2m(m-1) + m$. Note that by definition $w_m \leq m$. Moreover, by the definition of $w_{z-1}$, we have
\[   w_{z-1} = \frac{1}{z-1} \left( \sum_{\ellOther = z}^m w_\ellOther  - m\bigParens{z_\bit - (z-1)_\bit} \right) \leq  \frac{1}{z-1} \left( \sum_{\ellOther = z}^m w_\ellOther + m \right).
\]
Reformulating the above, we have
\[   (z-1) w_{z-1}
\leq \sum_{\ellOther = z}^m w_\ellOther + m \leq \sum_{\ellOther = z}^{m-1} w_\ellOther + 2m.
\]
Next, we sum over the above inequality for $z = 2, \dots, m$, obtaining
\begin{align*}
\sum_{z=2}^{m} (z-1) w_{z-1} &\leq  \sum_{z=2}^{m} \left( \sum_{\ellOther = z}^{m-1} w_\ellOther + 2m \right)\\
&= \sum_{z=2}^{m} \sum_{\ellOther = z}^{m-1} w_\ellOther + 2m (m-1)\\
&=  \sum_{h' = 2}^{m-1} (h' -1) w_{h'} + 2m (m-1)\\
&=  \sum_{h' = 3}^{m} (h' -2) w_{h'-1} + 2m (m-1),
\end{align*}
where the penultimate transition is by rearranging the summation and counting the number of times each $w_{h'}$ appears in the sum. Subtracting the two sides and changing parameter $z\gets z-1$, we get
\[
\sum_{z=1}^{m-1} w_{z} \leq   2m (m-1).
\]
Thus, using the fact that $w_m \leq m$, we obtain $\sum_{z=1}^m w_z \leq 2m(m-1) + m$. Since $\card{S_j}=m$, this means that $\pseudo{S_j}=\order(m^2)$.

\paragraph{Envy-free Item Pricing.}
 As we show in Appendix~\ref{app:impl_envy}, as far as the implementability of $\mGamma^{\textsc{IP}}$ is concerned, this setting is isomorphic to the case of VCG with bidder-specific prices. So, the pseudo-complexity of implementing $\mGamma^{\textsc{IP}}$ is also $O(m^2)$.

\paragraph{Level Auctions.}
Recall that the datasets used for implementing $\mGamma^{\textsc{SL}}$ are of the form of $S_{\vec v} = \{ (1, \vec v)\}$, for a suitable~$\vec v$, so
$\mGamma^{\textsc{SL}}$ is implementable with pseudo-complexity of $1$.

Similarly, recall that $\mGamma^\barRL$ is a matrix whose columns are indexed by $\vec v \in \{\vec v \mid \vec v\in \{0,  1/m, \dots,  m/m\}^n \\\text{ and } \|\vec v\|_0 \leq 3 \}$,
such that $\Gamma^\barRL_{\vec \theta, \vec v} = \rev( \vec \theta, \vec v)$.
That is, $\mGamma^\barRL$  is implemented by datasets $S_{\vec v} = \{(1, \vec v)\}$. Therefore $\mGamma^\barRL$ is implementable with pseudo-complexity of $1$.

\paragraph{Multi-unit Welfare Maximization.}
Recall that $\mGamma^{\textsc{MU}}$ is implemented by sets $S_j = \{(1, \vec v)\}$, for a suitable~$\vec v$, so $\mGamma^{\textsc{MU}}$ is implementable with pseudo-complexity of $1$.

\paragraph{Online Bidding in SiSPAs.}
Recall that $\mGamma^{\textsc{OB}}$ is implemented by datasets $S_j = \{ (w_\ell, \vec p_\ell) \}_{\ell\in\set{0,\dotsc,m-1}}$, where
\[
 w_{\ell} = \begin{cases}
\frac1m\cdot\frac{1}{v(\vec e_j) - \ell/m}   &\text{if } \ell/m < v(\vec e_j)  \\
0 & \text{otherwise.}
\end{cases}
\]
Since $v(\vec e_j) \in \{0, 1/m, \dots, m/m\}$,  we have that $w_{\ell} \leq 1$. Moreover, $|S_j|=m$, so $\mGamma$ is implementable with pseudo-complexity $m$.

\end{document}